\documentclass[10pt]{article}
\ifdefined\pdfsuppressptexinfo \pdfsuppressptexinfo=-1\relax\fi

\usepackage{preprint-layout}
\usepackage{amsmath}
\usepackage{amssymb}
\usepackage{amsthm}
\usepackage{graphicx}
\usepackage{booktabs}
\usepackage{tikz}
\usepackage{microtype}
\usepackage{hyperref}
\usepackage[capitalize,noabbrev]{cleveref}

\theoremstyle{plain}
\newtheorem{theorem}{Theorem}[section]
\newtheorem{proposition}[theorem]{Proposition}
\newtheorem{lemma}[theorem]{Lemma}
\theoremstyle{definition}
\newtheorem{definition}[theorem]{Definition}
\AddToHook{env/proposition/begin}{\crefalias{theorem}{proposition}}
\AddToHook{env/lemma/begin}{\crefalias{theorem}{lemma}}
\AddToHook{env/definition/begin}{\crefalias{theorem}{definition}}

\newcommand{\R}{\mathbb{R}}
\newcommand{\fld}{\phi}
\newcommand{\lvl}{\psi}
\newcommand{\str}{\beta}
\newcommand{\inc}{\delta}
\newcommand{\rmin}{\bar{\phi}}
\newcommand{\temp}{\tau}
\newcommand{\tol}{\varepsilon}
\newcommand{\shp}{\kappa}
\newcommand{\dep}{\eta}
\newcommand{\tokfn}{\lambda}
\newcommand{\sasc}{\omega}
\newcommand{\pairfn}{\chi}
\newcommand{\tokfnB}{\mu}
\newcommand{\lgt}{\ell}
\newcommand{\posb}{\rho}
\newcommand{\kpos}{\sigma}
\newcommand{\gapA}{\gamma_{\mathrm{A}}}
\newcommand{\gapD}{\gamma_{\mathrm{D}}}
\newcommand{\ladder}{\mathcal{F}}
\newcommand{\twowalk}{\mathcal{G}}
\newcommand{\fcls}{\mathcal{C}}
\newcommand{\sig}{\alpha}
\newcommand{\agg}{\Gamma}
\newcommand{\tokbos}{\langle\mathrm{s}\rangle}
\newcommand{\tokq}{\texttt{?}}

\usetikzlibrary{arrows.meta,calc}

\definecolor{fpacol}{HTML}{D55E00}
\definecolor{foxcol}{HTML}{0072B2}
\definecolor{softmaxcol}{HTML}{000000}
\definecolor{figframe}{HTML}{9CA3AF}
\definecolor{figpanel}{HTML}{F7F7F5}
\definecolor{figmuted}{HTML}{6B7280}
\definecolor{archpink}{HTML}{F8D7DA}
\definecolor{archyellow}{HTML}{FFF3BF}
\definecolor{archblue}{HTML}{D0E8F7}
\definecolor{archpurple}{HTML}{E5D9F2}
\definecolor{archgreen}{HTML}{D4EFDF}
\newcommand{\fpatok}[1]{{\setlength{\fboxsep}{1.2pt}\fcolorbox{figframe}{white}{\strut\texttt{#1}}}}
\newcommand{\bksp}{$\leftarrow$}

\newcommand{\PaperTitle}{First Passage Attention}

\newcommand{\AbstractText}{%
Augmenting softmax attention with forget gates typically relies on summing decay contributions along the key--query path.
Because such accumulation is monotonic, these mechanisms cannot selectively erase intermediate context while preserving older, structurally relevant tokens.
We propose First Passage Attention (FPA), which shifts from path summation to path extrema by lowering a key's logit proportionally to the deepest dip of a learned scalar field below that key's level over the intervening span.
FPA encompasses the Forgetting Transformer (FoX) as a special case, and when the field and levels follow a shared learned walk, the alive keys with unpenalised logits naturally form a stack.
We show that a constant-dimension FPA head realises any subset of keys left alive by subsequent dips, whereas softmax heads with positional biases or token-level forget gates require a head dimension linear in the number of keys.
In synthetic experiments, one-layer FPA heads converge to these target sets while comparable softmax and FoX baselines fail on every seed.
Scaling to a 110M-parameter language model, FPA matches FoX in perplexity and uniquely achieves perfect passkey retrieval at $8\times$ the training length, even with its key--value cache compressed to about $25\%$.%
}

\newcommand{\KeywordsText}{attention, forgetting, Forgetting Transformer, expressivity, length extrapolation}

\graphicspath{{figures/}}

\begin{document}

\title{\PaperTitle}
\author{Yuxuan Liu\\
  Department of Mathematics, University of Bergen, Bergen, Norway\\
  \texttt{yuxuan.liu@uib.no}}
\date{October 2026}
\hypersetup{pdftitle={\PaperTitle},pdfauthor={Yuxuan Liu},pdfkeywords={\KeywordsText}}
\maketitle

\begin{abstract}
  \AbstractText
\end{abstract}

\begin{figure*}[t]
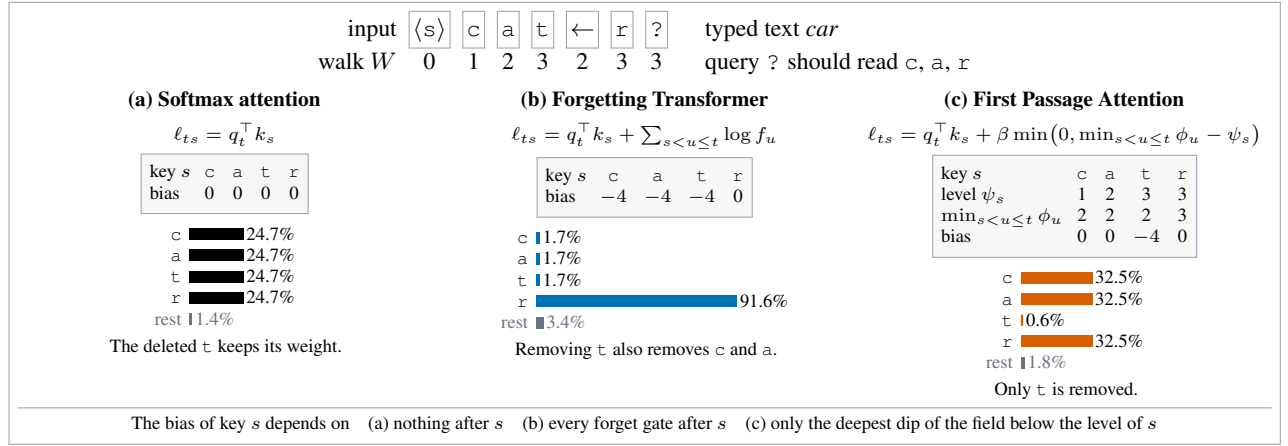

  \centering
  \setlength{\fboxsep}{3pt}%
\fcolorbox{figframe}{white}{\parbox{0.975\textwidth}{\centering
{\small\setlength{\tabcolsep}{2.2pt}\begin{tabular}{@{}rccccccc@{\hspace{14pt}}l@{}}
input & \fpatok{$\langle$s$\rangle$} & \fpatok{c} & \fpatok{a} & \fpatok{t} & \fpatok{\bksp} & \fpatok{r} & \fpatok{?} & typed text \emph{car} \\
walk $W$ & 0 & 1 & 2 & 3 & 2 & 3 & 3 & query \texttt{?} should read \texttt{c}, \texttt{a}, \texttt{r} \\
\end{tabular}}\\[4pt]
\begin{minipage}[t]{0.322\textwidth}\centering
{\footnotesize\textbf{(a) Softmax attention}}\\[1pt]
{\scriptsize $\ell_{ts}=q_t^{\top}k_s$}\\[2pt]
\fcolorbox{figframe}{figpanel}{\scriptsize\setlength{\tabcolsep}{3.2pt}%
\begin{tabular}{@{}l@{\hspace{5pt}}cccc@{}}
key $s$ & \texttt{c} & \texttt{a} & \texttt{t} & \texttt{r} \\
bias & 0 & 0 & 0 & 0 \\
\end{tabular}}\\[3pt]
{\scriptsize\setlength{\tabcolsep}{2pt}\begin{tabular}{@{}r@{\hspace{3pt}}l@{}}
\texttt{c} & \textcolor{softmaxcol}{\rule{0.715cm}{4.2pt}}\,24.7\% \\
\texttt{a} & \textcolor{softmaxcol}{\rule{0.715cm}{4.2pt}}\,24.7\% \\
\texttt{t} & \textcolor{softmaxcol}{\rule{0.715cm}{4.2pt}}\,24.7\% \\
\texttt{r} & \textcolor{softmaxcol}{\rule{0.715cm}{4.2pt}}\,24.7\% \\
{\color{figmuted}rest} & \textcolor{figmuted}{\rule{0.039cm}{4.2pt}}\,{\color{figmuted}1.4\%} \\
\end{tabular}}\\[2pt]
{\scriptsize The deleted \texttt{t} keeps its weight.}
\end{minipage}\hfill
\begin{minipage}[t]{0.322\textwidth}\centering
{\footnotesize\textbf{(b) Forgetting Transformer}}\\[1pt]
{\scriptsize $\ell_{ts}=q_t^{\top}k_s+\sum_{s<u\le t}\log f_u$}\\[2pt]
\fcolorbox{figframe}{figpanel}{\scriptsize\setlength{\tabcolsep}{3.2pt}%
\begin{tabular}{@{}l@{\hspace{5pt}}cccc@{}}
key $s$ & \texttt{c} & \texttt{a} & \texttt{t} & \texttt{r} \\
bias & $-4$ & $-4$ & $-4$ & 0 \\
\end{tabular}}\\[3pt]
{\scriptsize\setlength{\tabcolsep}{2pt}\begin{tabular}{@{}r@{\hspace{3pt}}l@{}}
\texttt{c} & \textcolor{foxcol}{\rule{0.049cm}{4.2pt}}\,1.7\% \\
\texttt{a} & \textcolor{foxcol}{\rule{0.049cm}{4.2pt}}\,1.7\% \\
\texttt{t} & \textcolor{foxcol}{\rule{0.049cm}{4.2pt}}\,1.7\% \\
\texttt{r} & \textcolor{foxcol}{\rule{2.656cm}{4.2pt}}\,91.6\% \\
{\color{figmuted}rest} & \textcolor{figmuted}{\rule{0.098cm}{4.2pt}}\,{\color{figmuted}3.4\%} \\
\end{tabular}}\\[2pt]
{\scriptsize Removing \texttt{t} also removes \texttt{c} and \texttt{a}.}
\end{minipage}\hfill
\begin{minipage}[t]{0.322\textwidth}\centering
{\footnotesize\textbf{(c) First Passage Attention}}\\[1pt]
{\scriptsize $\ell_{ts}=q_t^{\top}k_s+\beta\min\bigl(0,\min_{s<u\le t}\phi_u-\psi_s\bigr)$}\\[2pt]
\fcolorbox{figframe}{figpanel}{\scriptsize\setlength{\tabcolsep}{3.2pt}%
\begin{tabular}{@{}l@{\hspace{5pt}}cccc@{}}
key $s$ & \texttt{c} & \texttt{a} & \texttt{t} & \texttt{r} \\
level $\psi_s$ & 1 & 2 & 3 & 3 \\
$\min_{s<u\le t}\phi_u$ & 2 & 2 & 2 & 3 \\
bias & 0 & 0 & $-4$ & 0 \\
\end{tabular}}\\[3pt]
{\scriptsize\setlength{\tabcolsep}{2pt}\begin{tabular}{@{}r@{\hspace{3pt}}l@{}}
\texttt{c} & \textcolor{fpacol}{\rule{0.944cm}{4.2pt}}\,32.5\% \\
\texttt{a} & \textcolor{fpacol}{\rule{0.944cm}{4.2pt}}\,32.5\% \\
\texttt{t} & \textcolor{fpacol}{\rule{0.030cm}{4.2pt}}\,0.6\% \\
\texttt{r} & \textcolor{fpacol}{\rule{0.944cm}{4.2pt}}\,32.5\% \\
{\color{figmuted}rest} & \textcolor{figmuted}{\rule{0.052cm}{4.2pt}}\,{\color{figmuted}1.8\%} \\
\end{tabular}}\\[2pt]
{\scriptsize Only \texttt{t} is removed.}
\end{minipage}
\par\vspace{3pt}{\color{figframe}\hrule}\vspace{3pt}
{\scriptsize The bias of key $s$ depends on\quad(a)~nothing after $s$\quad(b)~every forget gate after $s$\quad(c)~only the deepest dip of the field below the level of $s$}
\par\vspace{1pt}}}

  \caption{\textbf{Three heads read the same typed word.}
    A user types \texttt{c}, \texttt{a} and \texttt{t}, deletes \texttt{t} with the backspace key~(\bksp) and types \texttt{r}, so the query~\tokq{} should attend to \texttt{c}, \texttt{a} and \texttt{r}.
    Every character has the content score $q_t^{\top}k_s=2$ and every other token~$-2$.
    (a)~Softmax attention has no term that a later token can change.
    (b)~The Forgetting Transformer (FoX) can remove \texttt{t} only through a small forget gate~$f_u$ at~\bksp, and that gate lowers \texttt{c} and \texttt{a} as well.
    (c)~First Passage Attention (FPA) in its stack form uses the walk~$W$ as the field~$\fld$ and as the level~$\lvl$, with strength $\str=4$.
    Only \texttt{t} has a level above the lowest later value of~$W$, so only \texttt{t} loses its weight.}
  \label{fig:1}
\end{figure*}

\section{Introduction}
\label{sec:1}

Transformers power most modern large language models, with softmax attention serving as their central operation \citep{vaswani2017attention,jiang2023mistral}.
Standard attention scores each earlier token by matching its key against the current query, typically incorporating relative positional encodings such as rotary embeddings or distance penalties \citep{press2022train,su2024roformer}.
Intervening tokens between a key and a query do not directly enter this score.
In many sequences, however, a subsequent event renders earlier information obsolete.
A backspace deletes the last typed character, a closing bracket terminates an enclosed scope, and a reassignment overwrites a variable's value.
Crucially, such events often invalidate a newer key while leaving older keys intact.

Several recent mechanisms modulate attention logits via sums, products, or transformations accumulated along the key--query path \citep{csordas2022neural,golovneva2024cope,tan2025stickbreaking,yang2025path,zhang2026grape}.
The Forgetting Transformer (FoX) integrates recurrent forget gates into softmax attention to improve long-context modelling and length extrapolation \citep{lin2025forgetting}.
Selective Attention allows tokens to mark earlier keys as obsolete \citep{leviathan2025selective}, and Dynamic Context Pruning (DCP) drops keys to reduce inference memory \citep{anagnostidis2023dynamic}.
However, because these mechanisms aggregate path contributions via summation or multiplication, they face two structural limitations.
First, token-level forget gates (as in FoX) attenuate keys monotonically with age, which makes it impossible to erase a recent token without simultaneously suppressing all older ones.
Second, pairwise accumulation (as in Selective Attention and DCP) causes suppression to grow with the number of intervening steps (confounding the \emph{depth} of an event with its \emph{duration}) while incurring pairwise computation overhead.

We introduce \textbf{First Passage Attention (FPA)}, which aggregates the key--query path by its \emph{minimum} rather than its sum.
Each head computes a scalar field at every position and a scalar level for every key.
A key's logit remains unattenuated (the key is alive) as long as the subsequent field stays at or above its level.
Otherwise, FPA lowers the logit in proportion to the depth of the deepest dip below that level, regardless of how long that dip lasts.
We show that a single linear parameterisation encompasses FoX, a bottleneck bias, and learned expiry \citep{sukhbaatar2021expire} as special cases.
Most notably, when the field and levels share a learned walk (a prefix sum of per-token increments), any drop in the walk selectively eliminates only those alive keys whose levels lie above the new value.
The alive keys naturally evolve as a stack, which allows a later token to erase a newer key while preserving older ones (\cref{fig:1}).

We prove that in one layer a single constant-dimension FPA head realises every dip-induced alive set, whereas softmax heads with positional biases or token-level forget gates need a head dimension linear in the number of keys.
We also prove that range extrema uniquely achieve duration-invariant suppression, so FPA fits variable-duration dips where additive selection fails, and that head-specific fields overcome the dimension bottleneck of layer-shared masks.
We implement FPA in a deterministic Triton kernel that runs about as fast as PyTorch's scaled dot-product attention (SDPA) in training.
In preregistered one-layer probes, FPA learns our theoretical target families on most seeds, whereas the baselines that our lower bounds cover fail on every seed.
We train a 110M-parameter FPA language model that matches FoX in perplexity and uniquely achieves $100\%$ passkey retrieval at $4\times$ and $8\times$ its training length, even with about a quarter of its key--value cache.

\section{First Passage Attention}
\label{sec:2}

In this section we define the bias and its parameterisation, list its special cases and properties, and describe how we train and compute it.

\paragraph{Definition.}
Let a causal head of dimension~$d$ have queries $q_t$, keys $k_s$ and values $v_s$, a scalar field~$\fld_u$ at every position and a scalar level~$\lvl_s$ for every key.
FPA outputs $\sum_{s\le t}w_{ts}v_s$ with the softmax weights $w_{ts}\propto e^{\lgt_{ts}}$ of the logits $\lgt_{ts}=q_t^{\top}k_s/\sqrt d+b_{ts}$, where the first-passage bias is
\begin{equation}
  b_{ts}=\str\min\Bigl(0,\ \min_{s<u\le t}\fld_u-\lvl_s\Bigr)
  \label{eq:1}
\end{equation}
with a learned strength $\str>0$, and $b_{tt}=0$.
We write $\rmin_{ts}=\min_{s<u\le t}\fld_u$ for the range minimum of the field after key~$s$.

\begin{definition}[Alive and dead keys]
  \label{def:2.1}
  Key~$s\le t$ is alive at time~$t$ if $b_{ts}=0$, that is if $s=t$ or $\rmin_{ts}\ge\lvl_s$, and dead otherwise.
  The alive keys at time~$t$ form the alive set, and a dip is a decrease of the field below the level of an earlier key.
\end{definition}

As $\str\to\infty$, the bias of a dead key tends to $-\infty$ and its weight to zero.
For a finite strength the bias of a dead key~$s$ is $-\str$ times the depth $\lvl_s-\rmin_{ts}$ of its deepest dip.

\paragraph{Walk, field and level.}
Each head computes three scalars at every position from the input $x_t$ of the layer by linear maps, an increment $\inc_t$ and two offsets $g_t$ and $h_t$.
With the walk $W_t=\sum_{u\le t}\inc_u$, the field and the level are
\begin{equation}
  \fld_t=W_t+g_t,\qquad \lvl_t=W_t+h_t,
  \label{eq:2}
\end{equation}
and the strength $\str$ is a learned positive scalar per head.
A token can thus move the reference of every later comparison through $\inc$, and it can dip or set a level without moving it through $g$ and~$h$.
A layer of width $D$ with $H$ heads gains $3H(D+1)+H$ parameters, 332 thousand in our 110M language model.

\begin{proposition}[Special cases]
  \label{prop:2.2}
  Under \eqref{eq:2} the following hold.
  \begin{enumerate}
    \item[(a)] If $g=h=0$, then $b_{ts}=-\str L_{ts}$ with $L_{ts}=\max\bigl(0,\max_{s<u\le t}(W_s-W_u)\bigr)$, the regulator of a queue that starts empty at time~$s$ and receives the net input~$\inc_u$ at step~$u$.
    \item[(b)] If in addition $\inc_u\le0$ for all~$u$, then $b_{ts}=\sum_{s<u\le t}\log f_u$ with $f_u=e^{\str\inc_u}\in(0,1]$, the bias of FoX with forget gates~$f_u$.
    \item[(c)] If $\inc=0$, $h=0$, $g\le0$ and $\str=1$, then $b_{ts}=\min_{s<u\le t}g_u$, a bottleneck bias.
    \item[(d)] If $\inc_u=-1$, $g=0$ and $h_s=-e_s$ with $e_s\ge0$, then $b_{ts}=\str\min\bigl(0,e_s-(t-s)\bigr)$, so key~$s$ is alive exactly while $t-s\le e_s$.
  \end{enumerate}
\end{proposition}

We call the case $g=h=0$ the stack form, for a reason that \cref{prop:2.3} gives.
Part~(b) makes FoX the stack form with a walk that never rises.
A rising walk lets FPA keep an old key while it removes newer ones, which no forget gate can do (\cref{fig:1}).
Part~(a) reads the stack form as a queue \citep{lindley1952theory,skorokhod1961stochastic}, and the strengths of the Neural Stack follow the same closed form \citep{grefenstette2015transduce}.
We use this reading only as an interpretation.
Part~(d) has the expiry rule of Expire-Span \citep{sukhbaatar2021expire}, and part~(b) is not the only way to contain FoX in a larger family \citep{zhang2026grape}.

\begin{proposition}[Death is permanent, and the stack form gives a stack]
  \label{prop:2.3}
  (a)~For all $s<t<t'$, $b_{t's}\le b_{ts}$, so a dead key stays dead.
  (b)~If $\fld=\lvl=W$, the levels $W_s$ of the keys alive at time~$t$ do not decrease with~$s$, and the keys that step~$u$ kills form a suffix in time of the keys alive at $u-1$.
  The alive set therefore evolves as a stack.
\end{proposition}

Selective Attention and DCP share the property of part~(a), which both papers state \citep{anagnostidis2023dynamic,leviathan2025selective}.
The weights of the alive keys still change when other keys die, because the softmax renormalises over the keys that remain.

\paragraph{Training.}
The bias is piecewise linear in the field and the levels.
On the closed side of the outer minimum it passes no gradient to the field, and in our pilot runs this minimum could close for every key of a head early in training and stay closed.
We train FPA with a straight-through estimator for this reason.
The forward pass uses the range minimum, and the backward pass uses in its place the gradient of the soft minimum $-\temp\log\sum_{s<u\le t}e^{-\fld_u/\temp}$ with $\temp=1$.
The outer minimum acts on the hard value, and evaluation uses the bias of \eqref{eq:1} unchanged.

\begin{proposition}[Cost]
  \label{prop:2.4}
  (a)~In a blocked computation with block size $B$, the bias between a query block~$I$ and an earlier key block~$J$ depends on the field only through three quantities.
  They are the suffix minima of $\fld$ within~$J$, the minimum of $\fld$ over the blocks strictly between $J$ and~$I$, and the prefix minima of $\fld$ within~$I$.
  It costs $O(B^2)$ operations per pair of blocks, as an element-wise bias such as that of FoX does, together with $O(T/B)$ block minima for a sequence of length~$T$.
  (b)~In autoregressive decoding, $\rmin_{ts}=\min(\rmin_{t-1,s},\fld_t)$ for $s<t-1$ and $\rmin_{t,t-1}=\fld_t$, so each step updates the bias row in $O(t)$ operations and stores two scalars per cached key and head.
\end{proposition}

The decomposition of part~(a) is the standard blocked range minimum \citep{bender2000lca}.
Our Triton kernel fuses it with the attention, which it computes in blocks with an online softmax as fused attention kernels do.
In the backward pass it computes the soft-minimum gradient for the positions in the key block, in the query block, in the blocks between and in the diagonal block.
It adds nothing atomically, so its gradients are deterministic (\cref{app:E}).
The kernel needs the range minimum only for a finite strength.
In the hard limit every key is visible on one interval of queries, which mask formats such as that of FlashMask already represent \citep{wang2025flashmask}.
Selective Attention and DCP, by contrast, accumulate a score for every pair of a later token and a key.
Their bias for a pair of blocks depends on the score of every token in between with every key of the block.
\Cref{sec:4.3} compares the cost of FPA and Selective Attention.

\section{Theory}
\label{sec:3}

In this section we compare FPA with other heads on three properties of the alive set.
All proofs are in \cref{app:A}, and each argument is short.

\subsection{Alive sets chosen by later dips}
\label{sec:3.1}

Fix $r\ge3$.
The ladder family $\ladder_r$ consists of $2^{r-1}$ sequences of length $3r$ (\cref{app:A.1}).
Every sequence starts with $\tokbos$, places the same key token~\texttt{K} at the positions $\kpos_j=3j-2$ for $j=1,\dots,r$, and ends with the query token~\tokq.
Between key~$j$ and key~$j+1$ it places either \texttt{U}\,\texttt{D} or \texttt{D}\,\texttt{U}, and the tokens \texttt{K}, \texttt{U} and \texttt{D} have the walk increments $1$, $\tfrac12$ and $-\tfrac12$.
In the stack form, key~$j<r$ is alive at the final query exactly when \texttt{U}\,\texttt{D} follows it, so the alive sets of the family are all $2^{r-1}$ subsets of the first $r-1$ keys, each with the last key.
We say that a head realises an alive set~$A$ with tolerance~$\tol$ when, at the final query, every key in~$A$ has weight at least $(1-\tol)/|A|$ and the other keys have total weight at most~$\tol$.

\begin{theorem}[Positional biases]
  \label{thm:3.1}
  Let a head have logits $\langle q_t,k_s\rangle+\posb(t,s)+c_t$.
  Here $k_s\in\R^d$ depends only on the token and the position of~$s$, $\posb$ is any fixed function of the two positions, and $q_t$ and $c_t$ may depend arbitrarily on the prefix up to~$t$.
  If on every sequence of $\ladder_r$ every alive key receives more weight than every dead key at the final query, then $d\ge r-2$.
  The bound still holds when the head adds a path term $\sum_{s<u\le t}\tokfn(x_u)$ in which $\tokfn$ depends only on the token~$x_u$ and may take either sign.
\end{theorem}

The path term covers a FoX head whose forget gates depend only on the token, as in the first layer of a model without positional embeddings.
The condition on the weights holds whenever the head realises the alive sets with tolerance $\tol<1/(r+1)$.
In the first layer every key depends only on its token and position.
Deeper keys may read their prefix, and \cref{thm:3.2} bounds such heads by their precision.

\begin{theorem}[Keys fixed when written]
  \label{thm:3.2}
  Let a head have logits $\langle q_t,k_s\rangle+\posb(t,s)+c_t$ with $q_t,k_s\in\R^d$.
  Let $k_s$ depend only on the prefix up to~$s$, let $q_t$ and $c_t$ depend on the prefix up to~$t$, and let $\posb$ be a fixed function of the two positions.
  Let every coordinate of $q_t$ take one of at most $2^p$ values from a set that does not depend on the input, with logits and weights evaluated exactly.
  If the head realises the alive set of every sequence of $\ladder_r$ with tolerance $0<\tol\le1/(2r+2)$, then $dp\ge r-2$.
\end{theorem}

\begin{proposition}[One FPA head]
  \label{prop:3.3}
  Let a head use the stack form on $\ladder_r$ and add to its bias a one-dimensional dot product, whose query is the constant $-M$ and whose key is $1$ at the tokens other than~\texttt{K} and $0$ at~\texttt{K}.
  If $\str\ge2\log(2(r-1)/\tol)$ and $M\ge\log(4r/\tol)$, this head realises the alive set of every sequence of $\ladder_r$ with tolerance~$\tol$, and its walk takes at most $2r+1$ values.
\end{proposition}

At $\tol=1/(2r+2)$ one FPA head thus uses scalars of $O(\log r)$ bits and constants $\str$ and $M$ of order $\log r$, whereas the heads of \cref{thm:3.1,thm:3.2} need $r-2$ dimensions or $r-2$ bits of query in total.
The proof of \cref{thm:3.1} bounds how many sign patterns affine functions can take on fixed points, a classical argument \citep{cover1965geometrical,dudley1978central} that earlier lower bounds for attention also use \citep{sanford2023representational,weller2026theoretical}.
In the proof of \cref{thm:3.2} we instead count the values that the query must take.
What is new is the family of alive sets that later dips create, and the comparison with a single range minimum.

The results of this subsection do not cover biases in which a later token acts on each key separately.
Selective Attention, DCP and keys that later tokens update \citep{song2026causal} realise the same alive sets in the hard limit, so FPA is not more expressive than they are in that limit.
FPA differs from Selective Attention and DCP in that its bias depends only on scalars of single positions (\cref{prop:2.4}), and from Selective Attention in its dependence on depth rather than duration (\cref{sec:3.2}).
The results also say nothing about deeper models, because a few layers of softmax attention can compute bracket depth \citep{yao2021selfattention,hayakawa2025}.
Without a bound on precision no lower bound on $d$ holds for the ordering condition of \cref{thm:3.1}, since keys that read their prefix can encode the alive set in one coordinate (\cref{app:A}).
Finally, the theorems constrain the weights of one head and not the function of a layer, since one attention layer with low-rank weights and feed-forward layers already approximates every continuous permutation-equivariant function \citep{kajitsuka2024transformers}.

\subsection{Depth, not duration}
\label{sec:3.2}

Let $D^{\max}_{ts}=\max_{s<u\le t}(\lvl_s-\fld_u)_+$ denote the deepest dip after key~$s$ and $D^{\mathrm{sum}}_{ts}=\sum_{s<u<t}(\lvl_s-\fld_u)_+$ the sum of its dips, where $(x)_+=\max(0,x)$.
Selective Attention subtracts from the logit of key~$s$ the selection $F_{ts}$ that the tokens between $s$ and~$t$ assign to it \citep{leviathan2025selective}.

\begin{proposition}[Depth, not duration]
  \label{prop:3.4}
  The first-passage bias is $b_{ts}=-\str D^{\max}_{ts}$, whereas Selective Attention with a selection head of the form $\sasc(\lvl_s-\fld_u)$, $\sasc>0$, subtracts $F_{ts}=\sasc D^{\mathrm{sum}}_{ts}$ for every key after the first token.
  \begin{enumerate}
    \item[(i)] Inserting into $(s,t)$ steps whose dips below $\lvl_s$ are at most $D^{\max}_{ts}$, while every other position keeps its values, leaves $b_{ts}$ unchanged.
    For a key after the first token it increases $F_{ts}$ by $\sasc$ times the sum of the inserted dips.
    \item[(ii)] For every depth $\dep>0$ and every $m\ge1$ there are inputs with $D^{\max}_{ts}=\dep$ and $D^{\mathrm{sum}}_{ts}=m\dep$.
    \item[(iii)] Let $\agg$ map every finite non-empty multiset of non-negative reals to a real number. If $\agg$ does not change when any value at most the maximum is added, then $\agg$ depends only on the maximum, and every function of the maximum has this property.
  \end{enumerate}
\end{proposition}

Part~(iii) makes functions of the maximum the only aggregations of dips that do not change when a dip no deeper than the deepest is added or held longer.
When the durations of the dips vary within the training data, the difference matters already for fitting it, not only for generalising.

\begin{proposition}[Durations that vary in training]
  \label{prop:3.5}
  Let a sequence contain $r\ge2$ keys at increasing levels at least five apart, each followed with positive probability by a dip of depth $\dep\in\{1,2,3,4\}$.
  The dip consists of three steps at the depths $\dep$, $\dep-1$ and~$0$ in any order, and its deepest step is held for one or two steps.
  Every combination of the level of the first key, depth, order and duration occurs with positive probability.
  Let the target weight of key~$j$ at the query be proportional to $e^{-\shp\dep_j}$ with $\shp>0$, where $\dep_j=0$ without a dip.
  (a)~No head whose logit for key~$s$ is $\tokfnB(x_s)-\sum_{s<u<t}\pairfn(x_u,x_s)$, where $\tokfnB$ and $\pairfn\ge0$ read only the tokens, realises the targets exactly on the keys of all such sequences.
  (b)~One FPA head realises them exactly, and its bias does not change when a step is held longer or a further dip of a key, no deeper than its own, is inserted before the next key.
\end{proposition}

The class of part~(a) contains Selective Attention whose queries and keys depend only on the token, and it contains FoX whose gates, queries and keys read only the token.
The proposition concerns exact fits, and we have not bounded what the class reaches within a tolerance.

\subsection{One field per head}
\label{sec:3.3}

Selective Attention computes one mask per layer from the selection of one head and subtracts it from every head, and DCP also adds one mask per layer \citep{anagnostidis2023dynamic,leviathan2025selective}.
The two-walk family $\twowalk_r$ interleaves two walks in the same tokens (\cref{app:A}).
Every key token raises both walks by one, and the two tokens after key~$j$ follow pattern $P^{(1)}_j$ for the first walk and $P^{(2)}_j$ for the second, as in~$\ladder_r$.
In its complementary part, $P^{(2)}$ is the complement of $P^{(1)}$, so every key except the last is alive under exactly one walk.

\begin{proposition}[A shared mask]
  \label{prop:3.6}
  Let two heads $h\in\{1,2\}$ have the logits $\langle q^{h},k^{h}_s\rangle+\posb^{h}(t,s)+c^{h}-F_s$ at the final query.
  Here $k^{h}_s\in\R^d$ depends only on the token and the position of~$s$, $\posb^h$ is fixed, $q^h$ and $c^h$ depend arbitrarily on the sequence, and the mask $F_s$ is the same for both heads.
  Let head~$h$ follow walk~$h$ on every sequence.
  (a)~At most $N(r,d)=\sum_{i\le2d+1}\binom{r-1}{i}$ sequences of the complementary part have each head rank its alive keys above its dead keys, so all $2^{r-1}$ require $2d\ge r-2$.
  (b)~With tolerance $0<\tol\le1/(2r+2)$, at most $N'(r,d)=\sum_{i\le2d}\binom{r-1}{i}$ of them have each head realise its alive set, so all require $2d\ge r-1$.
\end{proposition}

FPA with one walk per head realises both alive sets on every sequence with one-dimensional dot products (by the proof of \cref{prop:3.3}).
A shared mask therefore at most halves the dimension that \cref{thm:3.1} demands, and it does not remove the dependence on~$r$.
For $r=14$ and $d=2$ the bounds are 2{,}380 and 1{,}093 of the 8{,}192 sequences.
The bound concerns the shared mask, a design choice of Selective Attention and DCP, and Selective Attention with one mask per head lies outside it.

\section{Experiments}
\label{sec:4}

We ran three groups of experiments, namely one-layer probes of the results in \cref{sec:3}, language models, and measurements of cost.
We preregistered every experiment except those that we mark as exploratory, and \cref{app:B,app:C,app:D,app:E} give the full protocols, the records of the hyperparameter searches and the amendments.

\subsection{One-layer probes}
\label{sec:4.1}

\begin{table}[t]
  \caption{\textbf{Seeds out of 15 that succeeded in the one-layer probes.}
    The alive-set columns use the ladder family with head dimension $d=4$, where positional heads realise all alive sets only for $r\le6$.
    The varied durations hold the deepest step of each dip for one or two steps in training, the two walks are the complementary part of the two-walk family ($r=14$, $d=2$ per head), and the stack column uses push and pop sequences.
    A dash marks a pair that we did not run.
    With $d=r+2$ every positional head and FoX succeeded on all seeds.
    Selective Attention$^{\ast}$ lies outside the class of \cref{thm:3.1}.
    FPA, First Passage Attention.
    RoPE, rotary position embedding.
    NoPE, no positional encoding.
    ALiBi, attention with linear biases.
    FoX, Forgetting Transformer.
    DCP, Dynamic Context Pruning.}
  \label{tab:1}
  \centering
  \small
  \setlength{\tabcolsep}{3.2pt}
  \begin{tabular}{@{}lcccccc@{}}
\toprule
 & \multicolumn{3}{c}{Alive sets, $d=4$} & Varied & Two & \\
\cmidrule(lr){2-4}
Method & $r=10$ & $r=14$ & $r=18$ & durations & walks & Stack \\
\midrule
FPA & 11/15 & 12/15 & 13/15 & 13/15 & 11/15 & 14/15 \\
FPA, sum of dips & --- & --- & --- & 0/15 & --- & 11/15 \\
Softmax + RoPE & 0/15 & 0/15 & 0/15 & --- & --- & 0/15 \\
Softmax + NoPE & 0/15 & 0/15 & 0/15 & --- & --- & --- \\
Softmax + ALiBi & 0/15 & 0/15 & 0/15 & --- & --- & --- \\
FoX & 0/15 & 0/15 & 0/15 & --- & --- & 0/15 \\
Selective Attention$^{\ast}$ & 15/15 & 13/15 & 15/15 & 0/15 & 0/15 & 0/15 \\
DCP & --- & --- & --- & --- & 0/15 & --- \\
\bottomrule
\end{tabular}

\end{table}

Each probe was one attention layer whose weights at the final query we supervised directly (\cref{app:B}).
On the alive-set, two-walk and stack families we counted a seed as a success when the head ranked every alive key above every dead key on at least 90\% of the test sequences.
On the varied-duration family we required the weights on the keys to lie within total variation $0.05$ of the target on at least 90\% of the sequences, for both durations.
We tuned each competitor on a matched positive control, and every setting ran on 15 seeds.
In the probes we trained the selection of Selective Attention with a softplus straight-through gradient, which keeps its selection from closing early, whereas the language model used the gradient of its paper.
We compared FPA with each competitor by a one-sided exact McNemar test on the seeds that succeeded, at level $\sig=0.05/7$, and a hypothesis held only when all its comparisons rejected.

\paragraph{Alive sets.}
With $d=4$, \cref{thm:3.1} lets a positional head realise every alive set only for $r\le6$.
We tested softmax heads with rotary position embedding (RoPE), without positional encoding (NoPE) and with linear biases (ALiBi), and FoX heads whose gates read only the token.
They succeeded at $r=6$ on 14 or 15 seeds and failed on every seed for $r=10$, 14 and 18 (\cref{tab:1,tab:5}).
The same heads with $d=r+2$ succeeded on all seeds, so the failures do not come from training.
FPA heads succeeded on 11, 12 and 13 seeds, and all twelve comparisons rejected ($p=0.00049$).
Selective Attention, whose bias depends on later tokens, lies outside the theorem and succeeded as well.
Four-layer softmax and FoX models learned the task for $r=14$ and $r=18$, as \cref{sec:3.1} anticipates for deeper models.

\paragraph{Durations.}
When the deepest step of each dip lasted one or two steps in training (\cref{prop:3.5}), FPA succeeded on 13 seeds.
Selective Attention and FPA with the sum of dips in place of the minimum succeeded on none ($p=0.00012$ for both).
Both competitors fitted the same family with fixed durations on all 15 seeds, so they failed only when the durations varied.

\paragraph{Two walks.}
With two heads of dimension~2 on the complementary part of the two-walk family, FPA succeeded on 11 seeds, and Selective Attention and DCP, whose heads share one mask, on none ($p=0.00049$ each).
Both learned identical walks with one mask (15 and 13 seeds), and Selective Attention with one mask per head learned the complementary part on 13 seeds.
We thus attribute the failures to the shared mask, and no run exceeded the bounds of \cref{prop:3.6}.

\paragraph{Stack and retrieval.}
On push and pop sequences FPA succeeded on 14 seeds and FPA with the sum of dips on 11, while one-layer FoX, softmax and Selective Attention heads succeeded on none.
A further task asked for the occurrences of a name among the alive keys, so it needs both a level per key and retrieval by content (\cref{app:B.6}).
FPA succeeded on all 15 seeds, and variants without the level of the key, with the level of the query instead, or without retrieval by content succeeded on none, as did FoX, stick-breaking attention and softmax attention.
Selective Attention succeeded on 13 seeds.

\paragraph{Tasks from earlier work.}
On six published tasks with small decoders (\cref{app:B.7}), FPA succeeded on every seed in flip-flop language modelling, copying and counting, and on 9 of 15 in variable assignment, where FoX succeeded on all.
In multi-query repeated associative recall no model succeeded on more than 4 seeds.
In multi-query associative recall FPA failed on every seed, where softmax attention succeeded on all.
In our diagnosis, the bias of the first layer suppressed the attention from each value to its own key early in training, so the first step of the recall circuit never formed.
An FPA model without a bias in its first layer succeeded on all 15 seeds, and in an exploratory check three- and four-layer FPA models failed as the two-layer one did.

\subsection{Language models}
\label{sec:4.2}

We trained four decoders with 12 layers of width 768 and 12 heads of dimension 64 on 2.7B tokens of SlimPajama \citep{soboleva2023slimpajama}, following the 125M recipe of FoX \citep{lin2025forgetting} with the 32k-token vocabulary of Mistral 7B \citep{jiang2023mistral}.
The models differed only in attention, namely FPA and FoX without positional encoding, Selective Attention with RoPE, and softmax attention with RoPE.
We trained all of them with context 2048, one seed and mixed precision (\cref{app:C}).
We tested two hypotheses of FPA against FoX, each at level $0.05/7$.
The first asked whether the test loss of FPA is worse by no more than $\ln1.01$, that is 1\% in perplexity.
The second asked whether FPA has a lower loss beyond the training length on PG-19 \citep{rae2020compressive} or retrieves more passkeys placed beyond it \citep{mohtashami2023random}.

\begin{table}[t]
  \caption{\textbf{Language models with 110M parameters, one seed each.}
    Losses are in nats per token on the SlimPajama test set, on WikiText-103 \citep{merity2017pointer} and on PG-19 positions 2048 to 4095 in windows of 4096 tokens.
    Passkey accuracy averages 50 items at each of two lengths, and RULER averages 13 tasks at length 8192.
    The models were trained with context 2048.
    Bold marks the best value in each column.
    FPA, First Passage Attention.
    FoX, Forgetting Transformer.
    RoPE, rotary position embedding.}
  \label{tab:2}
  \centering
  \small
  \setlength{\tabcolsep}{2.6pt}
  \begin{tabular}{@{}lcccccc@{}}
\toprule
 & \multicolumn{3}{c}{Loss (nats per token) $\downarrow$} & \multicolumn{3}{c}{Accuracy $\uparrow$} \\
\cmidrule(lr){2-4}\cmidrule(lr){5-7}
 & SlimPajama & WikiText & PG-19 & Passkey & Passkey & RULER \\
Method & test & 103 & 2k--4k & 3k--4k & 8k--16k & 8k \\
\midrule
FPA & \textbf{2.8640} & \textbf{3.1001} & \textbf{3.282} & \textbf{1.00} & \textbf{1.00} & \textbf{0.202} \\
FoX & 2.8832 & 3.1425 & 3.316 & 0.73 & 0.20 & 0.117 \\
Selective Attention & 2.8686 & 3.1197 & 3.805 & 0.11 & 0.04 & 0.008 \\
Softmax + RoPE & 2.8908 & 3.1374 & 5.064 & 0.00 & 0.00 & 0.002 \\
\bottomrule
\end{tabular}

\end{table}

\paragraph{Perplexity and extrapolation.}
Both tests supported FPA (\cref{tab:2}).
On the SlimPajama test set the loss of FPA was 0.0192 nats below that of FoX, and the upper confidence bound of the difference, $-0.0186$, lay inside the margin of $0.00995$.
On 99 PG-19 test books its loss at positions 2048 to 4095 was 0.034 nats lower, with upper bound $-0.028$.
It retrieved all 100 passkeys at lengths 3072 and 4096, against 73 for FoX ($p=7.5\times10^{-9}$).

\paragraph{Longer contexts and other evaluations.}
In windows of 16,384 tokens FPA and FoX kept a flat loss far beyond the training length, while Selective Attention and softmax attention with RoPE did not (\cref{fig:2}).
FPA was 0.02 to 0.03 nats below FoX up to position 8192 but $0.026$ nats above it from 8192 to 16,384 (95\% interval $[0.019,0.032]$).
FPA retrieved every passkey at lengths 8192 and 16,384, where FoX retrieved 20\% (\cref{fig:3}).
On RULER \citep{hsieh2024ruler} FPA scored 0.202 at length 8192 against 0.117 for FoX, but slightly below FoX at 4096 and 16,384, and no model reached 0.09 on BABILong \citep{kuratov2024babilong}.
Zero-shot accuracies on six downstream tasks were close to those of FoX, and most 95\% intervals of the differences contained zero (\cref{app:C}).

\begin{figure}[t]
  \centering
  \includegraphics{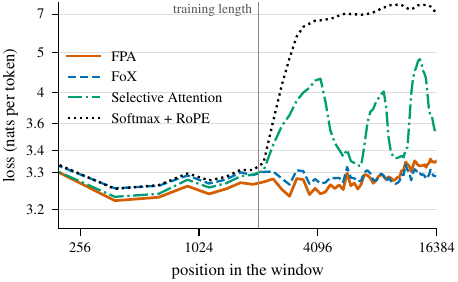}
  \caption{\textbf{Loss far beyond the training length.}
    Mean loss of the four language models on 94 PG-19 test books read in windows of 16,384 tokens, in bins of 256 positions.
    The models were trained with context 2048 (vertical line).
    The vertical axis is logarithmic in the loss minus 3.1 nats, which spreads out the lowest losses.
    FPA, First Passage Attention.
    FoX, Forgetting Transformer.
    RoPE, rotary position embedding.}
  \label{fig:2}
\end{figure}

\begin{figure}[t]
  \centering
  \includegraphics{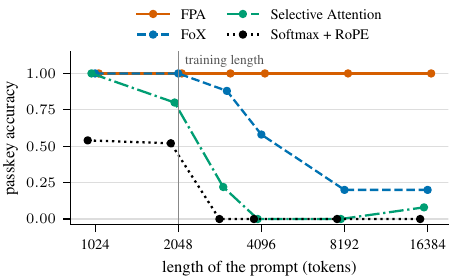}
  \caption{\textbf{Passkey accuracy by length of the prompt.}
    Each point averages 50 items, ten passkeys at each of five depths.
    Lengths up to 4096 come from the preregistered evaluation and the two longer ones from the descriptive one.
    FPA, First Passage Attention.
    FoX, Forgetting Transformer.
    RoPE, rotary position embedding.}
  \label{fig:3}
\end{figure}

\paragraph{What the trained model does.}
In an observational analysis of the FPA model on SlimPajama test windows, most heads left few keys alive (\cref{app:C.6}).
At queries beyond position 1024 the median head left 0.3\% of the keys with a bias of at least $-1$.
In 19 of the 144 heads more than 10\% of the alive sets were not suffixes, which no forget gate can produce.
The share of deaths at the document boundaries $\tokbos$ and $\langle/\mathrm{s}\rangle$ was 6.6 times their share of the tokens, in the median head.
In 56 heads the walk moved in opposite directions at opening and closing brackets and its correlation with bracket depth lay outside a familywise 99\% band of shifted controls (\cref{fig:4}).
The correlations were small (largest $0.25$), so we read them as a description, not as evidence that the model tracks brackets.

\begin{figure}[t]
  \centering
  \includegraphics{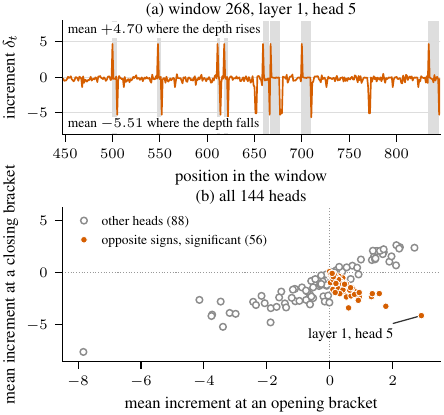}
  \caption{\textbf{Steps of the walk at brackets in the FPA language model.}
    (a)~Increments $\inc_t$ of layer 1, head 5 at positions 447 to 846 of one test window, with the bracket depth shaded and the mean increments where the depth rises and falls in the whole window.
    A fixed seed chose the window before the analysis, and the head contrasts the two kinds of bracket most in it.
    (b)~Mean increments at opening and at closing brackets of all 144 heads over 256 test windows.
    Filled markers show the 56 heads with increments of opposite signs and a familywise significant correlation with the bracket depth.
    FPA, First Passage Attention.}
  \label{fig:4}
\end{figure}

\subsection{Cost}
\label{sec:4.3}

\begin{table}[t]
  \caption{\textbf{Cost on one 12\,GB partition of a Turing GPU in half precision.}
    One attention layer has 12 heads of dimension 64 at batch 4 and length 2048 (medians of 20 rounds).
    Model throughput is in training tokens per second of the 110M models in the benchmark.
    Decoding uses batch size 1 at length 8192, and bytes per position count the keys, values and per-key state of 12 layers in single precision.
    The lower rows switch off parts of the FPA kernel.
    SDPA, the scaled dot-product attention of PyTorch, with its memory-efficient kernel on this GPU.
    FPA, First Passage Attention.
    FoX, Forgetting Transformer.}
  \label{tab:3}
  \centering
  \small
  \setlength{\tabcolsep}{2.4pt}
  \begin{tabular}{@{}lccccc@{}}
\toprule
 & \multicolumn{2}{c}{One layer, ms} & Model & \multicolumn{2}{c}{Decoding at 8k} \\
\cmidrule(lr){2-3}\cmidrule(lr){5-6}
Method & forward & fwd + bwd & tokens/s & ms/token & bytes/pos. \\
\midrule
FPA & 0.78 & 4.31 & 29{,}694 & 22.0 & 74{,}880 \\
FoX & 0.74 & 2.87 & 32{,}481 & 20.2 & 74{,}304 \\
Softmax (SDPA) & 0.87 & 4.42 & 26{,}913 & 23.9 & 73{,}728 \\
Selective Attention & 2.00 & 10.45 & 20{,}150 & 36.7 & 73{,}776 \\
\midrule
\multicolumn{6}{@{}l}{The same kernel with parts of FPA switched off} \\
\quad FoX bias & 0.66 & 3.01 & --- & --- & --- \\
\quad $+$ range minimum & 0.78 & 3.43 & --- & --- & --- \\
\quad $+$ gradient of the field & 0.78 & 4.33 & --- & --- & --- \\
\bottomrule
\end{tabular}

\end{table}

We measured the cost of the kernels and models (\cref{tab:3}, \cref{app:E}).
For the forward and backward pass, one FPA layer took 0.97 times as long as SDPA and 1.50 times as long as FoX, whose bias needs no range minimum.
On this GPU SDPA runs its memory-efficient kernel.
Over five shapes from $16\times512$ to $1\times8192$ these ratios ranged from 0.95 to 1.08 and from 1.42 to 1.56.
Selective Attention took 2.4 times as long as FPA, and 2.0 to 2.5 times over the five shapes.
The forward pass of FPA alone took 1.03 to 1.07 times as long as that of FoX.
Within one kernel, the range minimum added 14\% to the time of the bias of FoX and the gradient of the field another 26\%.
In the benchmark of \cref{tab:3}, the whole FPA model trained at 0.914 times the throughput of FoX, 1.103 times that of softmax attention and 1.47 times that of Selective Attention.
At batch size 1 and length 8192, decoding with FPA took 22.0\,ms per token, against 20.2\,ms for FoX and 23.9\,ms for softmax attention.
These times did not grow from 1024 to 16,384 tokens, so fixed overheads dominate them at this size.
The state of FPA enlarged the cache by 1.56\% over keys and values.

The bias of a key never rises again (\cref{prop:2.3}), so a decoder can drop a key once its bias falls below a threshold.
When we dropped every key whose bias fell below $-10$, the loss of FPA rose by 0.00044 nats at length 2048 and by 0.00047 at 8192, while the model kept 34\% and 26\% of the keys.
With between 24\% and 27\% of its cache the model still retrieved every passkey at 8192 and 16,384.
The same rule cost FoX 0.0195 to 0.0227 nats while it kept 15\% and 7\% of the keys, and a sliding window of the same mean size cost softmax attention 0.115 nats at 2048.

\section{Related Work}
\label{sec:5}

\paragraph{Path-dependent biases.}
FoX adds a sum of log forget gates \citep{lin2025forgetting}, and it is the stack form of FPA with a walk that never rises.
Its bias falls with the age of a key, so its visible sets are suffixes.
Stick-breaking attention \citep{tan2025stickbreaking}, the geometric attention of the Neural Data Router \citep{csordas2022neural}, contextual position encoding \citep{golovneva2024cope}, PaTH \citep{yang2025path} and GRAPE \citep{zhang2026grape} change logits through products, sums or transformations along the path.
GRAPE contains FoX and ALiBi as special cases.
Selective Attention \citep{leviathan2025selective} and DCP \citep{anagnostidis2023dynamic} let later tokens remove a key, and in the hard limit they realise the same alive sets as FPA.
Their bias accumulates a score for every pair of a later token and a key, whereas the bias of FPA depends only on scalars of single positions and costs less to compute than that of Selective Attention (\cref{sec:4.3}).
The suppression of Selective Attention also adds up over the path, so it depends on the duration of a dip (\cref{prop:3.4}), and in both methods one mask serves all heads of a layer (\cref{prop:3.6}).
The concurrent Universal Attention also accumulates a pairwise sum along the path and evicts keys whose bias passes a threshold \citep{wertheimer2026selfpruning}.

\paragraph{Range extremes and stacks.}
StructFormer compares a maximum over a range with a level through a sigmoid, the smooth form of the same comparison \citep{shen2021structformer}.
Its level belongs to the token that attends, however, so a fixed token sees a contiguous span, and the comparison feeds a parser that the whole layer shares in a masked language model.
PRPN \citep{shen2018neural} and syntax-guided local attention \citep{hou2022syntax} also gate by a range extreme against a threshold of the query, and our variant of FPA with the level of the query learned only suffixes (\cref{app:B}).
The strengths of the Neural Stack have the first-passage form as a closed form, in an external memory of a recurrent network that reads only its top \citep{grefenstette2015transduce}.
Ordered Neurons \citep{shen2019ordered_neurons}, Ordered Memory \citep{shen2019ordered}, Stack Attention \citep{dusell2024stack,li2024transformer}, Pushdown Layers \citep{murty2023pushdown}, Transformer Grammars \citep{sartran2022transformer} and DeltaStack \citep{nguyen2026deltastack} build stacks or syntax into recurrent or attention models, several with supervision of the structure.
FPA learns its walk without such supervision.
Expire-Span \citep{sukhbaatar2021expire} and adaptive attention spans \citep{sukhbaatar2019adaptive} learn when keys expire, which FPA contains as the case of a clock as field.
Preisach Attention studies extremum stacks as a sequence memory, without a learned level per key or scores between queries and keys \citep{frydrych2026preisach}.
Tropical Attention takes max-plus extremes over keys rather than over the path between key and query \citep{hashemi2025tropical}.

\paragraph{Expressivity.}
Lower bounds on head dimension, rank or precision \citep{bhojanapalli2020lowrank,sanford2023representational,amsel2025quality,kozachinskiy2025strassen,weller2026theoretical} rest on rank, Vapnik--Chervonenkis (VC) dimension or sign-rank arguments of the kind that \cref{thm:3.1} uses, and our contribution is the family of alive sets that later dips create.
One layer needs linear size for bounded Dyck languages \citep{bhattamishra2024separations}, while a few layers suffice \citep{yao2021selfattention,hayakawa2025}, which is why we state the separation for one layer.
Soft attention simulates hard attention when its temperature scales with the margin \citep{yang2026tacl}, as \cref{prop:3.3} does with its strength.
Non-negative gates limit state tracking in linear recurrences \citep{sarrof2024expressive,grazzi2025unlocking}, and the walk of FPA may rise as well as fall.

\paragraph{Eviction.}
H2O \citep{zhang2023h2o} and attention sinks \citep{xiao2024efficient} drop cached keys by heuristics, and adaptive computation pruning skips the blocks that the bias of FoX makes negligible \citep{lin2025adaptive}.
The bias of an FPA key never rises again, so a decoder can drop a key once its bias falls below a threshold, and in the hard limit it drops exactly the dead keys.

\section{Limitations and Conclusion}
\label{sec:6}

\paragraph{Limitations.}
Our language models had 110M parameters and one seed each, and they trained on 2.7B tokens on one 12\,GB partition of a Turing GPU, so the comparison may change at larger scale.

\paragraph{Conclusion.}
We have introduced First Passage Attention, which lowers the logit of a key by the deepest later dip of a learned field below the level of the key.
One parameterisation contains the Forgetting Transformer, a stack, a bottleneck bias and learned expiry.
In one layer a single FPA head realises alive sets for which positional heads and forget-gate heads with token-level gates need $r-2$ dimensions.
Without positional information, no head that subtracts a sum of non-negative scores of token pairs fits exactly the targets of dips whose durations vary, while one FPA head does.
Preregistered probes agreed with these predictions.
A 110M-parameter FPA language model was not worse than FoX in perplexity, had a lower loss on PG-19 at positions 2048 to 4095 and retrieved every passkey at lengths up to 16,384.

\paragraph{Use of large language models.}
We used large language models to assist with code, literature search, proof checking and editing, and we take full responsibility for the content of the paper.

\bibliographystyle{plainnat}
\bibliography{refs}

@inproceedings{vaswani2017attention,
  author    = {Vaswani, Ashish and Shazeer, Noam and Parmar, Niki and Uszkoreit, Jakob and Jones, Llion and Gomez, Aidan N. and Kaiser, {\L}ukasz and Polosukhin, Illia},
  title     = {Attention is All you Need},
  booktitle = {Advances in Neural Information Processing Systems},
  volume    = {30},
  year      = {2017}
}

@article{su2024roformer,
  author  = {Su, Jianlin and Ahmed, Murtadha and Lu, Yu and Pan, Shengfeng and Bo, Wen and Liu, Yunfeng},
  title   = {{RoFormer}: Enhanced Transformer with Rotary Position Embedding},
  journal = {Neurocomputing},
  volume  = {568},
  pages   = {127063},
  year    = {2024},
  doi     = {10.1016/j.neucom.2023.127063}
}

@inproceedings{press2022train,
  author    = {Press, Ofir and Smith, Noah and Lewis, Mike},
  title     = {Train Short, Test Long: Attention with Linear Biases Enables Input Length Extrapolation},
  booktitle = {International Conference on Learning Representations},
  year      = {2022}
}

@inproceedings{lin2025forgetting,
  author    = {Lin, Zhixuan and Nikishin, Evgenii and He, Xu and Courville, Aaron},
  title     = {Forgetting Transformer: Softmax Attention with a Forget Gate},
  booktitle = {International Conference on Learning Representations},
  pages     = {69704--69738},
  year      = {2025}
}

@inproceedings{tan2025stickbreaking,
  author    = {Tan, Shawn and Yang, Songlin and Courville, Aaron and Panda, Rameswar and Shen, Yikang},
  title     = {Scaling Stick-Breaking Attention: An Efficient Implementation and In-depth Study},
  booktitle = {International Conference on Learning Representations},
  pages     = {4030--4046},
  year      = {2025}
}

@inproceedings{leviathan2025selective,
  author    = {Leviathan, Yaniv and Kalman, Matan and Matias, Yossi},
  title     = {Selective Attention Improves Transformer},
  booktitle = {International Conference on Learning Representations},
  pages     = {4389--4414},
  year      = {2025}
}

@inproceedings{anagnostidis2023dynamic,
  author    = {Anagnostidis, Sotiris and Pavllo, Dario and Biggio, Luca and Noci, Lorenzo and Lucchi, Aurelien and Hofmann, Thomas},
  title     = {Dynamic Context Pruning for Efficient and Interpretable Autoregressive Transformers},
  booktitle = {Advances in Neural Information Processing Systems},
  volume    = {36},
  pages     = {65202--65223},
  year      = {2023},
  doi       = {10.52202/075280-2845}
}

@article{lindley1952theory,
  author  = {Lindley, D. V.},
  title   = {The theory of queues with a single server},
  journal = {Mathematical Proceedings of the Cambridge Philosophical Society},
  volume  = {48},
  number  = {2},
  pages   = {277--289},
  year    = {1952},
  doi     = {10.1017/S0305004100027638}
}

@article{skorokhod1961stochastic,
  author  = {Skorokhod, A. V.},
  title   = {Stochastic Equations for Diffusion Processes in a Bounded Region},
  journal = {Theory of Probability and Its Applications},
  volume  = {6},
  number  = {3},
  pages   = {264--274},
  year    = {1961},
  doi     = {10.1137/1106035}
}

@inproceedings{grefenstette2015transduce,
  author    = {Grefenstette, Edward and Hermann, Karl Moritz and Suleyman, Mustafa and Blunsom, Phil},
  title     = {Learning to Transduce with Unbounded Memory},
  booktitle = {Advances in Neural Information Processing Systems},
  volume    = {28},
  year      = {2015}
}

@inproceedings{sukhbaatar2021expire,
  author    = {Sukhbaatar, Sainbayar and Ju, Da and Poff, Spencer and Roller, Stephen and Szlam, Arthur and Weston, Jason and Fan, Angela},
  title     = {Not All Memories are Created Equal: Learning to Forget by Expiring},
  booktitle = {Proceedings of the 38th International Conference on Machine Learning},
  series    = {Proceedings of Machine Learning Research},
  volume    = {139},
  pages     = {9902--9912},
  year      = {2021}
}

@inproceedings{zhang2026grape,
  author    = {Zhang, Yifan and Chen, Zixiang and Liu, Yifeng and Qin, Zhen and Yuan, Huizhuo and Xu, Kangping and Yuan, Yang and Gu, Quanquan and Yao, Andrew Chi-Chih},
  title     = {Group Representational Position Encoding},
  booktitle = {International Conference on Learning Representations},
  pages     = {56496--56519},
  year      = {2026}
}

@inproceedings{bender2000lca,
  author    = {Bender, Michael A. and Farach-Colton, Mart{\'i}n},
  title     = {The {LCA} Problem Revisited},
  booktitle = {{LATIN} 2000: Theoretical Informatics},
  series    = {Lecture Notes in Computer Science},
  pages     = {88--94},
  year      = {2000},
  doi       = {10.1007/10719839_9}
}

@inproceedings{wang2025flashmask,
  author    = {Wang, Guoxia and Zeng, Jinle and Xiao, Xiyuan and Wu, Siming and Yang, Jiabin and Zheng, Lujing and Chen, Zeyu and Bian, Jiang and Yu, Dianhai and Wang, Haifeng},
  title     = {{FlashMask}: Efficient and Rich Mask Extension of {FlashAttention}},
  booktitle = {International Conference on Learning Representations},
  pages     = {27641--27661},
  year      = {2025}
}

@article{cover1965geometrical,
  author  = {Cover, Thomas M.},
  title   = {Geometrical and Statistical Properties of Systems of Linear Inequalities with Applications in Pattern Recognition},
  journal = {IEEE Transactions on Electronic Computers},
  volume  = {EC-14},
  number  = {3},
  pages   = {326--334},
  year    = {1965},
  doi     = {10.1109/PGEC.1965.264137}
}

@article{dudley1978central,
  author  = {Dudley, R. M.},
  title   = {Central Limit Theorems for Empirical Measures},
  journal = {The Annals of Probability},
  volume  = {6},
  number  = {6},
  pages   = {899--929},
  year    = {1978},
  doi     = {10.1214/aop/1176995384}
}

@inproceedings{bhojanapalli2020lowrank,
  author    = {Bhojanapalli, Srinadh and Yun, Chulhee and Rawat, Ankit Singh and Reddi, Sashank and Kumar, Sanjiv},
  title     = {Low-Rank Bottleneck in Multi-head Attention Models},
  booktitle = {Proceedings of the 37th International Conference on Machine Learning},
  series    = {Proceedings of Machine Learning Research},
  volume    = {119},
  pages     = {864--873},
  year      = {2020}
}

@inproceedings{sanford2023representational,
  author    = {Sanford, Clayton and Hsu, Daniel J. and Telgarsky, Matus},
  title     = {Representational Strengths and Limitations of Transformers},
  booktitle = {Advances in Neural Information Processing Systems},
  volume    = {36},
  pages     = {36677--36707},
  year      = {2023},
  doi       = {10.52202/075280-1596}
}

@inproceedings{weller2026theoretical,
  author    = {Weller, Orion and Boratko, Michael and Naim, Iftekhar and Lee, Jinhyuk},
  title     = {On the Theoretical Limitations of Embedding-Based Retrieval},
  booktitle = {International Conference on Learning Representations},
  pages     = {11202--11216},
  year      = {2026}
}

@inproceedings{song2026causal,
  author    = {Song, Zhuoqing and Sun, Peng and Yuan, Huizhuo and Gu, Quanquan},
  title     = {Causal Attention with Lookahead Keys},
  booktitle = {Proceedings of the 43rd International Conference on Machine Learning},
  series    = {Proceedings of Machine Learning Research},
  volume    = {306},
  pages     = {115068--115107},
  year      = {2026}
}

@inproceedings{yao2021selfattention,
  author    = {Yao, Shunyu and Peng, Binghui and Papadimitriou, Christos and Narasimhan, Karthik},
  title     = {Self-Attention Networks Can Process Bounded Hierarchical Languages},
  booktitle = {Proceedings of the 59th Annual Meeting of the Association for Computational Linguistics and the 11th International Joint Conference on Natural Language Processing (Volume 1: Long Papers)},
  pages     = {3770--3785},
  year      = {2021},
  doi       = {10.18653/v1/2021.acl-long.292}
}

@inproceedings{hayakawa2025,
  author    = {Hayakawa, Daichi and Sato, Issei},
  title     = {Theoretical Analysis of Hierarchical Language Recognition and Generation by Transformers without Positional Encoding},
  booktitle = {Proceedings of the 63rd Annual Meeting of the Association for Computational Linguistics (Volume 1: Long Papers)},
  pages     = {30777--30834},
  year      = {2025},
  doi       = {10.18653/v1/2025.acl-long.1488}
}

@inproceedings{kajitsuka2024transformers,
  author    = {Kajitsuka, Tokio and Sato, Issei},
  title     = {Are Transformers with One Layer Self-Attention Using Low-Rank Weight Matrices Universal Approximators?},
  booktitle = {International Conference on Learning Representations},
  pages     = {35177--35205},
  year      = {2024}
}

@inproceedings{csordas2022neural,
  author    = {Csord{\'a}s, R{\'o}bert and Irie, Kazuki and Schmidhuber, J{\"u}rgen},
  title     = {The Neural Data Router: Adaptive Control Flow in Transformers Improves Systematic Generalization},
  booktitle = {International Conference on Learning Representations},
  year      = {2022}
}

@article{golovneva2024cope,
  author  = {Golovneva, Olga and Wang, Tianlu and Weston, Jason and Sukhbaatar, Sainbayar},
  title   = {Contextual Position Encoding: Learning to Count What's Important},
  journal = {arXiv preprint arXiv:2405.18719},
  year    = {2024}
}

@inproceedings{yang2025path,
  author    = {Yang, Songlin and Shen, Yikang and Wen, Kaiyue and Tan, Shawn and Mishra, Mayank and Ren, Liliang and Panda, Rameswar and Kim, Yoon},
  title     = {{PaTH} Attention: Position Encoding via Accumulating {Householder} Transformations},
  booktitle = {Advances in Neural Information Processing Systems},
  volume    = {38},
  pages     = {62220--62247},
  year      = {2025},
  doi       = {10.52202/085713-2080}
}

@article{wertheimer2026selfpruning,
  author  = {Wertheimer, Davis and Shen, Haochen and Gupta, Ahan and Liu, Derrick and Lim, Yu Chin Fabian and Srivatsa, Mudhakar and Ganti, Raghu K. and Zhang, Minjia and Wang, Naigang},
  title   = {A Self-Pruning Transformer: Extreme {KV}-Cache Compression with Universal Attention},
  journal = {arXiv preprint arXiv:2610.09051},
  year    = {2026}
}

@inproceedings{shen2021structformer,
  author    = {Shen, Yikang and Tay, Yi and Zheng, Che and Bahri, Dara and Metzler, Donald and Courville, Aaron},
  title     = {{StructFormer}: Joint Unsupervised Induction of Dependency and Constituency Structure from Masked Language Modeling},
  booktitle = {Proceedings of the 59th Annual Meeting of the Association for Computational Linguistics and the 11th International Joint Conference on Natural Language Processing (Volume 1: Long Papers)},
  pages     = {7196--7209},
  year      = {2021},
  doi       = {10.18653/v1/2021.acl-long.559}
}

@inproceedings{shen2018neural,
  author    = {Shen, Yikang and Lin, Zhouhan and Huang, Chin-wei and Courville, Aaron},
  title     = {Neural Language Modeling by Jointly Learning Syntax and Lexicon},
  booktitle = {International Conference on Learning Representations},
  year      = {2018}
}

@inproceedings{hou2022syntax,
  author    = {Hou, Shengyuan and Kai, Jushi and Xue, Haotian and Zhu, Bingyu and Yuan, Bo and Huang, Longtao and Wang, Xinbing and Lin, Zhouhan},
  title     = {Syntax-guided Localized Self-attention by Constituency Syntactic Distance},
  booktitle = {Findings of the Association for Computational Linguistics: {EMNLP} 2022},
  pages     = {2334--2341},
  year      = {2022},
  doi       = {10.18653/v1/2022.findings-emnlp.173}
}

@inproceedings{shen2019ordered_neurons,
  author    = {Shen, Yikang and Tan, Shawn and Sordoni, Alessandro and Courville, Aaron},
  title     = {Ordered Neurons: Integrating Tree Structures into Recurrent Neural Networks},
  booktitle = {International Conference on Learning Representations},
  year      = {2019}
}

@inproceedings{shen2019ordered,
  author    = {Shen, Yikang and Tan, Shawn and Hosseini, Arian and Lin, Zhouhan and Sordoni, Alessandro and Courville, Aaron},
  title     = {Ordered Memory},
  booktitle = {Advances in Neural Information Processing Systems},
  volume    = {32},
  year      = {2019}
}

@inproceedings{dusell2024stack,
  author    = {DuSell, Brian and Chiang, David},
  title     = {Stack Attention: Improving the Ability of Transformers to Model Hierarchical Patterns},
  booktitle = {International Conference on Learning Representations},
  pages     = {7411--7430},
  year      = {2024}
}

@inproceedings{li2024transformer,
  author    = {Li, Jiaoda and White, Jennifer and Sachan, Mrinmaya and Cotterell, Ryan},
  title     = {A Transformer with Stack Attention},
  booktitle = {Findings of the Association for Computational Linguistics: {NAACL} 2024},
  pages     = {4318--4335},
  year      = {2024},
  doi       = {10.18653/v1/2024.findings-naacl.269}
}

@inproceedings{murty2023pushdown,
  author    = {Murty, Shikhar and Sharma, Pratyusha and Andreas, Jacob and Manning, Christopher D.},
  title     = {Pushdown Layers: Encoding Recursive Structure in Transformer Language Models},
  booktitle = {Proceedings of the 2023 Conference on Empirical Methods in Natural Language Processing},
  pages     = {3233--3247},
  year      = {2023},
  doi       = {10.18653/v1/2023.emnlp-main.195}
}

@article{sartran2022transformer,
  author  = {Sartran, Laurent and Barrett, Samuel and Kuncoro, Adhiguna and Stanojevi{\'c}, Milo{\v{s}} and Blunsom, Phil and Dyer, Chris},
  title   = {Transformer Grammars: Augmenting Transformer Language Models with Syntactic Inductive Biases at Scale},
  journal = {Transactions of the Association for Computational Linguistics},
  volume  = {10},
  pages   = {1423--1439},
  year    = {2022},
  doi     = {10.1162/tacl_a_00526}
}

@inproceedings{nguyen2026deltastack,
  author    = {Nguyen, Anh T and Momeni, Saleh and Chaubey, Ashutosh and Xiao, Changnan and Liu, Bing},
  title     = {Push, Pop, Parallelize: Stack-Augmented Linear Attention via the Delta Rule},
  booktitle = {Proceedings of the 43rd International Conference on Machine Learning},
  series    = {Proceedings of Machine Learning Research},
  volume    = {306},
  pages     = {92579--92592},
  year      = {2026}
}

@inproceedings{sukhbaatar2019adaptive,
  author    = {Sukhbaatar, Sainbayar and Grave, {\'E}douard and Bojanowski, Piotr and Joulin, Armand},
  title     = {Adaptive Attention Span in Transformers},
  booktitle = {Proceedings of the 57th Annual Meeting of the Association for Computational Linguistics},
  pages     = {331--335},
  year      = {2019},
  doi       = {10.18653/v1/P19-1032}
}

@article{frydrych2026preisach,
  author  = {Frydrych, Piotr},
  title   = {Preisach Attention: A Hysteretic Model of Sequential Memory},
  journal = {arXiv preprint arXiv:2605.23603},
  year    = {2026}
}

@inproceedings{hashemi2025tropical,
  author    = {Hashemi, Baran and Pasque, Kurt and Teska, Chris and Yoshida, Ruriko},
  title     = {Tropical Attention: Neural Algorithmic Reasoning for Combinatorial Algorithms},
  booktitle = {Advances in Neural Information Processing Systems},
  volume    = {38},
  pages     = {98631--98656},
  year      = {2025},
  doi       = {10.52202/085713-3295}
}

@inproceedings{amsel2025quality,
  author    = {Amsel, Noah and Yehudai, Gilad and Bruna, Joan},
  title     = {Quality over Quantity in Attention Layers: When Adding More Heads Hurts},
  booktitle = {International Conference on Learning Representations},
  pages     = {62431--62483},
  year      = {2025}
}

@inproceedings{kozachinskiy2025strassen,
  author    = {Kozachinskiy, Alexander and Urrutia, Felipe and Orellana, Hector and Steifer, Tomasz and Pizarro, Germ{\'a}n and Fuentes, Mat{\'i}as and Meza V{\'a}squez, Francisco and Buc Calderon, Cristian and Rojas, Cristobal},
  title     = {Strassen Attention, Split {VC} Dimension and Compositionality in Transformers},
  booktitle = {Advances in Neural Information Processing Systems},
  volume    = {38},
  pages     = {9665--9696},
  year      = {2025},
  doi       = {10.52202/085713-0327}
}

@inproceedings{bhattamishra2024separations,
  author    = {Bhattamishra, Satwik and Hahn, Michael and Blunsom, Phil and Kanade, Varun},
  title     = {Separations in the Representational Capabilities of Transformers and Recurrent Architectures},
  booktitle = {Advances in Neural Information Processing Systems},
  volume    = {37},
  pages     = {36002--36045},
  year      = {2024},
  doi       = {10.52202/079017-1135}
}

@article{yang2026tacl,
  author  = {Yang, Andy and Strobl, Lena and Chiang, David and Angluin, Dana},
  title   = {Simulating Hard Attention Using Soft Attention},
  journal = {Transactions of the Association for Computational Linguistics},
  volume  = {14},
  pages   = {147--166},
  year    = {2026},
  doi     = {10.1162/tacl.a.597}
}

@inproceedings{sarrof2024expressive,
  author    = {Sarrof, Yash and Veitsman, Yana and Hahn, Michael},
  title     = {The Expressive Capacity of State Space Models: A Formal Language Perspective},
  booktitle = {Advances in Neural Information Processing Systems},
  volume    = {37},
  pages     = {41202--41241},
  year      = {2024},
  doi       = {10.52202/079017-1304}
}

@inproceedings{grazzi2025unlocking,
  author    = {Grazzi, Riccardo and Siems, Julien and Zela, Arber and Franke, J{\"o}rg and Hutter, Frank and Pontil, Massimiliano},
  title     = {Unlocking State-Tracking in Linear {RNNs} Through Negative Eigenvalues},
  booktitle = {International Conference on Learning Representations},
  pages     = {36565--36597},
  year      = {2025}
}

@inproceedings{lin2025adaptive,
  author    = {Lin, Zhixuan and Obando-Ceron, Johan and He, Xu Owen and Courville, Aaron},
  title     = {Adaptive Computation Pruning for the Forgetting Transformer},
  booktitle = {Second Conference on Language Modeling},
  year      = {2025}
}

@inproceedings{liu2023exposing,
  author    = {Liu, Bingbin and Ash, Jordan and Goel, Surbhi and Krishnamurthy, Akshay and Zhang, Cyril},
  title     = {Exposing Attention Glitches with Flip-Flop Language Modeling},
  booktitle = {Advances in Neural Information Processing Systems},
  volume    = {36},
  pages     = {25549--25583},
  year      = {2023},
  doi       = {10.52202/075280-1112}
}

@inproceedings{arora2024zoology,
  author    = {Arora, Simran and Eyuboglu, Sabri and Timalsina, Aman and Johnson, Isys and Poli, Michael and Zou, James Y and Rudra, Atri and R{\'e}, Christopher},
  title     = {Zoology: Measuring and Improving Recall in Efficient Language Models},
  booktitle = {International Conference on Learning Representations},
  pages     = {15664--15730},
  year      = {2024}
}

@inproceedings{paperno2016lambada,
  author    = {Paperno, Denis and Kruszewski, Germ{\'a}n and Lazaridou, Angeliki and Pham, Ngoc Quan and Bernardi, Raffaella and Pezzelle, Sandro and Baroni, Marco and Boleda, Gemma and Fern{\'a}ndez, Raquel},
  title     = {The {LAMBADA} dataset: Word prediction requiring a broad discourse context},
  booktitle = {Proceedings of the 54th Annual Meeting of the Association for Computational Linguistics (Volume 1: Long Papers)},
  pages     = {1525--1534},
  year      = {2016},
  doi       = {10.18653/v1/P16-1144}
}

@inproceedings{zellers2019hellaswag,
  author    = {Zellers, Rowan and Holtzman, Ari and Bisk, Yonatan and Farhadi, Ali and Choi, Yejin},
  title     = {{HellaSwag}: Can a Machine Really Finish Your Sentence?},
  booktitle = {Proceedings of the 57th Annual Meeting of the Association for Computational Linguistics},
  pages     = {4791--4800},
  year      = {2019},
  doi       = {10.18653/v1/P19-1472}
}

@inproceedings{bisk2020piqa,
  author    = {Bisk, Yonatan and Zellers, Rowan and Le Bras, Ronan and Gao, Jianfeng and Choi, Yejin},
  title     = {{PIQA}: Reasoning about Physical Commonsense in Natural Language},
  booktitle = {Proceedings of the AAAI Conference on Artificial Intelligence},
  volume    = {34},
  number    = {05},
  pages     = {7432--7439},
  year      = {2020},
  doi       = {10.1609/aaai.v34i05.6239}
}

@article{clark2018arc,
  author  = {Clark, Peter and Cowhey, Isaac and Etzioni, Oren and Khot, Tushar and Sabharwal, Ashish and Schoenick, Carissa and Tafjord, Oyvind},
  title   = {Think you have Solved Question Answering? {Try} {ARC}, the {AI2} Reasoning Challenge},
  journal = {arXiv preprint arXiv:1803.05457},
  year    = {2018}
}

@inproceedings{sakaguchi2020winogrande,
  author    = {Sakaguchi, Keisuke and Le Bras, Ronan and Bhagavatula, Chandra and Choi, Yejin},
  title     = {{WinoGrande}: An Adversarial {Winograd} Schema Challenge at Scale},
  booktitle = {Proceedings of the AAAI Conference on Artificial Intelligence},
  volume    = {34},
  number    = {05},
  pages     = {8732--8740},
  year      = {2020},
  doi       = {10.1609/aaai.v34i05.6399}
}

@article{sauer1972density,
  author  = {Sauer, N.},
  title   = {On the density of families of sets},
  journal = {Journal of Combinatorial Theory, Series A},
  volume  = {13},
  number  = {1},
  pages   = {145--147},
  year    = {1972},
  doi     = {10.1016/0097-3165(72)90019-2}
}

@article{shelah1972combinatorial,
  author  = {Shelah, Saharon},
  title   = {A combinatorial problem; stability and order for models and theories in infinitary languages},
  journal = {Pacific Journal of Mathematics},
  volume  = {41},
  number  = {1},
  pages   = {247--261},
  year    = {1972},
  doi     = {10.2140/pjm.1972.41.247}
}

@inproceedings{mohtashami2023random,
  author    = {Mohtashami, Amirkeivan and Jaggi, Martin},
  title     = {Random-Access Infinite Context Length for Transformers},
  booktitle = {Advances in Neural Information Processing Systems},
  volume    = {36},
  pages     = {54567--54585},
  year      = {2023},
  doi       = {10.52202/075280-2378}
}

@inproceedings{zhang2023h2o,
  author    = {Zhang, Zhenyu and Sheng, Ying and Zhou, Tianyi and Chen, Tianlong and Zheng, Lianmin and Cai, Ruisi and Song, Zhao and Tian, Yuandong and R{\'e}, Christopher and Barrett, Clark and Wang, Zhangyang and Chen, Beidi},
  title     = {{H2O}: Heavy-Hitter Oracle for Efficient Generative Inference of Large Language Models},
  booktitle = {Advances in Neural Information Processing Systems},
  volume    = {36},
  pages     = {34661--34710},
  year      = {2023},
  doi       = {10.52202/075280-1506}
}

@inproceedings{kuratov2024babilong,
  author    = {Kuratov, Yuri and Bulatov, Aydar and Anokhin, Petr and Rodkin, Ivan and Sorokin, Dmitry and Sorokin, Artyom and Burtsev, Mikhail},
  title     = {{BABILong}: Testing the Limits of {LLMs} with Long Context Reasoning-in-a-Haystack},
  booktitle = {Advances in Neural Information Processing Systems, Datasets and Benchmarks Track},
  volume    = {37},
  pages     = {106519--106554},
  year      = {2024},
  doi       = {10.52202/079017-3381}
}

@inproceedings{hsieh2024ruler,
  author    = {Hsieh, Cheng-Ping and Sun, Simeng and Kriman, Samuel and Acharya, Shantanu and Rekesh, Dima and Jia, Fei and Ginsburg, Boris},
  title     = {{RULER}: What's the Real Context Size of Your Long-Context Language Models?},
  booktitle = {First Conference on Language Modeling},
  year      = {2024}
}

@inproceedings{xiao2024efficient,
  author    = {Xiao, Guangxuan and Tian, Yuandong and Chen, Beidi and Han, Song and Lewis, Mike},
  title     = {Efficient Streaming Language Models with Attention Sinks},
  booktitle = {International Conference on Learning Representations},
  pages     = {21875--21895},
  year      = {2024}
}

@misc{soboleva2023slimpajama,
  author       = {Soboleva, Daria and Al-Khateeb, Faisal and Myers, Robert and Steeves, Jacob R and Hestness, Joel and Dey, Nolan},
  title        = {{SlimPajama}: A 627{B} token cleaned and deduplicated version of {RedPajama}},
  howpublished = {Cerebras blog},
  month        = jun,
  year         = {2023}
}

@misc{gao2024harness,
  author       = {Gao, Leo and Tow, Jonathan and Abbasi, Baber and Biderman, Stella and Black, Sid and DiPofi, Anthony and Foster, Charles and Golding, Laurence and Hsu, Jeffrey and Le Noac'h, Alain and Li, Haonan and McDonell, Kyle and Muennighoff, Niklas and Ociepa, Chris and Phang, Jason and Reynolds, Laria and Schoelkopf, Hailey and Skowron, Aviya and Sutawika, Lintang and Tang, Eric and Thite, Anish and Wang, Ben and Wang, Kevin and Zou, Andy},
  title        = {A framework for few-shot language model evaluation},
  howpublished = {Zenodo, version v0.4.3},
  year         = {2024},
  doi          = {10.5281/zenodo.12608602}
}

@article{jiang2023mistral,
  author  = {Jiang, Albert Q. and Sablayrolles, Alexandre and Mensch, Arthur and Bamford, Chris and Chaplot, Devendra Singh and de las Casas, Diego and Bressand, Florian and Lengyel, Gianna and Lample, Guillaume and Saulnier, Lucile and Lavaud, L{\'e}lio Renard and Lachaux, Marie-Anne and Stock, Pierre and Le Scao, Teven and Lavril, Thibaut and Wang, Thomas and Lacroix, Timoth{\'e}e and El Sayed, William},
  title   = {Mistral {7B}},
  journal = {arXiv preprint arXiv:2310.06825},
  year    = {2023}
}

@inproceedings{rae2020compressive,
  author    = {Rae, Jack W. and Potapenko, Anna and Jayakumar, Siddhant M. and Hillier, Chloe and Lillicrap, Timothy P.},
  title     = {Compressive Transformers for Long-Range Sequence Modelling},
  booktitle = {International Conference on Learning Representations},
  year      = {2020}
}

@inproceedings{merity2017pointer,
  author    = {Merity, Stephen and Xiong, Caiming and Bradbury, James and Socher, Richard},
  title     = {Pointer Sentinel Mixture Models},
  booktitle = {International Conference on Learning Representations},
  year      = {2017}
}

\newpage
\appendix
\crefalias{section}{appendix}
\crefalias{subsection}{subappendix}
\section{Proofs}
\label{app:A}

Throughout, a head gives key~$s$ at query~$t$ the weight $w_{ts}=e^{\lgt_{ts}}/\sum_{u\le t}e^{\lgt_{tu}}$, so weight ratios depend only on logit differences, $w_{ts}/w_{ts'}=e^{\lgt_{ts}-\lgt_{ts'}}$.
All logits are finite unless a statement says otherwise, and all arithmetic is exact.
The field~$\fld_u$ and the level~$\lvl_s$ do not change once their positions are written.

\subsection{The ladder family}
\label{app:A.1}

Fix $r\ge3$.
The tokens are $\tokbos$, the key token~\texttt{K}, the tokens \texttt{D} and \texttt{U}, and the query token~\tokq, with walk increments $\inc(\tokbos)=\inc(\tokq)=0$, $\inc(\texttt{K})=1$, $\inc(\texttt{D})=-\tfrac12$ and $\inc(\texttt{U})=\tfrac12$.
For a pattern $P=(P_1,\dots,P_{r-1})\in\{0,1\}^{r-1}$ the sequence $x^P$ is
\begin{equation*}
  \tokbos,\ \texttt{K},\ \pi_1,\ \texttt{K},\ \pi_2,\ \dots,\ \texttt{K},\ \pi_{r-1},\ \texttt{K},\ \tokq,
\end{equation*}
where the block $\pi_i$ is \texttt{U}\,\texttt{D} if $P_i=1$ and \texttt{D}\,\texttt{U} if $P_i=0$.
With positions numbered from 0, every sequence has length $3r$, key~$j$ sits at $\kpos_j=3j-2$ and the query at $t=3r-1$.
The family $\ladder_r$ has the following four properties, which we also checked by computation for $r$ up to 10.
\begin{enumerate}
  \item \emph{Alive sets.} In the stack form, key $j<r$ is alive at the final query exactly when $P_j=1$, and key~$r$ is always alive. Indeed $W_{\kpos_j}=j$, and within block $\pi_j$ the walk visits $j+\tfrac12$ (pattern \texttt{U}\,\texttt{D}) or $j-\tfrac12$ (pattern \texttt{D}\,\texttt{U}) and returns to~$j$. Every later block starts at least at $j+1$ and dips by at most~$\tfrac12$, so the alive key set is $A^P=\{j<r:P_j=1\}\cup\{r\}$.
  \item \emph{Prefixes.} The prefix of $x^P$ up to position $\kpos_j$ depends only on $P_1,\dots,P_{j-1}$.
  \item \emph{Alive positions.} In the stack form each block with its key holds exactly one position that is dead at the final query, namely key~$j$ if $P_j=0$ and the \texttt{U} of the block if $P_j=1$. With $\tokbos$, key~$r$ and the query alive, there are $r-1$ dead positions and $2r+1$ alive ones.
  \item \emph{Equal multisets.} The blocks \texttt{U}\,\texttt{D} and \texttt{D}\,\texttt{U} are orderings of one multiset. So for any function $\tokfn$ of a token, $\sum_{s<u\le t}\tokfn(x_u)$ with $s$ a key position and $t$ the final query takes the same value for every pattern. In the stack form every dip lies at least $\tfrac12$ below the level of every position it kills.
\end{enumerate}

\subsection{Proof of \texorpdfstring{\cref{prop:2.2}}{Proposition 2.2}}

(a)~With $g=h=0$, $\min\bigl(0,\rmin_{ts}-\lvl_s\bigr)=\min\bigl(0,\min_{s<u\le t}(W_u-W_s)\bigr)=-\max\bigl(0,\max_{s<u\le t}(W_s-W_u)\bigr)=-L_{ts}$.
A queue that starts empty at time~$s$ and receives the net input $\inc_u$ at step~$u$ has the content $R_s=0$ and $R_u=\max(0,R_{u-1}+\inc_u)$ \citep{lindley1952theory}.
Unrolling the recursion gives $R_t=W_t-\min_{s\le v\le t}W_v=(W_t-W_s)+L_{ts}$, so $L_{ts}$ is the total amount by which the reflection at zero has pushed the walk $W-W_s$ up \citep{skorokhod1961stochastic}.
The proposition uses this reading only as a name for $L_{ts}$.

(b)~If $\inc\le0$, the walk does not increase, so $\rmin_{ts}=W_t\le W_s$ and $b_{ts}=\str(W_t-W_s)=\sum_{s<u\le t}\str\inc_u=\sum_{s<u\le t}\log f_u$.

(c)~Here $W=0$ and $\lvl=0$, so $b_{ts}=\min\bigl(0,\min_{s<u\le t}g_u\bigr)=\min_{s<u\le t}g_u$ because $g\le0$.

(d)~The walk is $W_u=W_0-u$, so $\fld=W$ and $\lvl_s=W_s-e_s$.
The minimum of the decreasing walk over $(s,t]$ is $W_t$, so $\rmin_{ts}-\lvl_s=W_t-W_s+e_s=e_s-(t-s)$, which gives the stated bias.
Key~$s$ is alive, for every $\str$, exactly when $e_s\ge t-s$. \hfill$\square$

\subsection{Proof of \texorpdfstring{\cref{prop:2.3}}{Proposition 2.3}}

(a)~Since $(s,t]\subset(s,t']$, $\rmin_{t's}\le\rmin_{ts}$, and $x\mapsto\str\min(0,x-\lvl_s)$ does not decrease.
So $b_{t's}\le b_{ts}$, and $b_{ts}<0$ gives $b_{t's}<0$.

(b)~Let $s<s'$ both be alive at time~$t$.
Since $s'\in(s,t]$, $W_{s'}\ge\rmin_{ts}\ge W_s$, so the levels of the alive keys do not decrease with the time of writing.
Step~$u$ kills a key $s<u-1$ that is alive at $u-1$ exactly when $W_u<W_s$, and it kills the key written at $u-1$ under the same condition.
Since the levels of the keys alive at $u-1$ do not decrease in~$s$, the keys with $W_s>W_u$ among them form a suffix in time.
Every key alive at $u-1$ has $W_s\le W_{u-1}$, so a step that does not decrease the walk kills nothing.
Every key alive at~$u$ has a level of at most~$W_u$, so the key written at~$u$ joins at the top.
The alive keys therefore evolve as a stack. \hfill$\square$

\subsection{Proof of \texorpdfstring{\cref{prop:2.4}}{Proposition 2.4}}

(a)~Let $s\in J$ and $t\in I$ with $J$ before $I$.
The range $(s,t]$ is the disjoint union of the rest of $J$ after~$s$, the full blocks strictly between $J$ and~$I$, and the part of~$I$ up to~$t$.
So $\rmin_{ts}$ is the minimum of the suffix minimum of $J$ after~$s$, the minimum over the blocks in between, and the prefix minimum of $I$ up to~$t$.
Suffix and prefix minima of a block take $O(B)$ operations, and the $O(T/B)$ block minima take $O(T)$ for the whole sequence.
The minimum over the blocks in between costs $O(1)$ per pair of blocks when it runs back along~$J$ from the block before~$I$.
The bias of a pair of blocks is then an element-wise expression in these minima, $\lvl$ and~$\str$.
On a diagonal block the range lies inside the block, and the running minimum $\rmin_{ts}=\min(\rmin_{t-1,s},\fld_t)$ gives all range minima with $B(B-1)/2$ updates.

(b)~Since $(s,t]=(s,t-1]\cup\{t\}$, $\rmin_{ts}=\min(\rmin_{t-1,s},\fld_t)$, and the bias needs only $\rmin_{ts}$ and $\lvl_s$ for each cached key.
Since $\rmin_{ts}$ does not increase in~$t$, a key with $\rmin_{ts}<\lvl_s$ stays dead. \hfill$\square$

\subsection{Proof of \texorpdfstring{\cref{thm:3.1}}{Theorem 3.1}}

At the final query $t=3r-1$ the token and position of key~$j$ are the same in every sequence, so $k_j:=k_{\kpos_j}$ and $\posb_j:=\posb(t,\kpos_j)$ do not depend on the pattern.
Fix a pattern~$P$ and let $q=q^P$.
All weights at the query share one denominator, so the hypothesis orders the logits of the keys, and $\min_{a\in A^P}\lgt_{\kpos_a}>\max_{x\notin A^P}\lgt_{\kpos_x}$, where the maximum over an empty set is $-\infty$.
Pick $\theta_0$ strictly in between and set $\theta=\theta_0-c^P$.
The function $f(j)=\langle q,k_j\rangle+\posb_j-\theta$ is then positive on the alive keys and negative on the dead keys among $j=1,\dots,r-1$.
So the class $\fcls=\{j\mapsto\langle q,k_j\rangle+\posb_j-\theta:q\in\R^d,\theta\in\R\}$ realises every sign pattern on $\{1,\dots,r-1\}$.

Suppose that $r-1\ge d+2$.
The $r-1$ vectors $(k_j,1)\in\R^{d+1}$ are then linearly dependent, $\sum_{j<r}a_j(k_j,1)=0$ with $a\ne0$, and the last coordinate gives $\sum_ja_j=0$, so both $J_+=\{j:a_j>0\}$ and $J_-=\{j:a_j<0\}$ are non-empty.
Every $f\in\fcls$ satisfies $\sum_{j<r}a_jf(j)=\sum_{j<r}a_j\posb_j=:C$, a constant.
If $C\le0$, a function positive exactly on $J_+$ would make the sum positive, and if $C>0$, a function positive exactly on $J_-$ would make it negative.
In both cases some pattern of the family is not realised, a contradiction, so $r-1\le d+1$.

For the path term, the range $(\kpos_j,t]$ consists of the blocks $\pi_j,\dots,\pi_{r-1}$, the keys $j+1,\dots,r$ and the query.
By the fourth property of $\ladder_r$, $\Lambda_j=\sum_{\kpos_j<u\le t}\tokfn(x_u)$ is the same for every pattern.
The proof uses $\posb$ only through the numbers $\posb_j$, which must not depend on the pattern, and replacing them by $\posb_j+\Lambda_j$ keeps that property. \hfill$\square$

This is the classical bound on the VC dimension of the positivity sets of a fixed function plus a $(d+1)$-dimensional space of functions \citep{cover1965geometrical,dudley1978central}.
By the Sauer--Shelah lemma \citep{sauer1972density,shelah1972combinatorial} such a head orders at most $\sum_{i\le d+1}\binom{r-1}{i}$ of the $2^{r-1}$ sequences, fewer than 90\% for $d=4$ and $r\ge9$.
With $d=r-2$ the bound is attained by the keys $k_{\kpos_j}=e_j$ for $j\le r-2$ and $k_{\kpos_{r-1}}=k_{\kpos_r}=0$.
The positional bias is zero on the keys $j<r$, large on the last key and large and negative on the positions without a key, and the query has the coordinate $L$ or $-L$ by the fate of each key.
The path term must depend on the token alone.
If it may also read the position, a head with $d=0$ already orders every sequence of $\ladder_r$.

\subsection{Proof of \texorpdfstring{\cref{thm:3.2}}{Theorem 3.2}}

Let $q^P$, $c^P$, $k^P_s$ and $\lgt^P_s$ denote the query, the offset, the keys and the logits at the final query of~$x^P$.

\begin{lemma}[Gap]
  \label{lem:A.1}
  Let $0<\tol\le1/(2r+2)$ and let the head realise $A^P$ with tolerance~$\tol$ on~$x^P$.
  Then for all alive keys $a,a'$ and every dead key~$x$,
  \begin{equation*}
    |\lgt^P_{\kpos_a}-\lgt^P_{\kpos_{a'}}|\le\gapA=\log\frac{1+(r-1)\tol}{1-\tol},\qquad
    \lgt^P_{\kpos_a}-\lgt^P_{\kpos_x}\ge\gapD=\log\frac{1-\tol}{r\tol},
  \end{equation*}
  and $\gapD>\gapA$.
\end{lemma}

\begin{proof}
  Let $m=|A^P|$.
  Every alive key has $w_a\ge(1-\tol)/m$, and the weights of all positions sum to one, so another alive key has $w_{a'}\le1-(m-1)(1-\tol)/m$ and $w_{a'}/w_a\le(1+(m-1)\tol)/(1-\tol)$.
  A dead key has $w_x\le\tol$, so $w_a/w_x\ge(1-\tol)/(r\tol)$.
  The inequality $\gapD>\gapA$ is equivalent to $(1-\tol)^2>r\tol(1+(r-1)\tol)$, whose left side decreases and whose right side increases in~$\tol$.
  At $\tol=1/(2r+2)$ it reads $(2r+1)^2>r(3r+1)$, that is $r^2+3r+1>0$.
\end{proof}

\begin{lemma}[One pattern per query]
  \label{lem:A.2}
  Under the hypothesis of \cref{thm:3.2}, if $P\ne P'$ and $P_1=P'_1=1$, then $q^P\ne q^{P'}$.
\end{lemma}

\begin{proof}
  Suppose that $q^P=q^{P'}=q$, and let $j$ be the smallest index with $P_j\ne P'_j$, so that $2\le j\le r-1$, and assume $P_j=1$ and $P'_j=0$.
  The prefixes of $x^P$ and $x^{P'}$ up to $\kpos_j$ coincide by the second property of $\ladder_r$, so the keys at $\kpos_1$ and $\kpos_j$ coincide as well.
  With the common query and positional bias the offset cancels, and $\Delta=\lgt^P_{\kpos_j}-\lgt^P_{\kpos_1}=\lgt^{P'}_{\kpos_j}-\lgt^{P'}_{\kpos_1}$.
  In $x^P$ keys 1 and~$j$ are both alive, so $\Delta\ge-\gapA$, and in $x^{P'}$ key~1 is alive and key~$j$ dead, so $\Delta\le-\gapD<-\gapA$, a contradiction.
\end{proof}

There are $2^{r-2}$ patterns with $P_1=1$, and by \cref{lem:A.2} they have pairwise different queries.
Each of the $d$ coordinates of a query takes at most $2^p$ values, so $2^{r-2}\le2^{dp}$, that is $dp\ge r-2$. \hfill$\square$

The proof uses four facts.
All sequences share the length and the key positions, the prefix up to key~$j$ depends only on the fates of the earlier keys, key~1 is alive in $2^{r-2}$ patterns, and the tolerance gives $\gapD>\gapA$.
Without a bound on the values of the query no lower bound on $d$ alone follows.
Keys built from the prefix with a one-coordinate midpoint code separate the alive keys from the dead ones by a single threshold at every time in exact arithmetic, so even $d=1$ meets the ordering condition of \cref{thm:3.1} with such keys.
The argument does not apply to biases in which a later token acts on each key separately, as in Selective Attention, DCP and causal attention with lookahead keys \citep{anagnostidis2023dynamic,leviathan2025selective,song2026causal}.
The contribution of such a bias for key~$j$ is not a function of the prefix up to $\kpos_j$.

\subsection{Proof of \texorpdfstring{\cref{prop:3.3}}{Proposition 3.3}}

Fix a sequence of $\ladder_r$ and its final query.
In the stack form the walk lives on the half-integer grid, so a dead position has $\rmin_{ts}\le W_s-\tfrac12$ and bias at most $-\str/2$, and an alive position has bias zero.
The logit of a key position is its bias, and the logit of every other position, the query included, is its bias minus~$M$.
So the alive keys have logit~0, the dead keys at most $-\str/2$ and the other positions at most $-M$.
With $Z=|A|+R$, at most $r-1$ dead keys and exactly $2r$ other positions,
\begin{equation*}
  R=\sum_{\text{dead keys}}e^{b_{ts}}+\sum_{\text{other}}e^{b_{ts}-M}\le(r-1)e^{-\str/2}+2re^{-M}\le\tfrac\tol2+\tfrac\tol2=\tol.
\end{equation*}
Key~$r$ is always alive, so $|A|\ge1$.
All positions other than the alive keys then have total weight $R/Z\le\tol$, and each alive key has weight $1/Z\ge1/(|A|+\tol)\ge(1-\tol)/|A|$.
The increments are multiples of $\tfrac12$ and $0\le W\le r$, so the walk takes at most $2r+1$ values. \hfill$\square$

Without the dot product the query position keeps the bias zero, the largest possible, so the field alone cannot meet the tolerance once $\tol\le1/(|A|+1)$.

\subsection{Proof of \texorpdfstring{\cref{prop:3.4}}{Proposition 3.4}}

First, $b_{ts}=\str\min\bigl(0,\min_u(\fld_u-\lvl_s)\bigr)=-\str\max\bigl(0,\max_u(\lvl_s-\fld_u)\bigr)=-\str D^{\max}_{ts}$.
In Selective Attention the selection of token~$u$ for key~$s$ passes through a ReLU and is summed over the positions strictly before~$t$, and the first token and the key itself are not masked \citep{leviathan2025selective}.
With a selection head of the form $\sasc(\lvl_s-\fld_u)$ and $\sasc>0$, the mask is therefore $F_{ts}=\sum_{s<u<t}\mathrm{relu}(\sasc(\lvl_s-\fld_u))=\sasc D^{\mathrm{sum}}_{ts}$.

(i)~An inserted step adds an element of at most $D^{\max}_{ts}$ to the set whose maximum is $D^{\max}_{ts}$, so the maximum and the bias do not change, while $D^{\mathrm{sum}}_{ts}$ grows by the inserted dip.
For tokens under \eqref{eq:2} this is not automatic, since the increment of a token shifts every later level.
If $\inc$, $g$ and~$h$ depend on the token alone, a block of tokens whose increments sum to zero and whose dips stay at most $D^{\max}_{ts}$ preserves the values of all other positions.

(ii)~Set $\lvl_s=0$ and $t=s+m+1$, and let every other $\fld_v$ with $s<v\le t$ be at least zero.
One step $u\in(s,t)$ with $\fld_u=-\dep$ gives $D^{\max}_{ts}=D^{\mathrm{sum}}_{ts}=\dep$, and $m$ such steps give $D^{\max}_{ts}=\dep$ and $D^{\mathrm{sum}}_{ts}=m\dep$.

(iii)~Let a multiset have maximum $x^\ast$ and other elements $x_1,\dots,x_k$.
Starting from $\{x^\ast\}$ and inserting $x_1,\dots,x_k$ one at a time never changes $\agg$, because each inserted element is at most the maximum, which stays $x^\ast$.
So $\agg$ of the multiset equals $\agg(\{x^\ast\})$. \hfill$\square$

\subsection{Proof of \texorpdfstring{\cref{prop:3.5}}{Proposition 3.5}}

The family is the graded-depth family of \cref{app:B.3}.
Key~$j$ has level~$L_j$, with $L_j-L_{j-1}\in\{5,6,7\}$, and a dip of depth $\dep$ consists of the dip tokens at levels $L_j-x$ for $x\in\{\dep,\dep-1,0\}$ (for $\dep=1$ the depths are $\{1,0,0\}$), its deepest token repeated $m_j\in\{1,2\}$ times.
A head realises the target exactly on the keys when $\lgt_s+\shp\dep_s$ is the same for every key~$s$ of the sequence.

(a)~Fix the level $L$ of the first key and a head of the class that realises the target exactly on every sequence whose first key has level~$L$.
Let $f_x=\pairfn(\texttt{D}\langle L-x\rangle,\texttt{K}\langle L\rangle)\ge0$ denote the contribution of the dip token at depth~$x$.
Let two sequences differ only in the duration of the first key's dip of depth~$x$, which is held once in one and twice in the other.
The second has one more copy of the deepest token between key~1 and the query, so the logit of key~1 is lower by $f_x$, while every later key keeps its logit and no target changes.
Key~2 shows that the common value of $\lgt_s+\shp\dep_s$ is the same in both sequences, and key~1 then gives $f_x=0$, for $x=1,2,3,4$.
Next let two sequences differ only in the first key's dip, of depth~2 in one and depth~3 in the other.
By the first step the dip tokens contribute $f_0$ to key~1 in both, every other token after key~1 is the same, and every later key keeps its logit.
Exact realisation then requires $\lgt_1+2\shp=\lgt_1+3\shp$, which contradicts $\shp>0$.

(b)~Let $\fld_u$ be the level of the token for dip and key tokens and a value above every level for all other tokens.
Let $\lvl_s$ be the level of key~$s$ and $\str=\shp$, and let the score between query and key be the same for every key.
The tokens after key~$s$ are its own dip, whose minimum level is $L_s-\dep_s$, other tokens above every level, and later blocks at levels of at least $L_s+1$.
So the bias of key~$s$ is $-\shp\dep_s$ and the target is met exactly.
Holding a step repeats a value already in the range, and an inserted dip of key~$s$ that is no deeper adds values of at least $L_s-\dep_s$ after it and above every earlier level, so no range minimum changes.
In \eqref{eq:2} this is $\inc=0$ with token-dependent offsets $g$ and~$h$. \hfill$\square$

Part~(a) uses only sequences in which key~1 has a dip and a later key follows it.
With fixed durations the class fits the targets exactly, with contributions $(0,1,1,2,2)\,\shp$ of a dip token at depths 0 to~4, so only durations that vary for dips of the same depth separate the two.

\subsection{Proof of \texorpdfstring{\cref{prop:3.6}}{Proposition 3.6}}

In the two-walk family $\twowalk_r$ every token carries one step for each of two walks.
The key token has the steps $(1,1)$, the four block tokens carry $(\pm\tfrac12,\pm\tfrac12)$, and $\tokbos$ and \tokq{} carry $(0,0)$.
The sequence $x^{P^{(1)},P^{(2)}}$ has the shape of a ladder sequence, and block~$i$ consists of two tokens whose first steps follow $P^{(1)}_i$ and whose second steps follow $P^{(2)}_i$ as in~$\ladder_r$.
The alive key sets at the final query are therefore $A^{(h)}=\{j<r:P^{(h)}_j=1\}\cup\{r\}$ for $h=1,2$, and the complementary part has $P^{(2)}=\lnot P^{(1)}$.

Let $k^h_j$ and $\posb^h_j$ denote the key and the positional bias of head~$h$ at key~$j$, which do not depend on the sequence.
Let $z_j=(k^2_j,-k^1_j)\in\R^{2d}$ and $\nu_j=\posb^2_j-\posb^1_j$.
Both heads subtract the same $F_{\kpos_j}$, so with $Q=(q^2,q^1)$ the difference of the heads is
\begin{equation*}
  \lgt^2_{\kpos_j}-\lgt^1_{\kpos_j}=\langle Q,z_j\rangle+\nu_j+(c^2-c^1),\qquad j=1,\dots,r.
\end{equation*}

(a)~On an ordered sequence of the complementary part, let $x<r$ be alive under walk~2 and dead under walk~1, and $y<r$ the reverse.
Head~2 gives $\lgt^2_{\kpos_x}>\lgt^2_{\kpos_y}$ and head~1 gives $\lgt^1_{\kpos_y}>\lgt^1_{\kpos_x}$, so the difference is larger at $x$ than at~$y$.
A threshold between the two groups then makes a function of $\fcls=\{j\mapsto\langle Q,z_j\rangle+\nu_j-\theta\}$ positive exactly on $\{j<r:P^{(1)}_j=0\}$.

(b)~With tolerance $\tol\le1/(2r+2)$, \cref{lem:A.1} applies to each head with its own alive set.
Key~$r$ is alive under both walks.
For $x<r$ alive under walk~2 and dead under walk~1 the difference at $x$ exceeds that at $r$ by at least $\gapD-\gapA>0$.
For $y<r$ alive under walk~1 the difference at $r$ exceeds that at $y$ by the same amount.
So $j\mapsto\langle Q,z_j-z_r\rangle+\nu_j-\nu_r$, in which the offsets cancel, is positive exactly on $\{j<r:P^{(1)}_j=0\}$.

In both parts distinct sequences give distinct sign patterns on $\{1,\dots,r-1\}$.
The class of part~(a) shatters no $2d+2$ points, because the vectors $(z_j,1)\in\R^{2d+1}$ of any $2d+2$ points are linearly dependent and the argument of \cref{thm:3.1} applies.
The class of part~(b) shatters no $2d+1$ points for the same reason with the vectors $z_j-z_r\in\R^{2d}$.
By the Sauer--Shelah lemma \citep{sauer1972density,shelah1972combinatorial} the classes have at most $N(r,d)$ and $N'(r,d)$ sign patterns on $r-1$ points. \hfill$\square$

The proof uses only that both heads subtract the same mask, which may depend arbitrarily on the sequence, and that the pair of heads is fixed in advance.
It covers Selective Attention whose keys depend only on the token and the position, including RoPE at a common final query, and in part~(a) a fixed positive scale of the mask in one head.
It also covers the masks of DCP, which may remove a key completely.
Every key that a counted sequence needs has a finite mask, because it is alive under one of the walks.
In part~(a) the two sequences in which one walk keeps every key are the exception, and the class~$\fcls$ already contains their patterns.
For keys and positional biases in general position and a mask chosen freely the counts $N(r,d)$ and $N'(r,d)$ are attained, the second once the biases are large enough for the tolerance.

\section{One-layer probes}
\label{app:B}

\subsection{Common setup}

Every probe was a single attention layer on token embeddings of width 64, supervised at the final query, and \cref{app:D} lists the settings of every probe.
The values were fixed one-hot vectors of the key positions and the readout was the identity, so the output of the head was its weight on each key.
The loss was the cross-entropy of these weights against the target.
The deeper models of \cref{tab:5} add $L-1$ pre-norm blocks of width 64 with 4 heads and a feed-forward width of 256 in front of the supervised head.
We trained with AdamW (0.9, 0.98), weight decay 0.01, batch 64, a warm-up of 200 steps followed by cosine decay to zero, gradient clipping at 1.0 and single precision, and we drew fresh training data at every step.
Training lasted 4000 steps, and 16,000 steps on the graded-depth families.
The seeds of the main runs were 1801 to 1815, and the sampled test sets used seeds that no tuning run had used.

We declared the same search of 24 tuning runs for every method on the alive-set, graded-depth and two-walk families.
FPA, FoX, Selective Attention and DCP crossed three learning rates in $\{10^{-3},3\times10^{-3},10^{-2}\}$ and two tuning seeds with four values of one more setting.
This setting was the definition of the minimum for FPA, the initial gate bias for FoX, the selection scale for Selective Attention and the initial bias of the interaction scores for DCP.
The methods without such a setting searched twelve learning rates from $10^{-4}$ to $10^{-1}$.
A grid whose end was selected was extended once by a decade.
Each competitor was tuned on its positive control and FPA on the setting where it can succeed, and we froze the chosen configurations with a hash before the main runs.
After the pilots we fixed the definition of FPA as the straight-through minimum of \cref{sec:2} with a learned strength per head, and we selected its learning rate again.

The search departed from the declaration in four ways.
Grid extensions, remedies for positive controls and the second selection for FPA raised the number of tuning runs above 24 (\cref{tab:4}).
The configurations on the graded-depth families and that of Selective Attention with one mask per head came from exploratory runs instead of the declared search.
These runs compared only the selection scales 1 and $1/3$ for Selective Attention on the graded depths.
The learning rates of Selective Attention on the graded depths and of the variant with one mask per head lay at the upper end of their grids, which we did not extend.
The sum of dips on the graded depths took the learning rate of FPA without a search of its own.
Selective Attention with a shared mask used the scale $1/3$, which tuning chose on the positive control of the two walks, on the two-walk and retrieval families.
Unlike on the alive sets, no main run on these families used the scale 1 of its paper.
After its one extension the selected learning rate still lay at an end of the grid in four cases.
These were FPA and FPA with the sum of dips on the stack family, stick-breaking attention in retrieval and softmax attention with RoPE in variable assignment.

\begin{table}[htbp]
  \caption{\textbf{Tuning runs and GPU minutes of the search in the pilots} on the two families whose comparisons rest on a declared search.
    A dagger marks a search that took more than 1.25 times the GPU time of the cheapest compared search of its family.
    We estimated the GPU minutes of runs that shared the GPU from a calibration of throughput.
    The configuration of Selective Attention with one mask per head came from later exploratory runs.
    FPA, First Passage Attention.
    FoX, Forgetting Transformer.
    ALiBi, attention with linear biases.
    NoPE, no positional encoding.
    RoPE, rotary position embedding.
    DCP, Dynamic Context Pruning.}
  \label{tab:4}
  \centering
  \small
  \setlength{\tabcolsep}{4pt}
  \begin{tabular}{@{}llrr@{}}
\toprule
Family & Method & Runs & GPU min \\
\midrule
Alive sets & FPA$^{\dagger}$ & 48 & 46.3 \\
 & FoX$^{\dagger}$ & 40 & 28.9 \\
 & Selective Attention$^{\dagger}$ & 40 & 30.2 \\
 & Softmax + ALiBi & 28 & 16.6 \\
 & Softmax + NoPE & 28 & 16.3 \\
 & Softmax + RoPE & 24 & 15.8 \\
 & Stick-breaking & 24 & 16.5 \\
\midrule
Two walks & FPA & 36 & 32.5 \\
 & Selective Attention, shared mask & 44 & 32.6 \\
 & Selective Attention, one mask per head & 40 & 30.0 \\
 & DCP$^{\dagger}$ & 44 & 66.7 \\
\bottomrule
\end{tabular}

\end{table}

For every component we ran a one-sided exact McNemar test of FPA against the competitor on the per-seed successes, and a hypothesis was supported only when all its components rejected at $\sig=0.05/7$.
The seven tests are the three tests of \cref{sec:4.1} and two tests of each of the two language-model hypotheses, of which we ran one each.
At the tested settings of the alive-set and two-walk families, the counts of \cref{app:A} rule out every success of the compared competitors.
Each test on these families then depends only on the successes of FPA, with $p=2^{-k}$ for $k$ of them.
The preregistered refutation of each probe hypothesis was a one-sided 99\% Clopper--Pearson upper bound below $0.5$ on the success rate of FPA.
On the alive-set and two-walk families a competitor run beyond these counts would also have refuted the hypothesis, unless an implementation error or, for two walks, a model outside the assumptions of \cref{prop:3.6} explained it.
No run went beyond them, and no hypothesis was refuted.
In these probes we trained the selection of Selective Attention with the gradient of a softplus in place of that of its ReLU, while the forward pass kept the ReLU of the paper.
With the ReLU gradient the selections of several pilot heads closed at the start of training and never reopened.
We checked the forward pass of our implementation element by element against a line-by-line transcription of the sketch in Figure~2 of the paper of Selective Attention.

\subsection{Alive sets}

The ladder families of \cref{app:A.1} with $r\in\{6,10,14,18\}$ used the full set of $2^{r-1}$ sequences as their test set.
The query at the final position came from a network that read the one-hot vectors of all tokens through a hidden layer of width 256 with the Gaussian error linear unit and an output of size $Hd$.
The query could thus depend arbitrarily on the prefix, as \cref{thm:3.1} allows.
Every key had a learned offset per position and head, and FPA, FoX, Selective Attention and softmax attention also used RoPE.
ALiBi added the fixed slope $2^{-8}$ of a single head, and NoPE and stick-breaking attention, which followed its paper, used the offsets alone.
The deeper models of \cref{tab:5} had neither this query network nor these offsets.
The field of FPA and the gates of FoX read only the token embedding.
A test of the rank of the key logits confirmed that each competitor stayed in the class of the theorem.
The positive controls with $d=r+2$ admit a construction in the class, keys equal to one-hot vectors of the key positions and queries equal to the indicator of the alive set.

\begin{table}[htbp]
  \caption{\textbf{Every cell of the alive-set family}, as seeds out of 15 that ordered the alive keys above the dead keys on at least 90\% of the sequences.
    Selective Attention appears with the selection scale 10 chosen in tuning and with the scale 1 of its paper.
    The last rows compare the level of the key with the level of the query, whose alive sets are always suffixes.
    A dash marks a cell that the design did not include.
    FPA, First Passage Attention.
    RoPE, rotary position embedding.
    NoPE, no positional encoding.
    ALiBi, attention with linear biases.
    FoX, Forgetting Transformer.}
  \label{tab:5}
  \centering
  \small
  \begin{tabular}{@{}lcccc@{}}
\toprule
Method ($d=4$, one layer) & $r=6$ & $r=10$ & $r=14$ & $r=18$ \\
\midrule
FPA & 15/15 & 11/15 & 12/15 & 13/15 \\
Softmax + RoPE & 15/15 & 0/15 & 0/15 & 0/15 \\
Softmax + NoPE & 14/15 & 0/15 & 0/15 & 0/15 \\
Softmax + ALiBi & 14/15 & 0/15 & 0/15 & 0/15 \\
FoX & 15/15 & 0/15 & 0/15 & 0/15 \\
Selective Attention, scale 10 & 15/15 & 15/15 & 13/15 & 15/15 \\
Selective Attention, scale 1 & 15/15 & 15/15 & 15/15 & 15/15 \\
Stick-breaking attention & 15/15 & 0/15 & 0/15 & 0/15 \\
\midrule
\multicolumn{5}{@{}l}{Positive controls, $d=r+2$} \\
Softmax + RoPE & --- & 15/15 & 15/15 & 15/15 \\
Softmax + NoPE & --- & 15/15 & 15/15 & 15/15 \\
Softmax + ALiBi & --- & 15/15 & 15/15 & 15/15 \\
FoX & --- & 15/15 & 15/15 & 15/15 \\
\midrule
\multicolumn{5}{@{}l}{Deeper models, $d=4$ in the last layer} \\
Softmax + RoPE, 2 layers & --- & --- & 15/15 & 0/15 \\
Softmax + RoPE, 4 layers & --- & --- & 15/15 & 15/15 \\
FoX, 2 layers & --- & --- & 15/15 & 8/15 \\
FoX, 4 layers & --- & --- & 15/15 & 15/15 \\
\midrule
\multicolumn{5}{@{}l}{Level of the key against the level of the query ($d=1$)} \\
FPA, field only & --- & 12/15 & 10/15 & --- \\
FPA, threshold of the query & --- & 0/15 & 0/15 & --- \\
\quad on suffix alive sets only & --- & --- & 12/15 & --- \\
\bottomrule
\end{tabular}

\end{table}

The four competitors in the class succeeded at $r=6$, where $d=r-2$ allows every pattern, and failed on every seed at larger~$r$, while their positive controls succeeded on every seed.
Two-layer softmax and FoX models learned the family at $r=14$ on every seed, at $r=18$ the two-layer softmax models failed on every seed and the two-layer FoX models succeeded on 8, and four-layer models learned both.
A field-only FPA head with a dot product of dimension~1 learned the family on 12 seeds at $r=10$ and on 10 at $r=14$. With the level of the query in place of the level of the key, as in StructFormer and PRPN, the same head succeeded on no seed, although it learned a family restricted to suffixes on 12 seeds.

\subsection{Graded depths and varied durations}
\label{app:B.3}

Each sequence of the graded-depth family starts with $\tokbos$ and holds six keys whose tokens carry their levels, with consecutive levels 5, 6 or 7 apart.
After each key come zero to three filler tokens, with probability 0.75 a dip, and zero to three fillers again, and the sequence ends with the query~\tokq.
A dip of depth $\dep\in\{1,2,3,4\}$ consists of three dip tokens whose levels lie $\dep$, $\dep-1$ and $0$ below the key, in random order, and the target weight of key~$j$ is proportional to $e^{-\dep_j}$.
The model had no positional information, its query came from the token~\tokq{} alone, its keys came from their tokens alone, and its head dimension was~8.
A seed fitted a test set when the weights on the keys, renormalised over the keys, were within total variation 0.05 of the target on at least 90\% of the 2000 sequences.
With fixed durations every deepest step lasted one step, and with varied durations each lasted one or two steps, drawn per dip.
We drew the spacings of the levels, the fillers, the dips and their depths, orders and durations independently.

\begin{table}[htbp]
  \caption{\textbf{Graded depths with fixed and with varied durations}, as seeds out of 15 that fitted the training distribution.
    The sum of dips replaces the minimum of FPA and keeps the same field.
    FPA, First Passage Attention.}
  \label{tab:6}
  \centering
  \small
  \begin{tabular}{@{}lcc@{}}
\toprule
Model & Fixed durations & Varied durations \\
\midrule
FPA & 9/15 & 13/15 \\
FPA, sum of dips & 15/15 & 0/15 \\
Selective Attention & 15/15 & 0/15 \\
\bottomrule
\end{tabular}

\end{table}

The secondary test measured, on the seeds that fitted fixed durations, the change of the total variation when a deepest step was held two or three steps or a dip no deeper was inserted.
The rule required both methods to fit fixed durations on 80\% of the seeds, which FPA reached on only 9.
After the pilots and before the main runs we decided that the deviation counts only the keys after the first one.
The dip of the first key lies in the range of no other key, and the learned walks read its depth by summing increments.
A separate family in which the levels follow the walk, built to check invariance on the side of FPA, was fitted by FPA and by its sum of dips on all 15 seeds.
With every dip held 0 or 16 steps at its bottom instead of 4 to 8, the share of ordered sequences fell by at most 0.1 on 14 seeds of FPA and 10 of the sum of dips.
Only FPA reached the preregistered line of 12 such seeds.

\subsection{Two walks}

The two-walk family of \cref{app:A} with $r=14$ has 8192 sequences in its complementary part, all of which formed the test set.
The queries and keys had learned offsets per position and head, the final query came from the network of the alive-set family, and the fields of FPA read only the token embedding.
The training data mixed uniform pairs of patterns (one half), complementary pairs (one quarter) and pairs in which one walk keeps every key (one eighth each).
DCP used an interaction dimension of 2 and an $\alpha$-sigmoid whose exponent rose from 1 to 2 during training.
Its step function at inference can remove an alive key completely, so in the training loss only, we replaced a log-probability of exactly $-\infty$ by $-50$.

\begin{table}[htbp]
  \caption{\textbf{Two walks with two heads of dimension 2}, as seeds out of 15 that ordered the keys of both walks on at least 90\% of the sequences.
    One mask per head is our variant of Selective Attention, in which every head computes its own selection.
    FPA, First Passage Attention.
    DCP, Dynamic Context Pruning.}
  \label{tab:7}
  \centering
  \small
  \begin{tabular}{@{}lcc@{}}
\toprule
Model & Complementary walks & Identical walks \\
\midrule
FPA & 11/15 & --- \\
Selective Attention, shared mask & 0/15 & 15/15 \\
Selective Attention, one mask per head & 13/15 & --- \\
DCP & 0/15 & 13/15 \\
\bottomrule
\end{tabular}

\end{table}

No run of a shared mask on the complementary part ordered more than 2{,}380 sequences or met the tolerance $1/30$ on more than 1{,}093, the bounds of \cref{prop:3.6}.

\subsection{Stack}

The stack family consists of random legal sequences of 40 operations, pushes of eight symbols and pops, with a stack depth of at most six.
The target at the final query is uniform over the pushed items still on the stack, and the head dimension was~8.
FPA succeeded on 14 seeds and FPA with the sum of dips on 11, and one-layer FoX, softmax attention with RoPE and Selective Attention succeeded on none.
These competitors used the learning rates of an earlier check, and FoX used the default initialisation of its gates.
We have no proof that they cannot learn the family in one layer, so we draw no comparison from them.

\subsection{Retrieval among alive keys}
\label{app:B.6}

This task needs both parts of FPA at once.
A sequence holds 16 slots and ends with a query for one name.
Each slot is a key with a name in $\{0,\dots,5\}$ and a level in $\{1,\dots,8\}$ with probability 0.6 and a dip at a level otherwise, and zero or one filler follows it.
A key is alive when no later dip falls below its level, and the queried name occurs four or five times with at least two alive occurrences whose margins differ.
At least one dead occurrence precedes the first alive one and at least one lies between two alive ones, so neither a cut by position nor a choice by level alone solves the task.
The target is uniform over the alive occurrences of the queried name.
A seed succeeded when the head realised it with tolerance $0.1$ on at least 90\% of the sequences, a rule that we chose after the pilots and before the main runs.
Before any run we fitted strategies of each ablation in the hard limit on a sample of 8000 sequences.
On a second sample of 2000 sequences from the same generator, separate from the test set of the main runs, the best that we found for any ablation reached $0.4905$, below the preregistered line of $0.6$.
By that rule the results read as evidence that both the level of the key and retrieval by content are needed among these variants.

\begin{table}[htbp]
  \caption{\textbf{Retrieval among alive keys}, as seeds out of 15 that succeeded.
    FPA, First Passage Attention.
    FoX, Forgetting Transformer.}
  \label{tab:8}
  \centering
  \small
  \begin{tabular}{@{}lc@{}}
\toprule
Model & Seeds that succeed \\
\midrule
FPA & 15/15 \\
FPA without the level of the key ($h\equiv0$) & 0/15 \\
FPA with the threshold of the query & 0/15 \\
FPA without content retrieval & 0/15 \\
FoX & 0/15 \\
Stick-breaking attention & 0/15 \\
Selective Attention & 13/15 \\
Softmax attention & 0/15 \\
\bottomrule
\end{tabular}

\end{table}

\subsection{Tasks from earlier work}
\label{app:B.7}

We ran six published tasks at a smaller scale than their papers, so our numbers do not compare with theirs.
Variable assignment follows Selective Attention (three variables, 1000 values, 32 assignments), and multi-query repeated associative recall (MQRAR) follows stick-breaking attention (64 keys, 256 key and value symbols).
Flip-flop language modelling uses length 256 and a probability of 0.8 for the ignore instruction \citep{liu2023exposing}, and copying follows Selective Attention.
Counting asks at every position how often its token occurred so far, and multi-query associative recall (MQAR) uses 16 pairs in sequences of 128 \citep{arora2024zoology}.
The models were two-layer decoders of width 128 with two heads and a feed-forward width of 512, and a one-layer decoder of width 64 for flip-flop.
FPA, FoX and Selective Attention used RoPE in addition to their mechanisms, although FoX uses none by default in its paper, and stick-breaking attention used no positional encoding, as in its paper.
In variable assignment the other attention models used learned absolute positions in place of RoPE, as the paper of Selective Attention does.
A long short-term memory network (LSTM) of the same depth and width served as a recurrent reference.
Training lasted 4000 steps, and 8000 for MQAR, with a warm-up over the first 10\% of the steps and weight decay on all parameters, although the papers of FoX and stick-breaking attention exempt the biases.
A seed succeeded when its main accuracy reached $0.99$, or for flip-flop when its error stayed at most $0.01$.

\begin{table}[htbp]
  \caption{\textbf{Tasks from earlier work}, as seeds out of 15 that succeeded.
    In variable assignment the softmax column used learned absolute positions in place of RoPE.
    MQRAR, multi-query repeated associative recall.
    MQAR, multi-query associative recall.
    FPA, First Passage Attention.
    RoPE, rotary position embedding.
    FoX, Forgetting Transformer.
    Sel.\ Att., Selective Attention.
    Stick-br., stick-breaking attention.
    LSTM, long short-term memory network.}
  \label{tab:9}
  \centering
  \small
  \setlength{\tabcolsep}{3pt}
  \begin{tabular}{@{}lcccccc@{}}
\toprule
Task & FPA & Softmax + RoPE & FoX & Sel.\ Att. & Stick-br. & LSTM \\
\midrule
Variable assignment & 9/15 & 0/15 & 15/15 & 2/15 & 12/15 & 15/15 \\
MQRAR & 0/15 & 0/15 & 0/15 & 4/15 & 0/15 & 0/15 \\
Flip-flop & 15/15 & 0/15 & 15/15 & 15/15 & 0/15 & 15/15 \\
Copying & 15/15 & 15/15 & 15/15 & 15/15 & 14/15 & 15/15 \\
Counting & 15/15 & 15/15 & 15/15 & 15/15 & 15/15 & 15/15 \\
MQAR & 0/15 & 15/15 & 14/15 & 12/15 & 15/15 & 0/15 \\
MQAR, FPA without the bias in layer 1 & 15/15 & --- & --- & --- & --- & --- \\
\bottomrule
\end{tabular}

\end{table}

In pilot runs of two-layer FPA models on MQAR, the first layer learned to suppress the attention from each value to its own key.
The previous-token head that starts the recall circuit therefore never formed.
An exploratory inspection of the FPA language model found its previous-token heads in the second and third of its 12 layers.
In these heads a large strength turned the bias into a narrow window, and the model copied as well as FoX did (\cref{app:C}).
Two exploratory checks, outside the preregistration, support this reading.
Three- and four-layer FPA models failed MQAR on each of 4 seeds, and the same models without a bias in their first layer succeeded on each of 2 seeds.
We also ran Selective Attention with the ReLU gradient of its paper in exploratory runs.
It succeeded on 1, 3, 15, 15, 15 and 15 seeds of the six tasks in the order of \cref{tab:9}, close to its results with the softplus gradient.

\section{Language models}
\label{app:C}

\subsection{Models, data and training}

All four models are pre-norm decoders with RMSNorm, a SwiGLU feed-forward network of width 2048, no biases and tied embeddings, with 12 layers of width 768 and 12 heads of dimension 64, following the 125M configuration of FoX \citep{lin2025forgetting}.
\Cref{fig:22,tab:15,tab:16} give the architecture and every setting.
The vocabulary of 32,000 tokens gives 109.5M parameters for softmax attention and Selective Attention and 109.9M for FPA.
FPA computes its field projection in single precision, uses no positional encoding and learns one strength per head.
FoX uses the kernel of its authors without positional encoding and initialises its gate biases at zero, as its paper does for data-dependent gates.
Selective Attention uses RoPE, and its selection has the ReLU gradient of its paper.
Softmax attention uses RoPE with base 10,000.

The training data was a chunk of SlimPajama \citep{soboleva2023slimpajama} with every document wrapped in $\tokbos$ and $\langle/\mathrm{s}\rangle$, and the validation and test sets were the first 20M tokens of its validation and test splits.
The sources of SlimPajama come with their own terms of use and may include copyrighted text.
We cut the training stream into non-overlapping windows of 2049 tokens, visited them in one fixed order for all four models and used each window at most once.
Training ran 5{,}150 steps of 524{,}288 tokens, 2.70B tokens in total, with AdamW (0.9, 0.95), learning rate $2\times10^{-3}$ after 512 warm-up steps and cosine decay to zero, weight decay 0.1 on the matrices, gradient clipping at 1.0 and half-precision autocast with loss scaling.
All runs used one 12\,GB partition of an NVIDIA Turing GPU (compute capability 7.5, without native bfloat16).
Training took 23.8 hours for FPA, 21.5 for FoX, 38.9 for Selective Attention and 26.3 for softmax attention (\cref{tab:15}).
The evaluated model trained with an earlier version of the kernel of \cref{sec:4.3}, which computes the same function (\cref{app:E}).
To time that kernel, we trained a second FPA model in an exploratory run of the same configuration with seed 2202, outside the preregistration.

\Cref{fig:5} shows how the four models trained.
FPA started more slowly than FoX and Selective Attention and had a higher validation loss at step 250.
From step 500 on it had the lowest validation loss of the four at every evaluation.
At the end its validation loss lay 0.020 nats below that of FoX, 0.006 below that of Selective Attention and 0.027 below that of softmax attention.

\begin{figure}[htbp]
  \centering
  \includegraphics{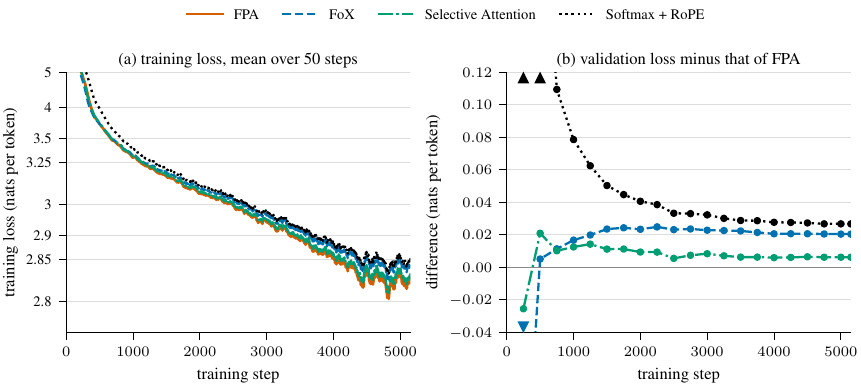}
  \caption{\textbf{Training of the four language models.}
    (a)~Training loss, averaged over 50 steps of 524{,}288 tokens, on a vertical axis that is logarithmic in the loss minus 2.75 nats.
    (b)~Validation loss of each baseline minus that of FPA, on 1024 validation windows every 250 steps, so that positive values favour FPA.
    The triangles mark values beyond the axis, which were $-0.16$ for FoX at step 250 and $+0.46$ and $+0.25$ for softmax attention at steps 250 and 500.
    FPA, First Passage Attention.
    FoX, Forgetting Transformer.
    RoPE, rotary position embedding.}
  \label{fig:5}
\end{figure}

\subsection{Evaluation and tests}

The test loss averaged the cross-entropy over the 20M test tokens in windows of 2048.
WikiText-103 \citep{merity2017pointer} used its test set, and PG-19 used up to four windows of 4096 tokens per test book and averaged the loss at positions 2048 to 4095 over 99 books.
The passkey items hide a five-digit number in repeated filler text \citep{mohtashami2023random}, ten numbers at each of five depths, so 50 items at each of 1024, 2048, 3072 and 4096 tokens.
The non-inferiority test resampled documents 10,000 times and compared the upper $1-\sig$ quantile of the difference in test loss with $\ln1.01$.
The extrapolation test resampled books and asked for an upper $1-\sig/2$ quantile below zero, and the passkey test was an exact one-sided McNemar test at $\sig/2$ on the items at 3072 and 4096.
Evaluating in single precision changed the difference of the first test by less than $10^{-6}$.

\subsection{Amendments}

Three amendments changed the preregistration during the runs, and each was frozen before the result it affects existed.
The first round of training, with a reference implementation of FPA at context 512, stopped before its first checkpoint, and the round reported here replaced it.
The PG-19 evaluation first kept the carriage returns of the books, a token that the training data never contained, so we amended the extrapolation test to Unix newlines. We had seen the PG-19 losses of FPA and softmax attention at that point, but FoX, the comparator of the test, had not yet been evaluated.
With carriage returns the test would also have supported FPA, with a difference of $-0.250$ nats.
The carriage returns raised the loss of every model, by 1.34 to 1.97 nats on average in windows of 2048 tokens and most for softmax attention.
Under both conventions FPA had the lowest loss beyond the training length (\cref{fig:6}).
Selective Attention first trained with the softplus gradient of the probes.
A preregistered check after 1000 steps favoured the ReLU gradient of its paper, and the reported model continued the run of that check, so it trained with the ReLU gradient from the start.
Its mask and attention ran through a fused Triton kernel that we wrote for it, with a micro-batch of 2 instead of 4 and the same tokens per step.

\begin{figure}[htbp]
  \centering
  \includegraphics{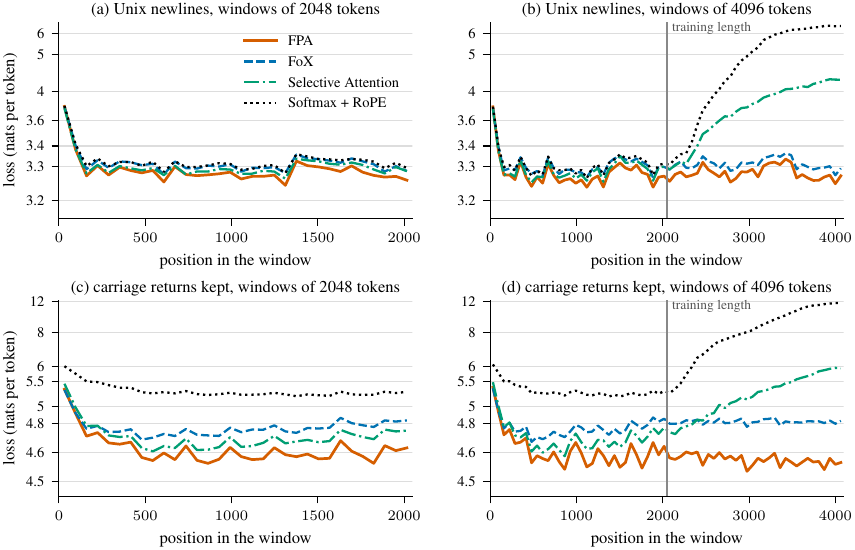}
  \caption{\textbf{Loss by position on the PG-19 test books.}
    Mean loss per bin of 64 positions in windows of 2048 and of 4096 tokens, up to four windows per book from 100 and 99 books.
    The top panels use Unix newlines, as the amended extrapolation test does, and the bottom panels keep the carriage returns of the books.
    The extrapolation test compared the positions 2048 to 4095 of panel~(b).
    The vertical axes are logarithmic in the loss minus 3.1 nats in the top panels and 4.4 nats in the bottom panels, which spreads out the lowest losses.
    The models were trained with context 2048.
    FPA, First Passage Attention.
    FoX, Forgetting Transformer.
    RoPE, rotary position embedding.}
  \label{fig:6}
\end{figure}

\subsection{Descriptive evaluations}

\begin{table*}[t]
  \caption{\textbf{FPA minus each baseline} on the descriptive evaluations, with 95\% bootstrap intervals in brackets.
    PG-19 losses come from windows of 16,384 tokens, passkey and NIAH (needle in a haystack) are accuracies on 50 items, and the downstream rows are zero-shot accuracies or log-likelihoods from the evaluation harness of \citet{gao2024harness}.
    Negative losses and positive accuracies favour FPA.
    FPA, First Passage Attention.
    FoX, Forgetting Transformer.
    RoPE, rotary position embedding.}
  \label{tab:10}
  \centering
  \small
  \begin{tabular}{@{}lccc@{}}
\toprule
FPA minus & FoX & Selective Attention & Softmax + RoPE \\
\midrule
PG-19 loss, 0--2k & $-$0.029 [$-$0.033, $-$0.026] & $-$0.014 [$-$0.024, $-$0.007] & $-$0.038 [$-$0.050, $-$0.028] \\
PG-19 loss, 2k--4k & $-$0.031 [$-$0.037, $-$0.026] & $-$0.516 [$-$0.540, $-$0.495] & $-$1.847 [$-$1.939, $-$1.769] \\
PG-19 loss, 4k--8k & $-$0.020 [$-$0.027, $-$0.015] & $-$0.300 [$-$0.317, $-$0.285] & $-$3.666 [$-$3.806, $-$3.545] \\
PG-19 loss, 8k--16k & $+$0.026 [$+$0.019, $+$0.032] & $-$0.528 [$-$0.557, $-$0.501] & $-$4.258 [$-$4.428, $-$4.114] \\
Passkey 8k & $+$0.80 [$+$0.68, $+$0.90] & $+$1.00 [$+$1.00, $+$1.00] & $+$1.00 [$+$1.00, $+$1.00] \\
Passkey 16k & $+$0.80 [$+$0.68, $+$0.90] & $+$0.92 [$+$0.84, $+$0.98] & $+$1.00 [$+$1.00, $+$1.00] \\
NIAH 4k, 1 needle & $-$0.06 [$-$0.20, $+$0.08] & $+$0.12 [$+$0.04, $+$0.22] & $+$0.12 [$+$0.04, $+$0.22] \\
NIAH 4k, 4 needles & $+$0.08 [$-$0.06, $+$0.22] & $+$0.20 [$+$0.10, $+$0.32] & $+$0.20 [$+$0.10, $+$0.32] \\
NIAH 8k, 1 needle & $-$0.06 [$-$0.14, $+$0.00] & $+$0.04 [$+$0.00, $+$0.10] & $+$0.04 [$+$0.00, $+$0.10] \\
NIAH 8k, 4 needles & $+$0.04 [$-$0.06, $+$0.16] & $+$0.10 [$+$0.02, $+$0.18] & $+$0.10 [$+$0.02, $+$0.18] \\
LAMBADA log-likelihood & $+$0.15 [$+$0.11, $+$0.19] & $+$0.28 [$+$0.23, $+$0.32] & $+$0.42 [$+$0.37, $+$0.46] \\
LAMBADA accuracy & $-$0.001 [$-$0.010, $+$0.008] & $+$0.014 [$+$0.004, $+$0.023] & $+$0.028 [$+$0.019, $+$0.038] \\
HellaSwag accuracy & $+$0.005 [$-$0.001, $+$0.010] & $+$0.005 [$-$0.000, $+$0.011] & $-$0.000 [$-$0.006, $+$0.005] \\
PIQA accuracy & $+$0.018 [$+$0.003, $+$0.033] & $+$0.007 [$-$0.008, $+$0.022] & $+$0.007 [$-$0.008, $+$0.022] \\
ARC-Easy accuracy & $+$0.003 [$-$0.014, $+$0.019] & $+$0.010 [$-$0.007, $+$0.027] & $+$0.013 [$-$0.003, $+$0.029] \\
ARC-Challenge accuracy & $-$0.002 [$-$0.020, $+$0.016] & $-$0.004 [$-$0.021, $+$0.014] & $-$0.009 [$-$0.026, $+$0.009] \\
WinoGrande accuracy & $-$0.028 [$-$0.060, $+$0.006] & $-$0.026 [$-$0.060, $+$0.009] & $-$0.022 [$-$0.056, $+$0.013] \\
\bottomrule
\end{tabular}

\end{table*}

Each needle-in-a-haystack item hides one or four sentences of the form ``The pass key of \emph{name} is \emph{number}.'' at random sentence boundaries in a passage of a PG-19 test book.
The item then asks for the key of one name, and there were 50 items for each length and number of needles.
The downstream tasks were LAMBADA \citep{paperno2016lambada}, HellaSwag \citep{zellers2019hellaswag}, PIQA \citep{bisk2020piqa}, ARC \citep{clark2018arc} and WinoGrande \citep{sakaguchi2020winogrande}.
At this scale all four models stayed close to chance on most of them, and most differences had intervals that contained zero.
On the SlimPajama test set FPA had a lower loss than FoX at every position bucket of the window, by 0.008 to 0.022 nats.
FPA retrieved every passkey at all six lengths, from 1024 to 16,384 tokens (\cref{fig:3}).
FoX fell from 1.00 at 2048 tokens to 0.58 at 4096 and 0.20 beyond, and Selective Attention retrieved 0.80 at 2048 tokens and almost none beyond 3072.
Softmax attention retrieved about half of the passkeys even within its training length.
A probe during training measured the greedy accuracy on the second copy of a span every 1000 steps (\cref{fig:7}).
At the end the four models copied 0.91 to 0.93 of the tokens at a gap of 0 tokens, 0.83 to 0.85 at 128 and 0.68 to 0.71 at 320, with softmax attention highest at every gap.

\begin{figure}[htbp]
  \centering
  \includegraphics{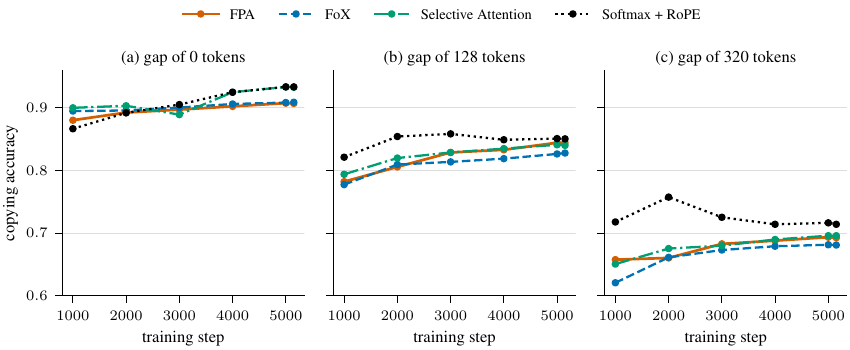}
  \caption{\textbf{Copying during training.}
    Greedy accuracy on the second copy of a span of text at gaps of 0, 128 and 320 tokens between the two copies, every 1000 steps and at the last step.
    FPA, First Passage Attention.
    FoX, Forgetting Transformer.
    RoPE, rotary position embedding.}
  \label{fig:7}
\end{figure}

\begin{table}[htbp]
  \caption{\textbf{RULER and BABILong}, mean accuracy over the 13 tasks of RULER and the five tasks qa1 to qa5 of BABILong, at the context lengths of the columns.
    All models were trained with context 2048.
    FPA, First Passage Attention.
    FoX, Forgetting Transformer.
    RoPE, rotary position embedding.}
  \label{tab:11}
  \centering
  \small
  \setlength{\tabcolsep}{3pt}
  \begin{tabular}{@{}lccccc@{}}
\toprule
 & \multicolumn{3}{c}{RULER, 13 tasks} & \multicolumn{2}{c}{BABILong, qa1--qa5} \\
\cmidrule(lr){2-4}\cmidrule(lr){5-6}
Method & 4k & 8k & 16k & 0k & 4k \\
\midrule
FPA & 0.208 & 0.202 & 0.051 & 0.082 & 0.006 \\
FoX & 0.218 & 0.117 & 0.072 & 0.078 & 0.000 \\
Selective Attention & 0.005 & 0.008 & 0.043 & 0.080 & 0.000 \\
Softmax + RoPE & 0.002 & 0.002 & 0.000 & 0.088 & 0.000 \\
\bottomrule
\end{tabular}

\end{table}

\subsection{A demonstration without a usable task}
\label{app:C.5}

We preregistered a zero-shot demonstration in generated text with five worlds, on scopes in Python code, on who saw which hidden object, on branches in a version history, on stop-loss orders and on a parenthetical remark about a novel.
Each world had 200 items and two kinds of control items that any model with the needed reading skill should answer.
One kind asks for the most recent mention of a candidate, and the other for an older mention that its content identifies.
The preregistered rule made a world usable when at least two models answered both kinds above chance by an exact lower bound.
No model did so in any world (\cref{tab:12}), so no world was usable and we ran no test.
Several models answered one kind of control item well and the other below chance.

\begin{table}[htbp]
  \caption{\textbf{Control items of the zero-shot demonstration}, as accuracy over 200 items without inserted filler sentences.
    The recent items ask for the most recent mention and the older items for an older mention that its content identifies.
    Chance is 0.5, and bold marks an exact lower 95\% bound above chance.
    A world needed two models with both kinds in bold.
    FPA, First Passage Attention.
    FoX, Forgetting Transformer.
    Sel.\ Att., Selective Attention.}
  \label{tab:12}
  \centering
  \small
  \setlength{\tabcolsep}{4pt}
  \begin{tabular}{@{}lcccccccc@{}}
\toprule
 & \multicolumn{2}{c}{FPA} & \multicolumn{2}{c}{FoX} & \multicolumn{2}{c}{Sel.\ Att.} & \multicolumn{2}{c}{Softmax} \\
\cmidrule(lr){2-3}\cmidrule(lr){4-5}\cmidrule(lr){6-7}\cmidrule(lr){8-9}
World & recent & older & recent & older & recent & older & recent & older \\
\midrule
Scopes in Python code & 0.24 & \textbf{0.66} & 0.01 & \textbf{0.86} & 0.39 & \textbf{0.77} & 0.45 & 0.52 \\
Who saw a hidden object & 0.52 & \textbf{0.62} & 0.47 & \textbf{0.69} & 0.41 & \textbf{0.67} & \textbf{0.69} & 0.49 \\
Branches in a version history & 0.41 & \textbf{0.60} & 0.40 & \textbf{0.60} & 0.27 & \textbf{0.73} & 0.45 & 0.56 \\
Stop-loss orders & 0.51 & \textbf{0.59} & \textbf{0.79} & 0.33 & 0.36 & 0.41 & 0.41 & 0.53 \\
A parenthetical remark & 0.55 & 0.56 & \textbf{0.69} & 0.30 & \textbf{0.77} & 0.38 & \textbf{0.77} & 0.33 \\
\bottomrule
\end{tabular}

\end{table}
We report the design for completeness and draw no conclusion from it.

\subsection{Observational analysis}
\label{app:C.6}

A preregistered analysis read the final FPA checkpoint on the SlimPajama test stream, on the CPU and without changing the model.
For the alive sets it used the first 256 windows of 2048 tokens and the queries at positions from 1024 on.
A key counted as dead when its bias was below $-10$ and as alive when its bias was at least $-1$.
For the brackets it used the first 256 windows inside one document with at least 32 bracket tokens, and it compared the walk with the bracket depth by a Spearman correlation after removing a linear trend in position from both.
A null distribution came from 1000 cyclic shifts of the depth within each window, and the familywise test took the maximum over the 144 heads.

The median head marked 96.5\% of the keys as dead and 0.3\% as alive.
In 19 heads more than 10\% of the alive sets were not suffixes, and in two heads of the first layer almost all were, with most of the attention on alive keys older than the newest dead key.
In 60 heads the mean increment at opening brackets and at closing brackets had opposite signs.
In 103 heads the median correlation lay outside the band of their own null distribution, against 1.44 expected by chance, and 82 heads passed the familywise test at $0.01$, 81 of them with a positive correlation.
The correlations were small, with a median absolute value of 0.032 over heads and a maximum of 0.248, and their significance comes from their consistency over 256 windows.
The detrended correlation mixes the level of the depth with the order of positions within a level, so we read these numbers as descriptions, not as evidence that the model tracks brackets.

\begin{figure}[htbp]
  \centering
  \includegraphics{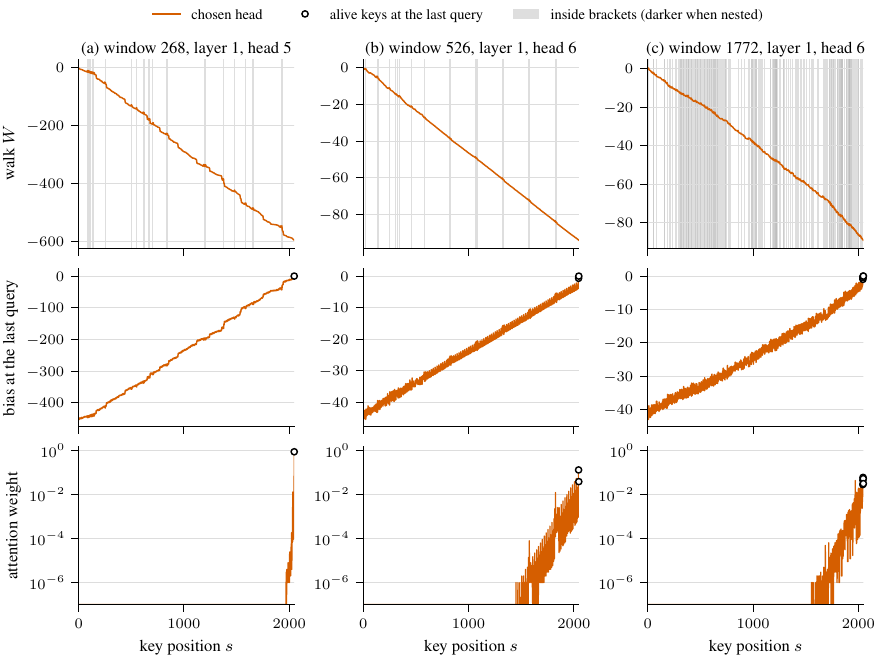}
  \caption{\textbf{The three preregistered windows of the observational analysis.}
    Each column shows one window with the head that the preregistered rule chose for it, with layers and heads counted from 1.
    The top panels show the walk~$W$ of that head, with the bracket depth shaded.
    The middle and bottom panels show the bias and the attention weight of every key at the last query of the window, with the alive keys marked, among them the query itself.
    Attention weights below $10^{-7}$ are drawn at that value.
    The walks of all three heads fall almost linearly, so their biases resemble that of FoX.
    FoX, Forgetting Transformer.}
  \label{fig:8}
\end{figure}

\Cref{fig:8} shows the three windows that a fixed seed chose before the analysis, each with the head of the largest standardised contrast between opening and closing brackets in that window.
The preregistration also placed the first window in the main text, and we moved it here for space.
In window 268 the only alive key at the last query is the query itself, which receives 0.90 of the attention.

\Cref{fig:4} looks at the steps of the walk instead of the walk itself.
In window 268 the chosen head stepped up by 4.70 on average where the depth rose and down by 5.51 where it fell, against $-0.29$ at the other tokens.
Over all 256 windows, 56 heads had increments of opposite signs at opening and closing brackets together with a familywise significant correlation.
\Cref{fig:9} shows the same alignment in the two other windows, with steps about ten times smaller.

\begin{figure}[htbp]
  \centering
  \includegraphics{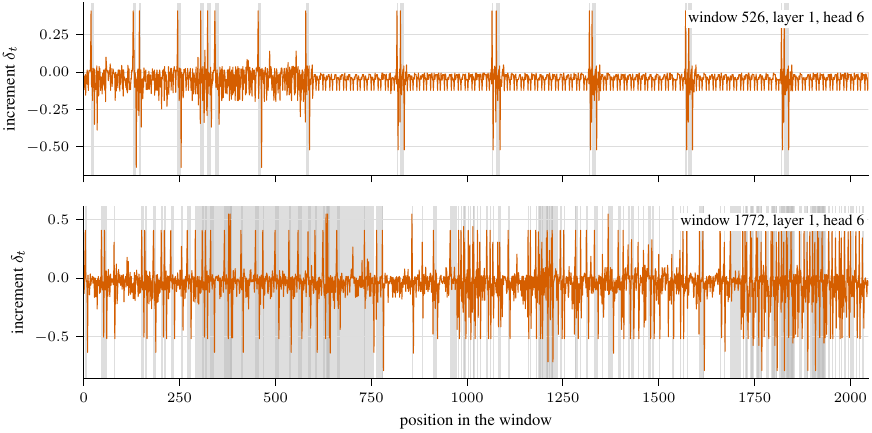}
  \caption{\textbf{Steps of the walk at brackets in the two other preregistered windows.}
    Increments $\inc_t$ of layer 1, head 6, the head that the rule chose for both windows, along the whole window, with the bracket depth shaded and darker where brackets nest.
    The steps are about ten times smaller than in \cref{fig:4} but keep their signs, with means of $+0.41$ and $+0.38$ where the depth rises and $-0.50$ and $-0.51$ where it falls.}
  \label{fig:9}
\end{figure}

\Cref{fig:10} shows the share of alive keys for every head of the final model.
Most heads keep a few keys alive, and a few heads keep nearly all of them.
The 19 heads whose alive sets are often not suffixes lie in nine of the twelve layers, and their shares of alive keys range from 0.3\% to 92\%.
\Cref{fig:11} counts where keys die.
In the median head the share of deaths at $\tokbos$ and $\langle/\mathrm{s}\rangle$ was 6.6 times their share of the tokens.
We took as local the 95 heads with a share of alive keys below 0.03.
On the 12 windows of the exploratory checks of \cref{app:C.7}, this threshold agrees with a median alive span below 64 tokens for 138 of the 144 heads of the final model.
Averaged over these heads, a token killed 0.98 keys per occurrence, about one for each new token.
The end-of-document token killed 15 keys, and capitalised words without a space before them killed 4.5 to 8.2.

\begin{figure}[htbp]
  \centering
  \includegraphics{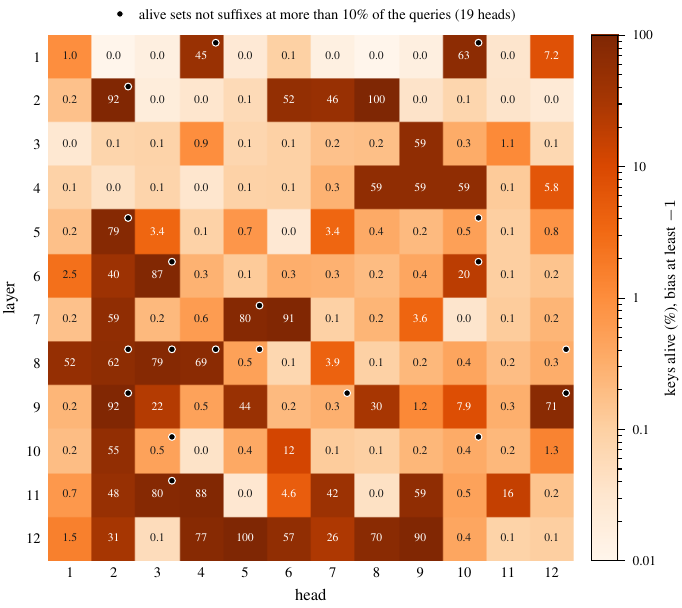}
  \caption{\textbf{Alive keys of every head in the final FPA model.}
    Each cell gives the percentage of keys with a bias of at least $-1$ over the queries from position 1024 on in 256 test windows, with layers and heads counted from 1.
    Dots mark the 19 heads whose alive sets are not suffixes at more than 10\% of the queries.
    FPA, First Passage Attention.}
  \label{fig:10}
\end{figure}

\begin{figure}[htbp]
  \centering
  \includegraphics{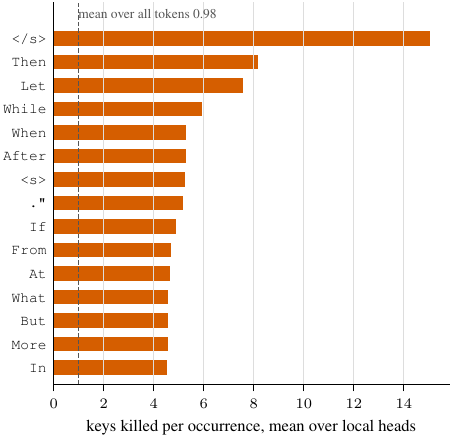}
  \caption{\textbf{Where keys die in the final FPA model.}
    Keys killed per occurrence of a token, averaged over the 95 local heads, for the 15 tokens with the highest rate among those occurring at least 30 times in 256 test windows.
    A visible space marks a token that starts with a space, and the dashed line is the mean over all tokens.
    FPA, First Passage Attention.}
  \label{fig:11}
\end{figure}

\subsection{Exploratory checks at three checkpoints}
\label{app:C.7}

We inspected the FPA model at steps 1500 and 3250 of 5150 during training and at its final checkpoint, always on the same 12 validation windows of 2048 tokens.
These checks were exploratory and outside every preregistered test.
They used the reference implementation in single precision on the CPU and the thresholds of \cref{app:C.6}, and the figures redraw the numbers recorded at the time.
The figures show the final checkpoint and add the two earlier steps where training changed the result.
\Cref{tab:13} summarises the three checks.

\begin{table}[htbp]
  \caption{\textbf{Three checks of the FPA language model}, at steps 1500 and 3250 of 5150 and at the final step, on the same 12 validation windows of 2048 tokens.
    Local, middle and global heads have median alive spans below 64 tokens, from 64 to 1023 tokens and from 1024 on.
    Alive sets count as often not suffixes when at least 20\% of the queries have an alive key older than a dead one.
    Document heads keep at least 80\% of the keys within a document and at most 5\% across documents.
    The copying items differ from those of \cref{fig:7}.
    FPA, First Passage Attention.}
  \label{tab:13}
  \centering
  \small
  \begin{tabular}{@{}lccc@{}}
\toprule
 & Step 1500 & Step 3250 & Step 5150 \\
\midrule
Mean loss over the 12 windows & 3.162 & 2.917 & 2.821 \\
Mean loss with the bias of every layer removed & 6.591 & 7.335 & 7.828 \\
Loss increase with alive sets forced into suffixes & $+$0.032 & $+$0.031 & $+$0.032 \\
\quad windows with a positive increase & 12 of 12 & 12 of 12 & 12 of 12 \\
Local, middle and global heads & 90, 22, 32 & 90, 20, 34 & 90, 20, 34 \\
Heads whose alive sets are often not suffixes & 21 & 15 & 14 \\
Heads whose walk is not monotone & 106 & 109 & 112 \\
Document heads & 7 & 9 & 9 \\
Copying accuracy, gap of 0 tokens & 0.89 & 0.93 & 0.93 \\
Copying accuracy, gap of 320 tokens & 0.78 & 0.82 & 0.84 \\
Copying accuracy, gap of 1000 tokens & 0.61 & 0.63 & 0.65 \\
\bottomrule
\end{tabular}

\end{table}

Removing the bias of every layer raised the loss on the 12 windows by 3.43, 4.42 and 5.01 nats at the three steps.
On the first four windows the first layer carried most of this increase, and its part grew from 1.49 to 2.67 and 3.48 nats (\cref{fig:12}).
At the final step the loss without the bias rose with position from 6.1 to 8.7 nats, while with the bias it fell from 3.9 to 2.9 (\cref{fig:13}).
Forcing the alive sets into suffixes, the sets that a forget gate can express, raised the loss on all 12 windows at every step, by 0.032, 0.031 and 0.032 nats on average (\cref{fig:14}).

\begin{figure}[htbp]
  \centering
  \includegraphics{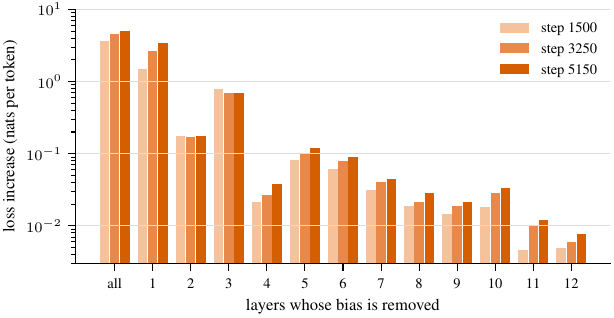}
  \caption{\textbf{The bias removed at three checkpoints of the FPA model.}
    Loss increase on the first four validation windows when the bias of all layers or of one layer is set to zero, at steps 1500, 3250 and 5150, the last step of training, on a logarithmic axis.
    FPA, First Passage Attention.}
  \label{fig:12}
\end{figure}

\begin{figure}[htbp]
  \centering
  \includegraphics{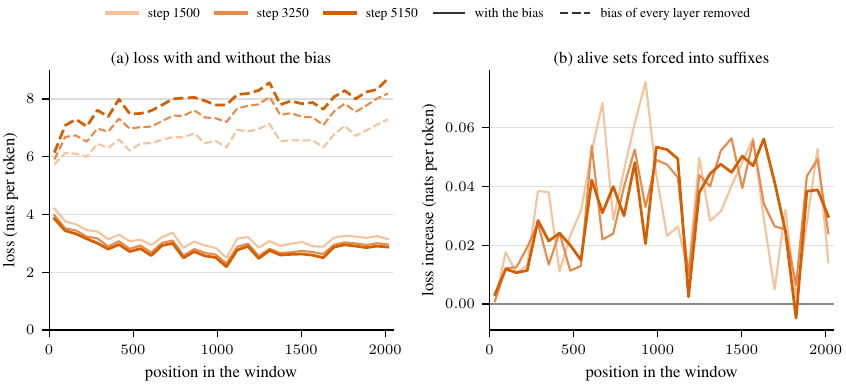}
  \caption{\textbf{Loss by position at three checkpoints.}
    (a)~Loss per bin of 64 positions over the 12 validation windows, with the bias and with the bias of every layer removed.
    (b)~Loss increase when the alive sets of every head are forced into suffixes, the sets that a forget gate can express.}
  \label{fig:13}
\end{figure}

\begin{figure}[htbp]
  \centering
  \includegraphics{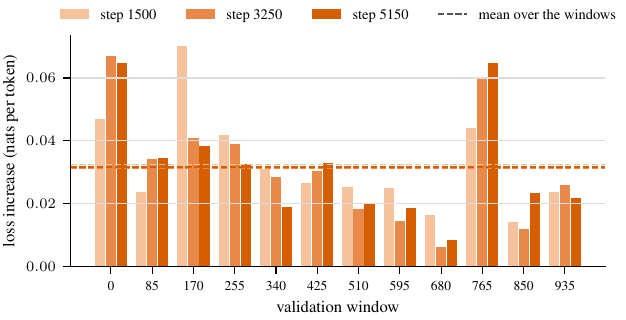}
  \caption{\textbf{Alive sets forced into suffixes, per validation window.}
    Loss increase on each of the 12 windows at steps 1500, 3250 and 5150, with the means of 0.032, 0.031 and 0.032 nats as dashed lines.}
  \label{fig:14}
\end{figure}

\Cref{fig:15} shows the median alive span of every head of the final model.
Of the 144 heads, 90 kept keys alive for fewer than 64 positions, 20 for 64 to 1023 and 34 for longer, the same counts as at step 3250.
Between steps 1500 and 3250 the spans of 17 heads changed by a factor of two or more, and between step 3250 and the end only those of two heads did.
Averaged over the local heads, a token killed about one key per occurrence at every step, 0.99, 0.98 and 0.98 (\cref{fig:16}).
The six tokens that killed most were the same at all three steps, and their rates grew during training, from 1.90 to 2.55 for a newline and from 3.17 to 3.84 for the token \texttt{The} without a space before it.

\begin{figure}[htbp]
  \centering
  \includegraphics{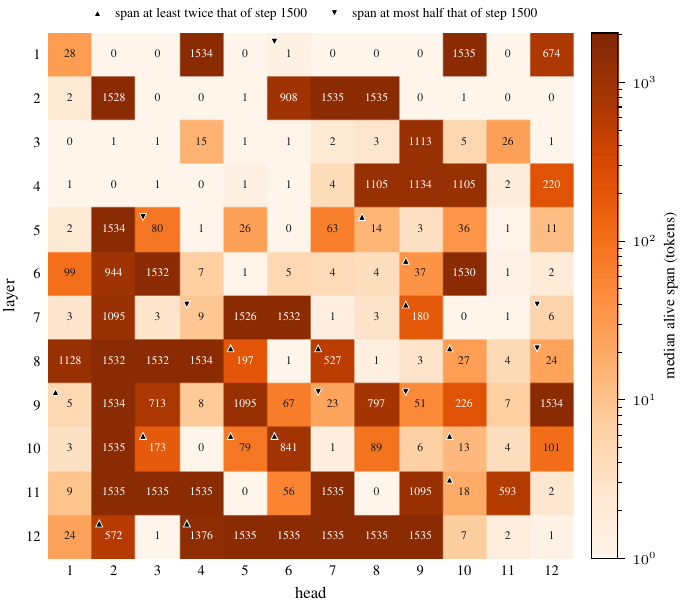}
  \caption{\textbf{Median alive span of every head in the final FPA model.}
    Spans over the queries 1024 to 2047 of the 12 validation windows, with layers and heads counted from 1.
    Triangles mark the 21 heads whose span grew or shrank by a factor of two or more since step 1500.
    FPA, First Passage Attention.}
  \label{fig:15}
\end{figure}

\begin{figure}[htbp]
  \centering
  \includegraphics{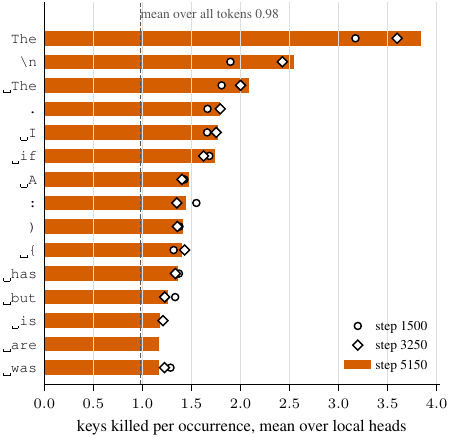}
  \caption{\textbf{Where keys die at three checkpoints.}
    Keys killed per occurrence of a token, averaged over the 90 local heads of each step, those with a median alive span below 64 tokens.
    The bars give the final step for its 15 tokens with the highest rate among those occurring at least 30 times in the 12 validation windows.
    Markers give steps 1500 and 3250 where the token was among the 15 highest at that step.
    A visible space marks a token that starts with a space, \texttt{\textbackslash n} is the newline, and the dashed line is the mean over all tokens at the final step.}
  \label{fig:16}
\end{figure}

In every layer of the final model the attention that queries paid to earlier documents stayed between 0.002 and 0.10 with the bias (\cref{fig:17}).
Without the bias it rose to 0.10 to 0.21, against 0.42 for weights spread evenly over the visible keys.
\Cref{fig:18} shows, for heads that we selected by inspection, the share of keys of each token class that stayed alive 256 positions later.
Layer 1, head 10 kept capitalised words (0.93) and word continuations (0.83) but no digits.
Its alive keys therefore do not form a suffix, which no forget gate can produce.
The document heads kept nearly every key within a document and none across a document boundary.

\begin{figure}[htbp]
  \centering
  \includegraphics{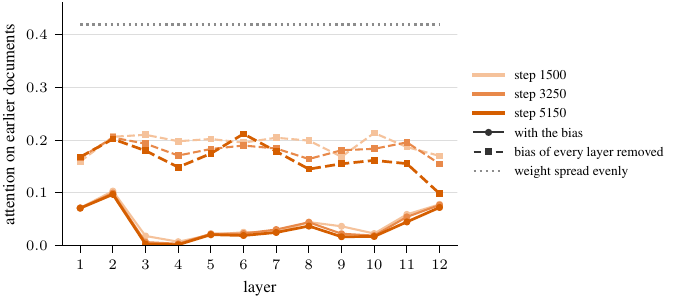}
  \caption{\textbf{Attention on earlier documents at three checkpoints.}
    Attention that queries pay to the keys of earlier documents, averaged over the heads of each layer, with the bias and with the bias of every layer removed.
    The dotted line is the share that weights spread evenly over the visible keys would give.}
  \label{fig:17}
\end{figure}

\begin{figure}[!htbp]
  \centering
  \includegraphics{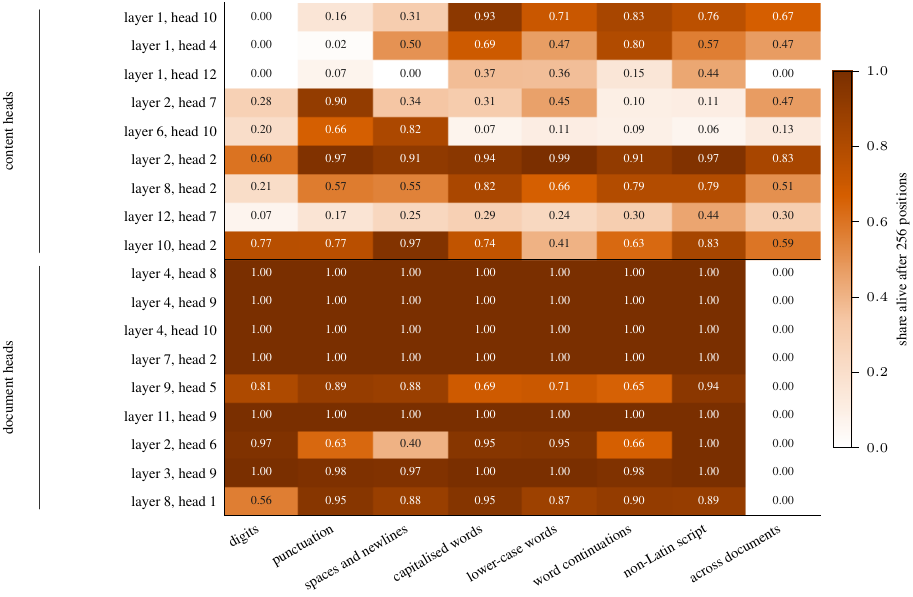}
  \caption{\textbf{Retention by token class in the final FPA model.}
    Share of the keys of each class that are still alive 256 positions later, for pairs of key and query in one document, and in the last column for pairs across a document boundary.
    Content heads lie above the line and the document heads of step 3250 below it, with layers and heads counted from 1.
    By the final step layer 9, head 5 kept only 76\% of the keys within a document and no longer counted as a document head.
    FPA, First Passage Attention.}
  \label{fig:18}
\end{figure}

\Cref{fig:19,fig:20,fig:21} show the bias of three heads of the final model on stretches of 64 tokens.
Layer 1, head 1 let groups of keys stop being alive at the end of a statement and at newlines, dots and underscores in C++ code.
Layer 1, head 10 let articles, prepositions and punctuation lose their weight at the next token in a news report and kept content words.
Layer 4, head 9 removed the previous document at its end-of-document token.
Copying improved during training (\cref{tab:13}).
The ablations show that the trained FPA model relies on these structures, not that a model trained with FoX could not reach the same function in another way.

\begin{figure}[htbp]
  \centering
  \includegraphics{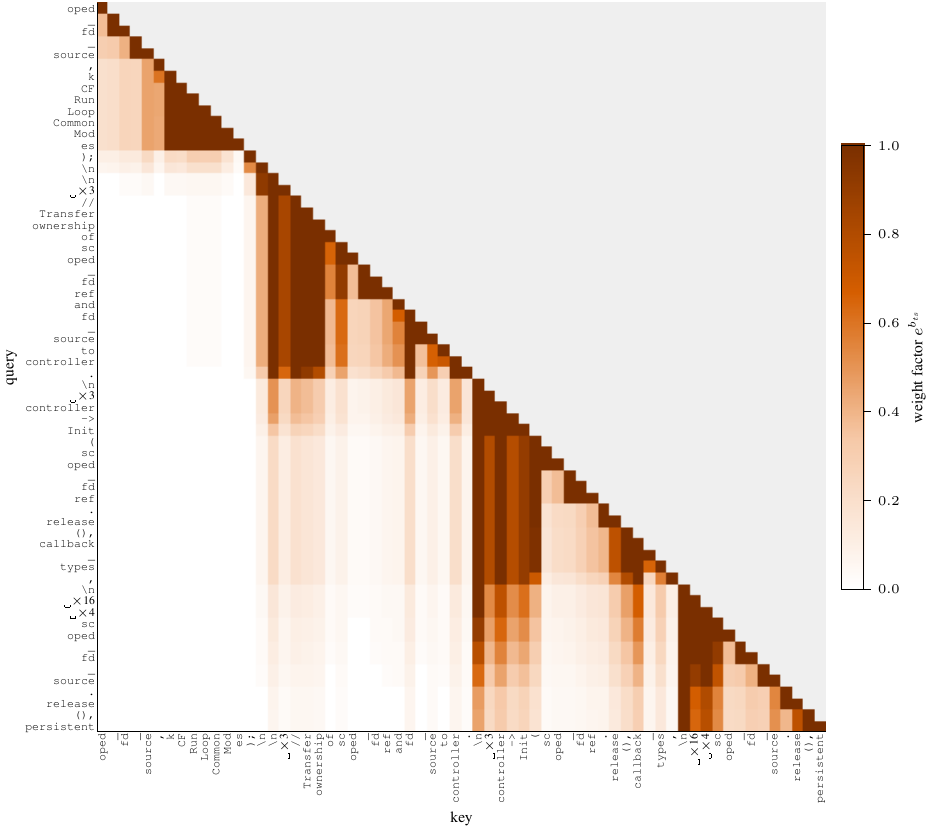}
  \caption{\textbf{The bias of layer 1, head 1 on C++ code in the final FPA model.}
    Weight factor $e^{b_{ts}}$ of every key~$s$ for every query~$t$ at positions 1100 to 1163 of validation window 0, with factors below 0.02 left white.
    A key is alive while its factor is at least $e^{-1}$.
    Groups of keys stop being alive at once, nine at the end of a statement and three to six at newlines, dots and underscores.
    FPA, First Passage Attention.}
  \label{fig:19}
\end{figure}

\begin{figure}[htbp]
  \centering
  \includegraphics{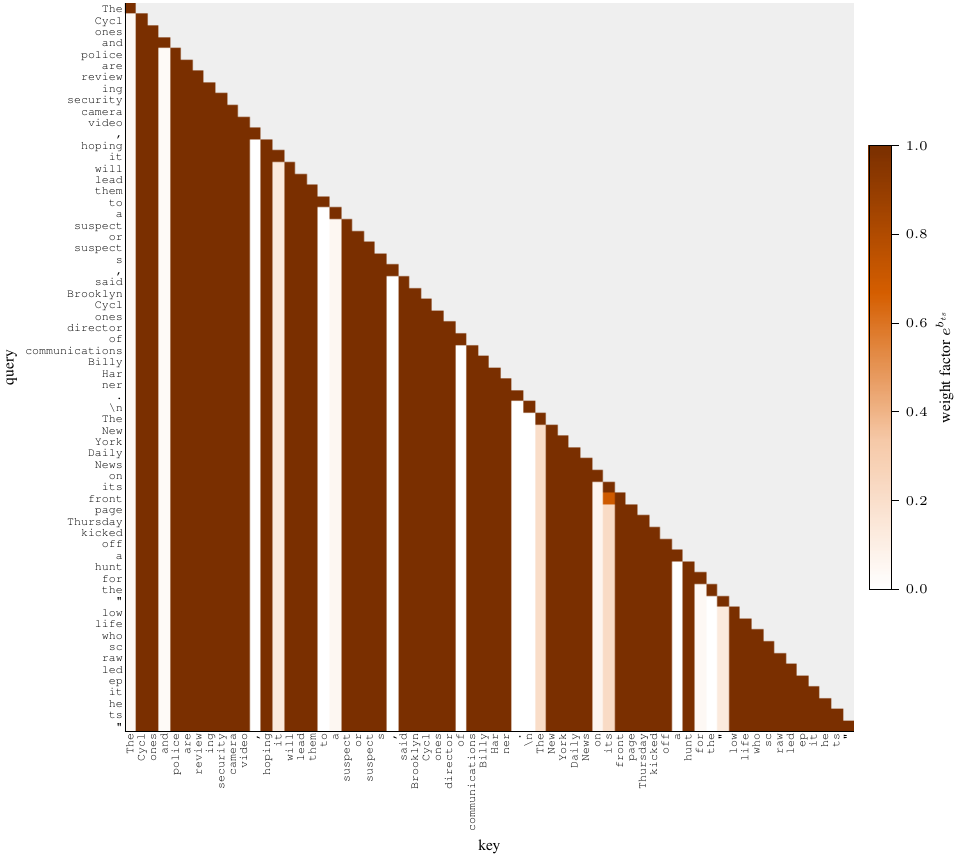}
  \caption{\textbf{The bias of layer 1, head 10 on a news report in the final FPA model.}
    Weight factor $e^{b_{ts}}$ at positions 1100 to 1163 of validation window 680, with factors below 0.02 left white.
    Articles, prepositions, punctuation and newlines stop being alive at the next token, while content words such as \emph{said} and \emph{page} stay alive to the end of the stretch.
    FPA, First Passage Attention.}
  \label{fig:20}
\end{figure}

\begin{figure}[htbp]
  \centering
  \includegraphics{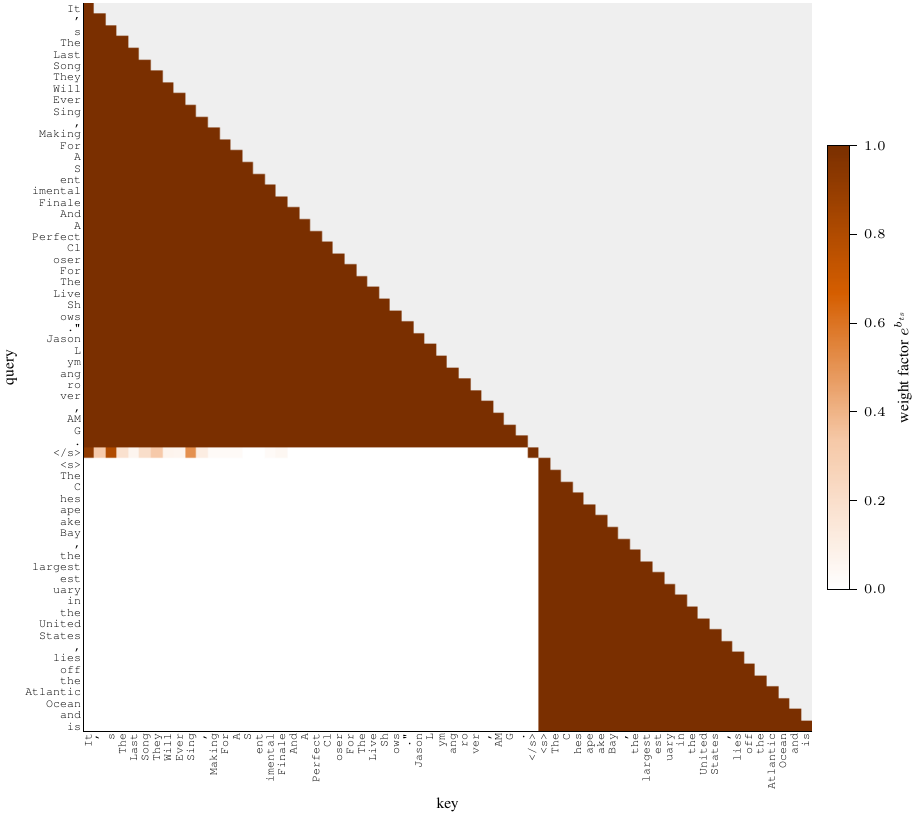}
  \caption{\textbf{The bias of layer 4, head 9 at a document boundary in the final FPA model.}
    Weight factor $e^{b_{ts}}$ at positions 233 to 296 of validation window 680, with factors below 0.02 left white.
    At the end-of-document token 36 of the 39 earlier keys of the stretch stop being alive, and at the start-of-document token the others do.
    No key of the stretch before the boundary stays alive in the next document.
    FPA, First Passage Attention.}
  \label{fig:21}
\end{figure}

\clearpage
\section{Architectures and training configurations}
\label{app:D}

This appendix collects the architectures, the hyperparameters and the searches of all experiments in one place, with every value read from the records of the runs.

\subsection{Architecture}
\label{app:D.1}

All our models are causal decoders with pre-norm blocks, and the methods differ only in the bias that each attention layer adds to its logits (\cref{tab:14}).
\Cref{fig:22} shows the FPA model down to the operations of one head.
FPA computes its three scalars per head with one linear map from the same normalised input as the queries, keys and values.
The walk sums the increments over positions, the range minimum runs over the positions after each key, and the strength turns the depth of the deepest dip into the bias.
The bias enters every head of every layer before the softmax, and FPA uses no other positional information.
In training the forward pass takes the hard minimum and the backward pass the gradient of a soft minimum with $\temp=1$ (\cref{sec:2}).
Our kernel computes the range minimum in blocks of 64 positions without a matrix of size $T\times T$ (\cref{app:E}).

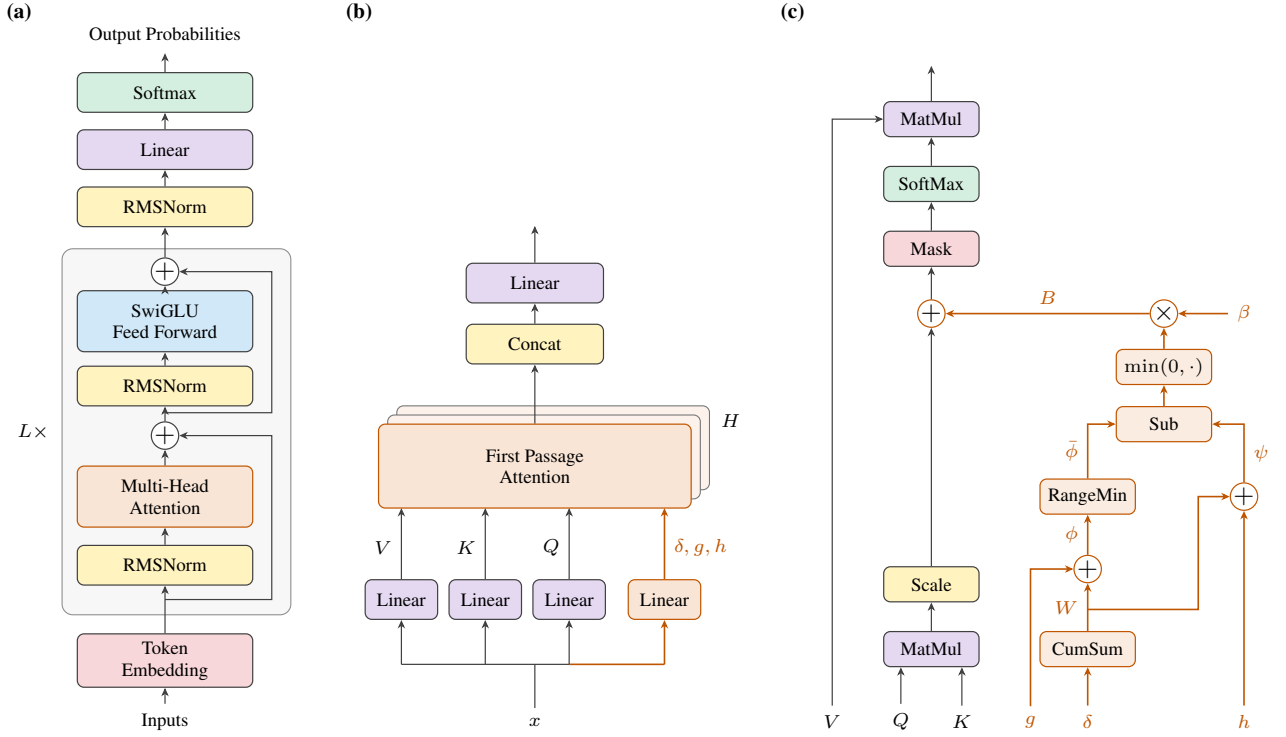
\begin{figure}[htbp]
  \centering
  \resizebox{\textwidth}{!}{\begin{tikzpicture}[x=1cm, y=1cm, >={Stealth[length=3.2pt,width=2.8pt]},
    blk/.style={draw=black!75, rounded corners=2.5pt, line width=0.45pt, align=center, font=\scriptsize,
                minimum height=0.52cm, inner sep=2pt},
    emb/.style={blk, fill=archpink},
    nrm/.style={blk, fill=archyellow},
    ffn/.style={blk, fill=archblue},
    lin/.style={blk, fill=archpurple},
    smx/.style={blk, fill=archgreen},
    fpa/.style={blk, fill=fpacol!16, draw=fpacol!90!black},
    op/.style={blk, minimum width=1.25cm, minimum height=0.46cm},
    fop/.style={op, fill=fpacol!12, draw=fpacol!90!black},
    sum/.style={circle, draw=black!75, fill=white, inner sep=0pt, minimum size=0.3cm, line width=0.45pt},
    fsum/.style={sum, draw=fpacol!90!black},
    sym/.style={font=\scriptsize},
    arr/.style={->, line width=0.45pt, draw=black!75},
    farr/.style={->, line width=0.6pt, draw=fpacol!90!black},
    ttl/.style={font=\footnotesize\bfseries, anchor=west}]

  \node[ttl] at (0.0, 9.55) {(a)};
  \node[sym] (in) at (2.2, 0.25) {Inputs};
  \node[emb, minimum width=2.3cm, minimum height=0.7cm] (te) at (2.2, 1.05) {Token\\Embedding};
  \draw[draw=black!45, fill=black!3, rounded corners=4pt] (0.85, 1.65) rectangle (3.85, 6.45);
  \node[sym, anchor=east] at (0.8, 4.05) {$L\times$};
  \node[nrm, minimum width=2.3cm] (n1) at (2.2, 2.3) {RMSNorm};
  \node[fpa, minimum width=2.3cm, minimum height=0.8cm] (mha) at (2.2, 3.2) {Multi-Head\\Attention};
  \node[sum] (r1) at (2.2, 4.0) {$+$};
  \node[nrm, minimum width=2.3cm] (n2) at (2.2, 4.65) {RMSNorm};
  \node[ffn, minimum width=2.3cm, minimum height=0.8cm] (ff) at (2.2, 5.5) {SwiGLU\\Feed Forward};
  \node[sum] (r2) at (2.2, 6.15) {$+$};
  \node[nrm, minimum width=2.3cm] (n3) at (2.2, 7.0) {RMSNorm};
  \node[lin, minimum width=2.3cm] (lo) at (2.2, 7.75) {Linear};
  \node[smx, minimum width=2.3cm] (so) at (2.2, 8.5) {Softmax};
  \node[sym, align=center] (out) at (2.2, 9.25) {Output Probabilities};
  \foreach \a/\b in {in/te, te/n1, n1/mha, mha/r1, r1/n2, n2/ff, ff/r2, r2/n3, n3/lo, lo/so, so/out}
    \draw[arr] (\a) -- (\b);
  \draw[arr] (2.2, 1.85) -- (3.6, 1.85) -- (3.6, 4.0) -- (r1);
  \draw[arr] (2.2, 4.3) -- (3.6, 4.3) -- (3.6, 6.15) -- (r2);

  \node[ttl] at (4.45, 9.55) {(b)};
  \node[sym] (x) at (7.05, 0.25) {$x$};
  \node[lin, minimum width=0.95cm] (lv) at (5.3, 1.85) {Linear};
  \node[lin, minimum width=0.95cm] (lk) at (6.4, 1.85) {Linear};
  \node[lin, minimum width=0.95cm] (lq) at (7.5, 1.85) {Linear};
  \node[fpa, minimum width=0.95cm] (lf) at (8.75, 1.85) {Linear};
  \draw[line width=0.45pt, draw=black!75] (x) -- (7.05, 1.0);
  \draw[line width=0.45pt, draw=black!75] (5.3, 1.0) -- (7.5, 1.0);
  \draw[line width=0.6pt, draw=fpacol!90!black] (7.5, 1.0) -- (8.75, 1.0);
  \foreach \c in {lv, lk, lq} \draw[arr] (\c |- 0, 1.0) -- (\c);
  \draw[farr] (8.75, 1.0) -- (lf);
  \foreach \k in {2, 1} \draw[draw=black!55, fill=fpacol!8, rounded corners=2.5pt]
    ($(5.0, 3.05) + (0.12*\k, 0.12*\k)$) rectangle ($(9.1, 4.15) + (0.12*\k, 0.12*\k)$);
  \node[fpa, minimum width=4.1cm, minimum height=1.1cm] (sdpa) at (7.05, 3.6) {First Passage\\Attention};
  \node[sym] at (9.62, 4.2) {$H$};
  \draw[arr] (lv) -- node[sym, left, pos=0.45] {$V$} (5.3, 3.05);
  \draw[arr] (lk) -- node[sym, left, pos=0.45] {$K$} (6.4, 3.05);
  \draw[arr] (lq) -- node[sym, left, pos=0.45] {$Q$} (7.5, 3.05);
  \draw[farr] (lf) -- node[sym, right, pos=0.45, text=fpacol!90!black] {$\inc, g, h$} (8.75, 3.05);
  \node[nrm, minimum width=1.8cm] (cat) at (7.05, 5.2) {Concat};
  \node[lin, minimum width=1.8cm] (lo2) at (7.05, 6.0) {Linear};
  \draw[arr] (sdpa) -- (cat);
  \draw[arr] (cat) -- (lo2);
  \draw[arr] (lo2) -- (7.05, 6.75);

  \node[ttl] at (10.15, 9.55) {(c)};
  \node[sym] (V) at (10.95, 0.25) {$V$};
  \node[sym] (Q) at (11.85, 0.25) {$Q$};
  \node[sym] (K) at (12.65, 0.25) {$K$};
  \node[sym, text=fpacol!90!black] (g) at (13.55, 0.25) {$g$};
  \node[sym, text=fpacol!90!black] (d) at (14.3, 0.25) {$\inc$};
  \node[sym, text=fpacol!90!black] (h) at (16.35, 0.25) {$h$};
  \node[op, fill=archpurple] (mm1) at (12.25, 1.2) {MatMul};
  \node[op, fill=archyellow] (sc) at (12.25, 2.05) {Scale};
  \node[fsum] (add) at (12.25, 5.6) {$+$};
  \node[op, fill=archpink] (mk) at (12.25, 6.45) {Mask};
  \node[op, fill=archgreen] (sm) at (12.25, 7.3) {SoftMax};
  \node[op, fill=archpurple] (mm2) at (12.25, 8.15) {MatMul};
  \draw[arr] (Q) -- (Q |- mm1.south);
  \draw[arr] (K) -- (K |- mm1.south);
  \draw[arr] (mm1) -- (sc);
  \draw[arr] (sc) -- (add);
  \draw[arr] (add) -- (mk);
  \draw[arr] (mk) -- (sm);
  \draw[arr] (sm) -- (mm2);
  \draw[arr] (V) -- (10.95, 8.15) -- (mm2.west);
  \draw[arr] (mm2) -- (12.25, 8.85);
  \node[fop] (cs) at (14.3, 1.2) {CumSum};
  \node[fsum] (ap) at (14.3, 2.25) {$+$};
  \node[fop] (rm) at (14.3, 3.2) {RangeMin};
  \node[fsum] (aq) at (16.35, 3.2) {$+$};
  \node[fop] (sb) at (15.3, 4.15) {Sub};
  \node[fop] (mn) at (15.3, 4.9) {$\min(0,\cdot)$};
  \node[fsum] (mb) at (15.3, 5.6) {$\times$};
  \node[sym, text=fpacol!90!black] (bt) at (16.35, 5.6) {$\str$};
  \draw[farr] (d) -- (cs);
  \draw[farr] (cs) -- node[sym, left, pos=0.5, text=fpacol!90!black] {$W$} (ap);
  \draw[farr] (g) -- (13.55, 2.25) -- (ap);
  \draw[farr] (14.3, 1.72) -- (15.75, 1.72) -- (15.75, 3.2) -- (aq);
  \draw[farr] (h) -- (aq);
  \draw[farr] (ap) -- node[sym, left, pos=0.5, text=fpacol!90!black] {$\fld$} (rm);
  \draw[farr] (rm) |- node[sym, left, pos=0.25, text=fpacol!90!black] {$\rmin$} (sb.west);
  \draw[farr] (aq) |- node[sym, right, pos=0.25, text=fpacol!90!black] {$\lvl$} (sb.east);
  \draw[farr] (sb) -- (mn);
  \draw[farr] (mn) -- (mb);
  \draw[farr] (bt) -- (mb);
  \draw[farr] (mb) -- node[sym, above, pos=0.5, text=fpacol!90!black] {$B$} (add);
\end{tikzpicture}}
  \caption{\textbf{Architecture of the FPA language model.}
    (a)~Pre-norm decoder with $L=12$ blocks of width 768.
    (b)~Multi-head attention with $H=12$ heads of dimension $d=64$, in which a fourth linear map gives every head its increments $\inc$ and offsets $g$ and $h$.
    (c)~One head, in which the bias $B$ with the entries $b_{ts}$ joins the scaled scores before the causal mask.
    RangeMin returns $\rmin_{ts}=\min_{s<u\le t}\fld_u$ for every query~$t$ and key~$s$, and Sub returns $\rmin_{ts}-\lvl_s$.
    The parts of FPA are orange, and the one-layer probes use the attention of panel~(b) alone on token embeddings.
    FPA, First Passage Attention.
    RMSNorm, root mean square normalisation.
    SwiGLU, gated linear unit with the Swish activation.}
  \label{fig:22}
\end{figure}

\begin{table}[htbp]
  \caption{\textbf{Bias that each method adds to the logit} $q_t^{\top}k_s/\sqrt d$ of key~$s$ at query~$t$, with the parameters that it adds to a layer of width $D=768$ with $H=12$ heads and its positional encoding in the language models.
    FoX has one gate $f_u$ per head with a weight vector $w$, a bias $c$ and the logistic function $\sigma$.
    Selective Attention takes $q^{(1)}$ and $k^{(1)}$ from its first head, shares the bias among all heads and never masks the start token.
    FPA, First Passage Attention.
    FoX, Forgetting Transformer.
    Sel.\ Att., Selective Attention.
    RoPE, rotary position embedding.}
  \label{tab:14}
  \centering
  \small
  \begin{tabular}{@{}llll@{}}
\toprule
Method & Bias $b_{ts}$ added to the logit & Parameters per layer & Positions \\
\midrule
FPA & $\str\min\bigl(0,\min_{s<u\le t}\fld_u-\lvl_s\bigr)$ with $\fld=W+g$, $\lvl=W+h$ & $3H(D+1)+H=27{,}696$ & none \\
FoX & $\sum_{s<u\le t}\log f_u$ with $f_u=\sigma(w^{\top}x_u+c)$ & $H(D+1)=9{,}228$ & none \\
Sel.\ Att. & $-\sum_{s<u<t}\max\bigl(0,\,q_u^{(1)\top}k_s^{(1)}/\sqrt d\bigr)$ & 0 & RoPE \\
Softmax & 0 & 0 & RoPE \\
\bottomrule
\end{tabular}

\end{table}

\subsection{Language models}
\label{app:D.2}

\Cref{tab:15} describes the four language models and \cref{tab:16} their common training.
All four used the same data in the same order, the same schedule and the hyperparameters of the 125M recipe of FoX, without a search.
They differ in the attention, its positional encoding, its kernel and the split of each batch into micro-batches, which changes the gradient only through rounding.

\begin{table}[htbp]
  \caption{\textbf{The four language models.}
    Parameters count the tied embeddings once.
    Throughput and time cover the whole run of 5150 steps with its validations and checkpoints on the GPU of \cref{app:C}.
    For FPA, throughput, time and memory come from an exploratory run of the same configuration with seed 2202 and the kernel of \cref{sec:4.3}, outside the preregistration.
    \Cref{tab:3} instead gives the throughput that we measured in the benchmark over 12 training steps after 8 steps of warm-up.
    FPA, First Passage Attention.
    FoX, Forgetting Transformer.
    Sel.\ Att., Selective Attention.
    MLP, multilayer perceptron.
    SwiGLU, gated linear unit with the Swish activation.
    RMSNorm, root mean square normalisation.
    RoPE, rotary position embedding.
    SDPA, the scaled dot-product attention of PyTorch.
    GPU, graphics processing unit.
    GiB, gibibyte.}
  \label{tab:15}
  \centering
  \small
  \begin{tabular}{@{}lcccc@{}}
\toprule
 & FPA & FoX & Sel.\ Att. & Softmax \\
\midrule
Parameters & 109,862,208 & 109,640,592 & 109,529,856 & 109,529,856 \\
\quad of the attention mechanism & 332,352 & 110,736 & 0 & 0 \\
Layers and width & \multicolumn{4}{c}{12 and 768} \\
Heads $\times$ head dimension & \multicolumn{4}{c}{12 $\times$ 64} \\
MLP & \multicolumn{4}{c}{SwiGLU with hidden width 2048} \\
Normalisation & \multicolumn{4}{c}{RMSNorm before each sublayer and before the output layer} \\
Linear layers & \multicolumn{4}{c}{without biases, except the field projection of FPA and the gates of FoX} \\
Vocabulary & \multicolumn{4}{c}{32,000, input and output embeddings tied} \\
Positional encoding & none & none & RoPE & RoPE \\
Attention kernel & ours, Triton & of FoX, Triton & ours, Triton & SDPA \\
Micro-batch $\times$ accumulation & 4 $\times$ 64 & 4 $\times$ 64 & 2 $\times$ 128 & 4 $\times$ 64 \\
Tokens per second & 31,473 & 34,870 & 19,298 & 28,517 \\
Training time (hours) & 23.8 & 21.5 & 38.9 & 26.3 \\
Peak GPU memory (GiB) & 9.3 & 9.0 & 5.8 & 9.2 \\
\bottomrule
\end{tabular}

\end{table}

\begin{table}[htbp]
  \caption{\textbf{Training of the language models}, the same for all four.
    FPA, First Passage Attention.
    Sel.\ Att., Selective Attention.
    FoX, Forgetting Transformer.
    ReLU, rectified linear unit.}
  \label{tab:16}
  \centering
  \small
  \begin{tabular}{@{}ll@{}}
\toprule
Setting & Value \\
\midrule
Data & SlimPajama, windows of 2049 tokens in one fixed order, each used at most once \\
Tokenizer & that of Mistral 7B, 32,000 tokens \\
Context & 2048 tokens \\
Batch & 256 windows, 524,288 tokens \\
Steps & 5150, 2.70B tokens \\
Optimiser & AdamW with $(0.9, 0.95)$ \\
Learning rate & $2\times10^{-3}$, linear warm-up over 512 steps, cosine decay to 0 \\
Weight decay & 0.1 on weight matrices and embeddings, none on gains, biases and strengths \\
Gradient clipping & 1.0 \\
Precision & half-precision autocast with loss scaling, the field of FPA and the selection of Sel.\ Att.\ in single precision \\
Initialisation & embeddings $\mathcal N(0,0.02^2)$, linear layers as in PyTorch, $\str=1$ in FPA, gate bias 0 in FoX \\
Minimum of FPA & hard in the forward pass, gradient of the soft minimum with $\temp=1$ \\
Selection of Sel.\ Att. & gradient of the ReLU, as in its paper \\
Seed & 2201 for the initialisation and the order of the data, in all models \\
Validation & every 250 steps on 1024 windows \\
Hyperparameter search & none, the 125M recipe of FoX in all models \\
\bottomrule
\end{tabular}

\end{table}

\subsection{Software and compute}
\label{app:D.3}

All runs used PyTorch 2.5.1 with CUDA 12.1 and Triton 3.1.0 under Python 3.12 on the GPU of \cref{app:C}.
The preregistered batch of probes and tasks had 1,815 runs, 15 seeds in each of 121 settings, which took 73.5 hours summed over the runs.
Training the four language models with the kernels of \cref{tab:15} took 110.5 hours.

\subsection{One-layer probes and tasks from earlier work}
\label{app:D.4}

\Cref{tab:17} gives the training of the probes of \cref{app:B} and of the small decoders of \cref{app:B.7}, and \cref{tab:18,tab:19} give their models.
\Cref{tab:20} gives the declared search, and \cref{tab:21} the configurations of the reported runs.
The pilots selected these configurations by the rules of \cref{app:B}, and a file with a hash froze them before the reported runs.

\begin{table}[htbp]
  \caption{\textbf{Training of the one-layer probes and of the tasks from earlier work.}
    FPA, First Passage Attention.
    MQAR, multi-query associative recall.}
  \label{tab:17}
  \centering
  \small
  \begin{tabular}{@{}lll@{}}
\toprule
 & One-layer probes & Tasks from earlier work \\
\midrule
Optimiser & \multicolumn{2}{l}{AdamW with $(0.9, 0.98)$, weight decay 0.01 on every parameter} \\
Batch & \multicolumn{2}{l}{64 sequences, drawn afresh at every step} \\
Steps & 4000, 16,000 on the graded depths & 4000, 8000 for MQAR \\
Warm-up & 200 steps & 10\% of the steps \\
Schedule & \multicolumn{2}{l}{cosine decay to 0, gradient clipping at 1.0} \\
Precision & \multicolumn{2}{l}{single} \\
Loss & cross-entropy of the weights at the final query & cross-entropy at the target positions \\
Width & 64, the token embedding & 128, and 64 for flip-flop \\
Seeds & \multicolumn{2}{l}{1801 to 1815 for the reported runs, 1616 to 1619 for tuning} \\
FPA & \multicolumn{2}{l}{straight-through minimum with $\temp=1$, initial strength $\str=1$ per head} \\
\bottomrule
\end{tabular}

\end{table}

\begin{table}[htbp]
  \caption{\textbf{Models of the one-layer probes}, with one supervised attention layer on token embeddings of width 64 and the head dimension~$d$.
    On the alive sets every key has a learned offset per position and head, FPA, FoX, Selective Attention and softmax attention add RoPE, ALiBi adds its fixed slope, and NoPE and stick-breaking attention use the offsets alone.
    The positive controls of the alive sets use $d=r+2$, and the field-only FPA of \cref{tab:5} uses $d=1$ without positional information.
    The two- and four-layer models put one or three pre-norm blocks with 4 heads and an MLP of width 256 before the supervised layer.
    On the alive sets and the two walks an MLP of the one-hot sequence with a hidden layer of width 256 gives the final query.
    FPA, First Passage Attention.
    FoX, Forgetting Transformer.
    RoPE, rotary position embedding.
    ALiBi, attention with linear biases.
    NoPE, no positional encoding.
    MLP, multilayer perceptron.}
  \label{tab:18}
  \centering
  \small
  \begin{tabular}{@{}llllll@{}}
\toprule
Family & Heads $\times$ $d$ & Positions & Final query & Steps & Test set \\
\midrule
\multicolumn{6}{@{}l}{\emph{One-layer probes, width 64}} \\
Alive sets & 1 $\times$ 4 & key offsets, RoPE or ALiBi & MLP of the sequence & 4000 & all $2^{r-1}$ \\
\quad two and four layers & 1 $\times$ 4 & RoPE & last token & 4000 & all $2^{r-1}$ \\
Graded depths & 1 $\times$ 8 & none & token \tokq & 16,000 & 2,000 per set \\
Two walks & 2 $\times$ 2 & query and key offsets & MLP of the sequence & 4000 & 8,192 \\
Stack & 1 $\times$ 8 & RoPE & last token & 4000 & 1,000 \\
Retrieval & 1 $\times$ 8 & none & query token & 4000 & 2,000 \\
\bottomrule
\end{tabular}

\end{table}

\begin{table}[htbp]
  \caption{\textbf{Models of the tasks from earlier work}, decoders with LayerNorm before each sublayer, an MLP with the Gaussian error linear unit, tied embeddings initialised as $\mathcal N(0,0.02^2)$ and the head dimension~$d$.
    The last three columns give the parameters of FPA, of softmax attention with RoPE and of an LSTM of the same depth and width.
    In variable assignment the attention models use learned absolute positions, elsewhere RoPE, and stick-breaking attention uses none.
    Variable assignment also has 2000 test sequences with only two values, and flip-flop 10,000 sparse and 2000 dense ones.
    MLP, multilayer perceptron.
    MQRAR, multi-query repeated associative recall.
    MQAR, multi-query associative recall.
    FPA, First Passage Attention.
    RoPE, rotary position embedding.
    LSTM, long short-term memory network.}
  \label{tab:19}
  \centering
  \small
  \setlength{\tabcolsep}{3.5pt}
  \begin{tabular}{@{}llllllllll@{}}
\toprule
Task & Layers $\times$ width & Heads $\times$ $d$ & MLP & Positions & Steps & Test sequences & FPA & Softmax & LSTM \\
\midrule
Variable assignment & 2 $\times$ 128 & 2 $\times$ 64 & 512 & learned & 4000 & 2,000 & 534,672 & 533,120 & 393,088 \\
MQRAR & 2 $\times$ 128 & 2 $\times$ 64 & 512 & RoPE & 4000 & 2,000 & 462,992 & 461,440 & 329,856 \\
Flip-flop & 1 $\times$ 64 & 2 $\times$ 32 & 256 & RoPE & 4000 & 2,000 & 50,632 & 50,240 & 33,664 \\
Copying & 2 $\times$ 128 & 2 $\times$ 64 & 512 & RoPE & 4000 & 2,000 & 399,888 & 398,336 & 266,752 \\
Counting & 2 $\times$ 128 & 2 $\times$ 64 & 512 & RoPE & 4000 & 2,000 & 407,696 & 406,144 & 274,560 \\
MQAR & 2 $\times$ 128 & 2 $\times$ 64 & 512 & RoPE & 8000 & 2,000 & 1,445,904 & 1,444,352 & 1,312,768 \\
\bottomrule
\end{tabular}

\end{table}

\begin{table}[htbp]
  \caption{\textbf{Declared search of the probes and of the tasks from earlier work.}
    Each configuration ran with two tuning seeds, later four, and a grid whose end was selected was extended once by a decade (\cref{tab:4} counts the runs).
    The 12 learning rates are equally spaced on a logarithmic scale and served softmax attention with RoPE, NoPE or ALiBi and stick-breaking attention on the alive sets.
    After the pilots the straight-through minimum replaced the four definitions of FPA, and its learning rate was selected again.
    On retrieval FoX and Selective Attention kept the settings selected on the alive sets and the two walks, on the stack they used the default initialisation of PyTorch and the scale 1, and on the tasks from earlier work the gate bias 0 and the scale 1.
    FPA, First Passage Attention.
    FoX, Forgetting Transformer.
    Sel.\ Att., Selective Attention.
    DCP, Dynamic Context Pruning.
    RoPE, rotary position embedding.
    NoPE, no positional encoding.
    ALiBi, attention with linear biases.}
  \label{tab:20}
  \centering
  \small
  \setlength{\tabcolsep}{4pt}
  \begin{tabular}{@{}lllll@{}}
\toprule
Method & Second setting & Values & Learning rates & Runs \\
\midrule
\multicolumn{5}{@{}l}{\emph{Alive sets, graded depths and two walks, per method and family}} \\
FPA & definition of the minimum & hard, soft with $\temp=1$, annealed, soft then hard & $\{1,3,10\}\times10^{-3}$ & 24 \\
FoX & initial bias of the gates & 0, 2, 4, 6 & $\{1,3,10\}\times10^{-3}$ & 24 \\
Sel.\ Att. & selection scale $\sasc$ & 1, 1/3, 3, 10 & $\{1,3,10\}\times10^{-3}$ & 24 \\
DCP & initial interaction bias & 2, 0, 4, 8 & $\{1,3,10\}\times10^{-3}$ & 24 \\
Softmax, stick-breaking & none & & 12 from $10^{-4}$ to $10^{-1}$ & 24 \\
\multicolumn{5}{@{}l}{\emph{Stack, retrieval, ablations and deeper models, per method and family}} \\
Every method & fixed & & $\{1,3,10\}\times10^{-3}$ & 6 \\
\multicolumn{5}{@{}l}{\emph{Tasks from earlier work, per method and task}} \\
Every method & fixed & & $\{1,3,10\}\times10^{-3}$ & 6 \\
\bottomrule
\end{tabular}

\end{table}

\begin{table}[htbp]
  \caption{\textbf{Learning rates of the reported runs}, with the second setting in brackets, namely the initial gate bias of FoX, the selection scale of Selective Attention and the initial bias of the interaction scores of DCP.
    The alive sets used one configuration for every~$r$ and, for the competitors, for their positive controls.
    Selective Attention with the scale 1 of its paper used $10^{-3}$ on the alive sets, and with one mask per head the scale 3 and $3\times10^{-2}$ on the two walks.
    The two- and four-layer models of FoX and of softmax attention used $3\times10^{-3}$ and $10^{-2}$.
    FPA used $3\times10^{-3}$ with the field only and $10^{-2}$ with the level of the query on the alive sets.
    On retrieval it used $3\times10^{-3}$ without the level of the key, $3\times10^{-2}$ with the level of the query and $10^{-3}$ without content, and on MQAR $10^{-3}$ without a bias in its first layer.
    FPA, First Passage Attention.
    FoX, Forgetting Transformer.
    Sel.\ Att., Selective Attention.
    DCP, Dynamic Context Pruning.
    RoPE, rotary position embedding.
    NoPE, no positional encoding.
    ALiBi, attention with linear biases.
    MQRAR, multi-query repeated associative recall.
    MQAR, multi-query associative recall.
    LSTM, long short-term memory network.}
  \label{tab:21}
  \centering
  \small
  \setlength{\tabcolsep}{4pt}
  \begin{tabular}{@{}lllllll@{}}
\toprule
 & Alive sets & Graded depths & Two walks & Stack & Retrieval & \\
\midrule
FPA & 0.03 & 0.003 & 0.001 & 0.1 & 0.01 & \\
FPA, sum of dips &  & 0.003 &  & 0.1 &  & \\
FoX & 0.03 (6) &  &  & 0.003 & 0.03 (6) & \\
Sel.\ Att. & 0.03 (10) & 0.03 (1) & 0.03 (1/3) & 0.001 (1) & 0.03 (1/3) & \\
DCP &  &  & 0.01 (0) &  &  & \\
Softmax with RoPE & 0.0534 &  &  & 0.001 &  & \\
Softmax with NoPE & 0.0152 &  &  &  & 0.003 & \\
ALiBi & 0.00433 &  &  &  &  & \\
Stick-breaking & 0.0152 &  &  &  & 0.1 & \\
\midrule
 & Variable assignment & MQRAR & Flip-flop & Copying & Counting & MQAR \\
\midrule
FPA & 0.001 & 0.001 & 0.01 & 0.003 & 0.01 & 0.0003 \\
FoX & 0.001 (0) & 0.003 (0) & 0.01 (0) & 0.003 (0) & 0.01 (0) & 0.001 (0) \\
Sel.\ Att. & 0.003 (1) & 0.001 (1) & 0.001 (1) & 0.003 (1) & 0.01 (1) & 0.003 (1) \\
Softmax with RoPE & 0.0001 & 0.001 & 0.001 & 0.01 & 0.01 & 0.001 \\
Stick-breaking & 0.001 & 0.001 & 0.01 & 0.01 & 0.001 & 0.003 \\
LSTM & 0.01 & 0.003 & 0.01 & 0.01 & 0.01 & 0.003 \\
\bottomrule
\end{tabular}

\end{table}

\clearpage
\section{Kernel and cost}
\label{app:E}

Our kernel computes FPA with its straight-through minimum in fused Triton kernels on a grid of blocks of 64 positions shared by queries and keys.
For a query block and an earlier key block the forward pass forms the range minimum from the suffix minima of the key block, the minimum over the blocks in between and the prefix minima of the query block.
It runs the longest query blocks first.
The backward pass needs the gradient of the soft minimum, whose normaliser has the same three parts.
It computes the gradient of the field in four parts, for the positions in the key block, in the query block, in the blocks in between and in the diagonal block.
All exponentials are powers of two of quantities in the log-two domain, the key side works on transposed tiles, and no value is added atomically, so the gradients are deterministic.
The evaluated FPA language model trained with an earlier version of this kernel.
Against a reference implementation in double precision, that version had relative errors of at most $5.1\times10^{-4}$ on the GPU, except $3.3\times10^{-3}$ in the gradient of the strength for tied values of the field.
In exploratory checks the current version had the same largest errors, and at most $6.7\times10^{-4}$ in backward passes over 4096 and 8192 positions.
In 200 exploratory training steps the two versions differed less than one version did under a different order of summing its micro-batches.

The benchmark ran on the GPU of \cref{app:C} in half precision, with every method timed in the same process and a probe for contention before and after each part.
\Cref{tab:22} lists one attention layer at five shapes, and \cref{tab:23} the eviction of keys whose bias fell below a threshold.
The bias of FoX inside our kernel took 1.05 times as long as the kernel of FoX, so the extra time of FPA in \cref{tab:3} comes from the operations that FPA adds.
The decoding times of \cref{tab:3} hardly changed with the length of the prompt, from 1024 to 16,384 tokens, because fixed costs dominate at batch size 1.
Our decoding code for Selective Attention also keeps a single-precision copy of the selection key, which raises its measured cache to 76{,}848 bytes per position, whereas \cref{tab:3} counts the 73{,}776 bytes that its definition needs.

\begin{table}[htbp]
  \caption{\textbf{One attention layer} with 12 heads of dimension 64 at batch $\times$ length, in milliseconds (medians of 20 rounds).
    The ratios in \cref{sec:4.3} are medians over the rounds of the ratio within each round.
    SDPA, the scaled dot-product attention of PyTorch, with its memory-efficient kernel on this GPU.
    FPA, First Passage Attention.
    FoX, Forgetting Transformer.
    Sel.\ Att., Selective Attention.}
  \label{tab:22}
  \centering
  \small
  \begin{tabular}{@{}lcccccccc@{}}
\toprule
 & \multicolumn{4}{c}{Forward, ms} & \multicolumn{4}{c}{Forward and backward, ms} \\
\cmidrule(lr){2-5}\cmidrule(lr){6-9}
Shape & FPA & FoX & SDPA & Sel.\ Att. & FPA & FoX & SDPA & Sel.\ Att. \\
\midrule
16$\times$512 & 0.37 & 0.34 & 0.27 & 0.67 & 1.76 & 1.13 & 1.64 & 3.54 \\
8$\times$1024 & 0.51 & 0.47 & 0.46 & 1.12 & 2.61 & 1.68 & 2.50 & 5.79 \\
4$\times$2048 & 0.78 & 0.74 & 0.87 & 2.00 & 4.31 & 2.87 & 4.42 & 10.45 \\
2$\times$4096 & 1.33 & 1.27 & 1.68 & 3.75 & 7.76 & 5.47 & 8.14 & 19.50 \\
1$\times$8192 & 2.39 & 2.33 & 3.30 & 7.29 & 15.19 & 10.45 & 15.79 & 37.37 \\
\bottomrule
\end{tabular}

\end{table}

\begin{table}[htbp]
  \caption{\textbf{Eviction} of the keys whose bias fell below $-\theta$ (FPA and FoX) or whose mask rose above $\theta$ (Selective Attention), with the loss change in nats per token and the share of keys kept on SlimPajama test windows.
    Softmax attention used a sliding window of the mean size that FPA keeps.
    Block skipping is the pruning of \citet{lin2025adaptive}.
    FPA, First Passage Attention.
    FoX, Forgetting Transformer.
    Sel.\ Att., Selective Attention.}
  \label{tab:23}
  \centering
  \small
  \setlength{\tabcolsep}{3pt}
  \begin{tabular}{@{}llcccc@{}}
\toprule
 & & \multicolumn{2}{c}{Length 2048} & \multicolumn{2}{c}{Length 8192} \\
\cmidrule(lr){3-4}\cmidrule(lr){5-6}
Method & Rule & loss change & kept & loss change & kept \\
\midrule
FPA & bias $\le-5$ & $+$0.0063 & 29\% & $+$0.0054 & 23\% \\
 & bias $\le-10$ & $+$0.0004 & 34\% & $+$0.0005 & 26\% \\
 & bias $\le-20$ & $+$0.0001 & 41\% & 0.0000 & 30\% \\
FoX & bias $\le-5$ & $+$0.1836 & 13\% & $+$0.2675 & 6\% \\
 & bias $\le-10$ & $+$0.0195 & 15\% & $+$0.0227 & 7\% \\
 & bias $\le-20$ & $+$0.0059 & 19\% & $+$0.0055 & 9\% \\
 & block skipping, 10 & $+$0.0007 & 24\% & $+$0.0007 & 10\% \\
Sel.\ Att. & mask $\ge5$ & $+$0.0001 & 33\% & $-$0.0006 & 15\% \\
 & mask $\ge10$ & 0.0000 & 44\% & 0.0000 & 22\% \\
Softmax & window, mean size of FPA & $+$0.1151 & 43\% & $-$2.8178 & 36\% \\
 & the same $+$ first key & $+$0.0559 & 43\% & $-$2.8077 & 36\% \\
\bottomrule
\end{tabular}

\end{table}

Beyond the training length, softmax attention without eviction diverged, so its sliding windows lowered its loss at length 8192.
FPA retrieved every passkey at 8192 and 16,384 tokens under every threshold, with 24\% to 27\% of its cache, whereas the windows of softmax attention retrieved at most 20\%.

\end{document}